\documentclass[acmsmall]{acmart}

\AtBeginDocument{%
  }

\usepackage{microtype}
\usepackage{graphicx}
\usepackage{algorithm}
\usepackage{algorithmic}
\usepackage{diagbox}
\usepackage{booktabs} 
\usepackage{amsmath}
\usepackage{amsthm}
\usepackage{color, colortbl}
\usepackage{xcolor}
\usepackage{hyperref}
\usepackage[capitalize]{cleveref}
\usepackage{nicefrac}

\usepackage{enumitem}
\usepackage{bm}
\usepackage{bbm}
\usepackage{graphicx}
\usepackage{thmtools}
\usepackage{thm-restate}
\usepackage{mathtools}
\usepackage[skip=0pt]{caption}
\usepackage[skip=0pt]{subcaption}
\usepackage{booktabs} 
\usepackage{enumitem}
\usepackage{accents}
\usepackage{tabularx}

\usepackage{threeparttable, array, float} 

\definecolor{red}{HTML}{E51400}  
\definecolor{blue}{HTML}{0050EF} 
\definecolor{green}{HTML}{008A00} 
\definecolor{purple}{HTML}{AA00FF} 
\definecolor{dark-red}{rgb}{0.4, 0.15, 0.15}
\definecolor{dark-blue}{rgb}{0.15, 0.15, 0.4}
\definecolor{medium-red}{rgb}{0.5, 0, 0}
\definecolor{medium-blue}{rgb}{0, 0, 0.5}
\definecolor{light-red}{rgb}{0.7, 0, 0}
\definecolor{light-blue}{rgb}{0, 0, 0.7}

\hypersetup{
    colorlinks=true,
    linkcolor=blue,
    filecolor=magenta,      
    urlcolor=blue,
}

\newtheorem{theorem}{\bf Theorem}
\newtheorem{lemma}{\bf Lemma}

\newtheorem{corollary}{\bf Corollary}
\newtheorem{definition}{\bf Definition}

\newtheorem{assumption}{Assumption}

\renewcommand*{\proofname}{\textnormal{\bf Proof}}

\theoremstyle{definition}
\newtheorem{remark}{\bf Remark}

\definecolor{red}{HTML}{E51400} 
\definecolor{blue}{HTML}{0050EF} 
\definecolor{green}{HTML}{008A00} 
\definecolor{purple}{HTML}{AA00FF} 
\definecolor{orange}{HTML}{FF7F00}
\definecolor{gray}{HTML}{848482}
\definecolor{Gray}{gray}{0.85}
\definecolor{LightGray}{gray}{0.96}

\usepackage{xspace}
\usepackage{comment}

\usepackage{fontawesome5}

\usepackage[colorinlistoftodos,prependcaption,textsize=tiny,textwidth=2cm,color=orange]{todonotes}

\newcommand{\carlee}[1]{{\textcolor{cyan}{CJW: #1}}}

\DeclareMathOperator*{\argmin}{argmin}
\DeclareMathOperator*{\argmax}{argmax}

\newcommand{\R}{\mathbb{R}}

\newcommand{\offalgo}{Off-C$^2$PL}
\newcommand{\hybalgo}{A$^2$-Off-C$^2$PL}

\newcommand{\hybmodel}{active-data augmented model}
\newcommand{\offmodel}{pure offline model}

\makeatletter
\newcommand*{\rom}[1]{\expandafter\@slowromancap\romannumeral #1@}
\makeatother

\newcommand{\ts}[1]{}

\newcommand{\compilefullversion}{true}
\ifthenelse{\equal{\compilefullversion}{false}}{%
	\newcommand{\OnlyInFull}[1]{}
	\newcommand{\OnlyInShort}[1]{#1}
}{%
	\newcommand{\OnlyInFull}[1]{#1}%
	\newcommand{\OnlyInShort}[1]{}%
}%

\newcommand{\compilehidecomments}{false}
\ifthenelse{ \equal{\compilehidecomments}{true} }{%
	\newcommand{\wei}[1]{}
	\newcommand{\xutong}[1]{}
	\newcommand{\jinhang}[1]{}
	\newcommand{\siwei}[1]{}
        \newcommand{\carlee}[1]{}

}{
\newcommand{\wei}[1]{{\color{blue}{[Wei: #1]}}}
\newcommand{\xutong}[1]{{\color{green} [Xutong: #1]}}
\newcommand{\jinhang}[1]{{\color{orange} [\text{Jinhang:} #1]}}
\newcommand{\siwei}[1]{{\color{red} [\text{Siwei:} #1]}}

}

\allowdisplaybreaks

\newenvironment{talign*}
 {\csname align*\endcsname}
 {\endalign}



\begin{document}

\title[Offline Clustering of Preference Learning with Active-data Augmentation]{Offline Clustering of Preference Learning with Active-data Augmentation}

\author{Jingyuan Liu}
\affiliation{
\institution{Nanjing University}
\country{China}}
\email{jingyuanliu@smail.nju.edu.cn}

\author{Fatemeh Ghaffari}
\affiliation{
\institution{University of Massachusetts Amherst}
\country{USA}}
\email{fghaffari@umass.edu}

\author{Xuchuang Wang}
\affiliation{
\institution{University of Massachusetts Amherst}
\country{USA}}
\email{xuchuangwang@cs.umass.edu}

\author{Xutong Liu}
\authornote{Xutong Liu is the corresponding author.}
\affiliation{
\institution{University of Washington}
\country{USA}}
\email{xutongl@uw.edu}

\author{Mohammad Hajiesmaili}
\affiliation{
\institution{University of Massachusetts Amherst}
\country{USA}}
\email{hajiesmaili@cs.umass.edu}

\author{Carlee Joe-Wong}
\affiliation{
\institution{Carnegie Mellon University}
\country{USA}}
\email{cjoewong@andrew.cmu.edu}




\begin{abstract}

Preference learning from pairwise feedback is a widely adopted framework in applications such as reinforcement learning with human feedback and recommendations. In many practical settings, however, user interactions are limited or costly, making offline preference learning necessary. Moreover, real-world preference learning often involves users with different preferences. For example, annotators from different backgrounds may rank the same responses differently. This setting presents two central challenges: (1) identifying similarity across users to effectively aggregate data, especially under scenarios where offline data is imbalanced across dimensions, and (2) handling the imbalanced offline data where some preference dimensions are underrepresented. To address these challenges, we study the Offline Clustering of Preference Learning problem, where the learner has access to fixed datasets from multiple users with potentially different preferences and aims to maximize utility for a test user.  To tackle the first challenge, we first propose Off-C$^2$PL for the pure offline setting, where the learner relies solely on offline data. Our theoretical analysis provides a suboptimality bound that explicitly captures the tradeoff between sample noise and bias. To address the second challenge of inbalanced data, we extend our framework to the setting with active-data augmentation where the learner is allowed to select a limited number of additional active-data for the test user based on the cluster structure learned by Off-C$^2$PL. In this setting, our second algorithm, A$^2$-Off-C$^2$PL, actively selects samples that target the least-informative dimensions of the test user’s preference. We prove that these actively collected samples contribute more effectively than offline ones. Finally, we validate our theoretical results through simulations on synthetic and real-world datasets. 

\end{abstract}

%
%

%


\maketitle

\section{Introduction}\label{sec:intro}


Learning human preferences is a fundamental building block of modern AI systems. Whether aligning large language models (LLMs) with human values~\citep{ouyang2022training, bai2022training}, recommending movies or products~\citep{yan2022dynamic, aramayo2023multiarmed}, or personalizing digital assistants~\citep{musto2021myrrorbot, stucke2017digital}, systems must understand not only what actions are available but which ones people actually prefer. Unlike traditional supervised learning tasks~\citep{verma2021introduction, jiang2020supervised} with clear ground-truth labels, preference learning must infer the subjective and often heterogeneous nature of human choices. In the examples above, a model that fails to capture preferences may generate fluent yet misaligned LLM outputs, or recommend items that frustrate rather than engage users. 

A practical way to elicit such preferences is through pairwise feedback: rather than assigning absolute scores, users (or annotators) simply indicate which of two options they prefer. Pairwise comparisons are natural in practice: for instance, evaluators can more easily judge which of two LLM responses is better instead of assigning absolute scores to possible responses, and users often reveal preferences implicitly by choosing one product over another. This learning framework has been extensively modeled and studied under the dueling bandits problem, which uses sequences of pairwise comparisons to infer underlying preference structures~\citep{bengs2021preference, yue2012k, dudik2015contextual, saha2021optimal, saha2022efficient}.

Despite significant progress, most prior work assumes a single, shared preference vector, overlooking the fact that preferences vary across users in real-world applications. In LLM alignment, for instance, annotators from different backgrounds or cultures may rank responses differently. In recommendation, users routinely disagree on the same items. If we aggregate all feedback indiscriminately, the result is a one-size-fits-none policy. If we treat each user separately, limited per-user data leads to poor learning. The natural solution is to cluster users with similar preferences: pooling their data to increase sample sufficiency while preserving personalization.


However, clustering similar users becomes particularly challenging in the offline preference learning, where the learner has access only to fixed, pre-collected datasets of pairwise comparisons rather than interactive feedback that enables more accurate preference estimation. This setting is increasingly relevant in practice: in LLM alignment, reinforcement learning from human feedback (RLHF)~\citep{lee2023rlaif, bai2022training, casper2023open} often relies on static datasets of human comparisons between possible responses, while in recommender systems~\citep{ek2015recommender, yan2022dynamic, aramayo2023multiarmed}, historical user logs provide pairwise evidence of preferences between different items. In both cases, the learner must leverage existing data to select actions that maximize user utility (i.e., satisfaction) from this fixed, given dataset~\citep{zhu2023principled, das2024active, liu2024dual, li2025provably}.

Motivated by this gap, we study the problem of \textit{Offline Clustering of Preference Learning}, where $U$ users are partitioned into $J$ clusters. Users in the same cluster share a common preference vector, while those in different clusters do not. Each user has a fixed offline dataset of pairwise comparisons, where for a given context (the input condition or situation, e.g., a prompt in RLHF or user profile in recommendation),
the user provides a binary preference between two candidate actions. We assume preferences follow the Bradley–Terry–Luce (BTL) model~\citep{bradley1952rank, debreu1960individual}. 
The goal is to identify users that have similar preferences as the test user 
and aggregate their data to increase sample sufficiency 
to learn a personalized policy that selects actions with near-optimal expected reward. 

This setting presents two central challenges:
(1) \textit{Identifying similarity across users without coverage assumptions}:  
A central challenge in leveraging offline data from users with potentially different preferences is to identify users that are similar to the test user.  
Prior works on clustering of bandits~\citep{gentile2014online, li2018online, li2019improved, wang2023onlinea, wang2025online, li2025demystifying, liu2025offline} typically rely on an \textit{item regularity} assumption, which requires offline actions to provide balanced information across all preference dimensions, ensuring sufficient data coverage but limiting generality.  
However, this assumption becomes unrealistic in our setting with pairwise feedback~\citep{wang2025online}, where the pairwise feedback may be interdependent and may distort data coverage.  
Thus, our problem demands identifying user similarity directly from imbalanced and potentially incomplete offline data, without relying on any coverage guarantees.
This leads to the second challenge. (2) \textit{Handling imbalanced offline data}: A natural way to mitigate imbalanced offline data is to collect new data samples of strategically chosen preferences, e.g., in a LLM setting, users may be presented with two carefully chosen responses and asked to indicate their preference between them, before receiving the final LLM response.
Active learning approaches~\citep{settles2009active, char2019offline, li2022near, mehta2023sample, ji2024reinforcement, das2024active} mitigate imbalance by querying new comparisons, but they assume fully interactive querying rather than the hybrid offline–active regime considered here. 
Hence, a key challenge is to effectively integrate actively collected data with fixed offline datasets, ensuring that new samples complement rather than exacerbate the imbalance in coverage across preference dimensions.
To address these challenges, we focus on two central research questions:  
\emph{(1) Can we effectively identify users with similar preferences, especially under fixed and imbalanced offline data without relying on coverage assumptions? (2) How can we actively collect additional data to mitigate the impact of poor coverage in imbalanced offline datasets that fail to represent all preference dimensions?}

\begin{table}[t]
\centering
\caption{Summary of main and additional theoretical results.}\label{tab:theoretical_results}
\resizebox{\textwidth}{!}{%
\begin{tabular}{@{}ccccc@{}}
\toprule
\multicolumn{5}{c}{\textbf{Comparisons of Algorithms for Pure Offline Model}} \\
\cmidrule(lr){1-5} & \textbf{Algorithm} & \textbf{Setting} & \textbf{Condition} & \textbf{Suboptimality} \\
\midrule

\textbf{Previous~\citep{zhu2023principled, li2025provably}} & \begin{tabular}[c]{@{}c@{}}
P-MLE~\citep{zhu2023principled}
\\ PDC~\citep{li2025provably}\end{tabular} & \begin{tabular}[c]{@{}c@{}} Pure Offline \\ Single User\end{tabular} & --- & 
$\tilde{O}\bigg( \sqrt{ \frac{d}{\lambda_{1}} } \bigg)$ \\

\rowcolor{gray!15}
\begin{tabular}[c]{@{}c@{}} \textbf{Main Result 1} \\ (\Cref{th1}) \end{tabular} & 
& 
& --- & 
$\tilde{O}\Bigg( \frac{\sqrt{d} \big( 1 + \hat\gamma \sqrt{ N_{1}} \big) }{ \sqrt{\lambda_2} } \Bigg)$  \\

\rowcolor{gray!15}
\begin{tabular}[c]{@{}c@{}} \textbf{Additional Result 1} \\ (\Cref{align_special_form_of_th1}) \end{tabular} & \begin{tabular}[c]{@{}c@{}} \offalgo{} \\ (\Cref{al1}) \end{tabular} 
& \begin{tabular}[c]{@{}c@{}} Pure Offline \\ Multiple Users \end{tabular} & \begin{tabular}[c]{@{}c@{}} Lower Threshold $\hat\gamma\leq\gamma$ \\(\Cref{definition_minimum_heterogeneity_gap})\end{tabular} & 
$\tilde{O}\bigg( \sqrt{ \frac{d}{\lambda_2} } \bigg)$ \\

\rowcolor{gray!15}
\begin{tabular}[c]{@{}c@{}} \textbf{Additional Result 2} \\ (\Cref{th1_item_regularity}) \end{tabular} & 
& 
& \begin{tabular}[c]{@{}c@{}} Item Regularity \\ (\Cref{assumption_item_regularity}) \end{tabular} & 
$\tilde{O}\bigg( \sqrt{\frac{d}{\tilde\lambda_a}} \Big( \sqrt{ \frac{1}{N_{3} } } + \hat\gamma\sqrt{\eta_{1}} \Big) \bigg)$ \\

\midrule
\multicolumn{5}{c}{\textbf{Comparisons of Algorithms for Active-data Augmented Model}} \\
\cmidrule(lr){1-5} & \textbf{Algorithm} & \textbf{Setting}
& \textbf{Condition} & \textbf{Suboptimality} \\

\midrule

\begin{tabular}[c]{@{}c@{}}\textbf{Previous~\citep{das2024active}}
\end{tabular}& APO~\citep{das2024active} & \begin{tabular}[c]{@{}c@{}} Pure Active \\ Single User \end{tabular} & --- & 
$\tilde{O}\bigg( \frac{ d }{ \sqrt{ N } } \bigg)$ \\

\rowcolor{gray!10}
\begin{tabular}[c]{@{}c@{}}\textbf{Main Result 2} \\ (\Cref{th2}) \end{tabular} & 
& 
& --- & 
$\tilde{O}\Bigg( \frac{ \sqrt{d} \big( 1 + \hat\gamma \sqrt{ N_{1} } \big) }{ \sqrt{\lambda_3 + N/d } } \Bigg)$ \\

\rowcolor{gray!10}
\begin{tabular}[c]{@{}c@{}} \textbf{Additional Result 3} \\ (\Cref{corollary_extreme_datasets}) \end{tabular} & \begin{tabular}[c]{@{}c@{}} \hybalgo{} \\ (\Cref{al2}) \end{tabular} & \begin{tabular}[c]{@{}c@{}} Hybrid (Offline + Active) \\ Multiple Users \end{tabular}  & Imbalanced Dataset (\Cref{definition_sample_extreme_matrix}) & 
$\tilde{O}\Bigg( \frac{\sqrt{d} \big( 1 + \hat\gamma \sqrt{ N_{1} } \big) }{ \sqrt{\lambda_2 + N } } \Bigg)$ \\

\rowcolor{gray!10}
\begin{tabular}[c]{@{}c@{}} \textbf{Additional Result 4} \\ (\Cref{corollary_imbalanced_regularity_datasets}) \end{tabular} & 
& 
& 
\begin{tabular}[c]{@{}c@{}}Item Regularity (\Cref{assumption_item_regularity})\\ + Imbalanced Dataset (\Cref{definition_sample_extreme_matrix})\end{tabular} & 
$\tilde O\bigg( \sqrt{\frac{d}{\tilde\lambda_a}} \Big(\sqrt{\frac{1}{N_{3}}} + \hat\gamma\sqrt{\eta_{2}} \Big)
\bigg)$ \\
\bottomrule\end{tabular}}
\begin{minipage}{\textwidth}
\scriptsize
\setlength{\parindent}{0pt}\raggedright
Here, $d$ denotes the dimension of each user's preference vector. $\lambda_1$, $\lambda_2$, and $\lambda_3$ represent the minimum eigenvalue of the (regularized) information matrix constructed from (i) the test user’s offline data only, (ii) the test user’s offline data combined with aggregated data from clustered neighbors, and (iii) case (ii) further augmented with $N$ actively selected samples for the test user, respectively. $\tilde\lambda_a$ is the smoothed item regularity parameter, which lower bounds the information matrix in terms of the number of samples used. $N_1$ denotes the number of heterogeneous offline samples included, $N_2$ the total number of offline samples used, and $N_3$ the total number of samples combining offline and active data. Finally, $\eta_1 = N_1/N_2$ and $\eta_2 = N_1/N_3$ represent the fraction of heterogeneous samples among all offline samples and among the combined offline–active datasets, respectively.
\end{minipage}
\vspace{-0.3in}
\end{table}

\Cref{tab:theoretical_results} summarizes the main contributions of our paper (with the key notations introduced at the bottom of the table). 
We highlight four key \textbf{contributions} as follows:

\noindent \textbf{(i) Model Formulations:} We are the first to introduce the \textit{Offline Clustering of Preference Learning} framework, where the learner needs to learn heterogeneous user preferences from offline pairwise feedback, without assuming any data coverage assumption. 
This setting naturally leads to the two core challenges discussed earlier: identifying user similarity and handling imbalanced offline data. To formalize the problem, we first present the pure offline model, followed by its extended model with active-data augmentation. 
In the \offmodel{}, the learner relies solely on the fixed offline datasets to infer each user’s preferences, cluster users with similar preferences, and aggregate their data to improve estimation accuracy. This reflects realistic scenarios such as aligning large language models using RLHF datasets collected from annotators across different regions, or personalizing recommendations from logged data of diverse user populations. Based on this, the \hybmodel{} allows the learner to actively acquire a fixed number of additional samples to refine the estimation for the test user, while still leveraging the offline data. This setting captures practical cases like requesting a small number of extra comparisons from annotators in RLHF, or collecting additional pairwise feedback from users in recommender systems.

\noindent \textbf{(ii) Algorithm and Results for Pure Offline Model:} In order to address the challenge of identifying similar users, we develop the first algorithm, \offalgo{} (Algorithm~\ref{al1}) for the \offmodel{}. 
\offalgo{} constructs confidence interval on preference estimation for each user based on the \textbf{minimum eigenvalue of each user’s information matrix}, which captures the least informative dimension, and applies Maximum Likelihood Estimation (MLE) under the BTL model to estimate preferences. This design ensures that the confidence interval directly reflects data sufficiency and estimation accuracy without requiring any coverage assumption. A \textbf{clustering threshold parameter} $\hat\gamma$ is then used to determine similarity across users: intuitively, $\hat\gamma$ balances inclusiveness of clusters against the risk of aggregating heterogeneous users whose preferences are different with the test user. Building on this structure, the algorithm aggregates data across identified clusters to improve estimation. Main Result~1 in \Cref{tab:theoretical_results} shows that \offalgo{} achieves a suboptimality of $\tilde O\big((\sqrt{d}+\hat\gamma\sqrt{dN_1})/\sqrt{\lambda_2}\big)$, where $d$ is the preference dimension, $N_1$ the number of heterogeneous samples utilized, and $\lambda_2$ the minimum eigenvalue of the aggregated offline information matrix across those identified similar users. 
This bound has a numerator representing \textbf{noise} ($\sqrt{d}$) and \textbf{bias} ($\hat\gamma \sqrt{dN_1}$), and a denominator $\sqrt{\lambda_2}$ that reflects the \textbf{information gain from aggregating samples of similar users} (as determined by $\hat\gamma$). 
A smaller $\hat\gamma$ enforces stricter similarity, reducing $N_1$ but also lowering $\lambda_2$, while a larger $\hat\gamma$ has the opposite effect. 
This quantifies the tradeoff in setting $\hat\gamma$. 
With a proper choice of $\hat\gamma$, the bias term can be eliminated (Additional Result~1), yielding guarantees that improve upon single-user baselines relying only on test user data~\citep{zhu2023principled, li2025provably}. 
Further, by analyzing the item regularity assumption~\citep{gentile2014online, wang2023onlinea, wang2025online, liu2025offline} as a special case, Additional Result 2 highlights more clearly the balance between reducing noise and bias, which extends prior offline clustering of bandits result in traditional linear reward~\citep{liu2025offline} to our setting with pairwise feedback. 

\noindent \textbf{(iii) Algorithm and Results for Active-data Augmented Model:} Building on the structure learned by \offalgo{}, we introduce \hybalgo{} under the \hybmodel{}, which extends \offalgo{} to address the imbalance of offline datasets. \hybalgo{} actively selects contexts and action pairs that maximize information gain along \textbf{the least-covered dimensions of the test user's information matrix}, thereby strengthening the weakest directions of the data. This active design yields significantly improved theoretical performance compared with only using pure offline data, as established in the following results. Main Result~2 shows that \hybalgo{} achieves suboptimality $\tilde O\big((\sqrt{d}+\hat\gamma\sqrt{dN_1})/\sqrt{\lambda_3+N/d}\big)$, where $\lambda_3$ is the minimum eigenvalue of the information matrix combining aggregated pure offline data from \offalgo{} with the $N$ actively selected samples. Compared to Main Result~1, this active augmentation improves the suboptimality gap in two ways: 
(1) by \textbf{directly adding $N$ new active samples}, which contributes an additional $N/d$ term in the denominator; and 
(2) by \textbf{increasing the minimum eigenvalue of the information matrix} from $\lambda_2$ to $\lambda_3$ through targeted sampling of underrepresented directions. 
As formalized in \Cref{lemma_difference_of_eigenvalues} and Additional Result~3, when the offline data is imbalanced and performance is bottlenecked by a few weak dimensions, each active sample can be \textbf{as valuable as up to $d$ equivalent offline samples}, yielding an additional $N$ term in the denominator compared to the pure offline case (Main Result 1). 
Finally, Additional Result 4 demonstrates the further benefits of active augmentation under the item regularity assumption, where the bias is more tightly controlled, yielding performance that strictly outperforms the pure offline case with item regularity assumption (Additional Result 2).

\noindent \textbf{(iv) Empirical Validation:} 
We run experiments on a synthetic benchmark and on the Reddit TL;DR dataset. 
In the offline setting, we vary the number of samples per user from \(10\%\) to \(100\%\) of the available data and report the suboptimality gap. 
In this setting, \offalgo{} consistently achieves the lowest gap, leveraging cross-user information within clusters, especially when samples are scarce. 
The improvements are \(61.47\%\) over KMeans and \(80.07\%\) over Off-DBSCAN. 
In the setting with active-data augmentation, each method is warm started with \(20\%\) of the data, followed by \(500\) rounds of learning. \hybalgo{} outperforms an online-only algorithm APO~\citep{das2024active} and \offalgo{} with only random-data augmentation baseline by \(87.58\%\) and \(57.51\%\), respectively.

This paper is organized as follows: We review crucial related works in \Cref{section_related_works}.  
In \Cref{sec: problem_setting}, we introduce the offline clustering of preference learning problem along with its two settings: the pure offline setting and the active-data augmented setting.  
We then present the algorithm design and theoretical analysis for the \offmodel{} in \Cref{sec:pure_offline}, followed by those for \hybmodel{} in \Cref{sec:active_data_selection}.  
Finally, we validate our theoretical findings through experiments on both synthetic and real-world datasets in \Cref{sec:simulations}, and conclude the paper in \Cref{sec:conclusion}.

\section{Related Works}\label{section_related_works}

\paragraph{Offline RL and Bandit Learning.} Offline statistical learning~\citep{zhang2014confidence, cai2017confidence} primarily focuses on parameter estimation, while offline reinforcement learning (batch RL) extends the scope to sequential decision-making problems using fixed offline datasets~\citep{lange2012batch, levine2020offline, jin2021pessimism, rashidinejad2021bridging, xiao2021optimality, xie2021bellman}, and has found wide applications in diverse domains such as dialogue generation~\citep{jaques2019way}, autonomous driving~\citep{yurtsever2020survey}, educational technologies~\citep{singla2021reinforcement} and personal recommendations~\citep{li2010contextual, bottou2013counterfactual}. 
Within this landscape, offline bandits\textemdash viewed as a special case of offline RL\textemdash extend the multi-armed bandit framework to learning solely from pre-collected data~\citep{shivaswamy2012multi}. Prior studies have considered settings where the offline distributions align with the online reward distributions~\citep{bu2020online, banerjee2022artificial} or where distribution shift arises between them~\citep{zhang2019warm, cheung2024leveraging}. Among them, studies on offline contextual linear bandits~\citep{li2022pessimism, wang2024oracle} are most closely related to our setting. However, our work goes beyond the standard contextual linear bandits formulation by studying pairwise feedback modeled through a logistic function, and by explicitly leveraging the clustering structure among users’ preferences for more efficient learning.

\paragraph{Preference Learning from Pairwise Feedback.}
Theoretical studies of preference learning from pairwise feedback trace back to the dueling bandit problem~\citep{yue2012k, saha2021optimal, bengs2021preference} and its extension, the contextual dueling bandit problem~\citep{dudik2015contextual}. These ideas extend naturally to preference-based reinforcement learning~\citep{xu2020preference, chen2022human, wang2023rlhf, zhan2023query}.
Recent work has emphasized offline preference-based RL, often motivated by reinforcement learning with human feedback (RLHF). Approaches include pessimism-driven methods\citep{zhu2023principled, zhan2023provable, li2023reinforcement} and KL-regularized formulations~\citep{tiapkin2023regularized, xiong2023iterative, xiao2024algorithmic}. For instance, \citet{xiong2023iterative} study active context selection under strong coverage assumptions, deriving sample-dependent bounds.
Beyond RLHF, researchers have explored general preference structures~\citep{rosset2024direct, gui2024bonbon, ye2024online}, pure active preference learning without offline datasets~\citep{das2024active}, safety-constrained alignment~\citep{wachi2024stepwise}, and sample-efficient learning under limited data~\citep{kim2025spread}.
Our work departs from these above mentioned works
by explicitly incorporating clustering into pairwise preference learning and combining it with active data augmentation. This introduces two new challenges: (1) reliably inferring clusters from noisy offline comparisons, and (2) selecting informative queries when both contexts and actions matter. Importantly, learning from pairwise feedback provides weaker supervision than full-reward feedback, making these challenges sharper. We address them with algorithms and bounds that reveal the interplay between clustering, data coverage, and active exploration in both pure offline and hybrid settings.

\paragraph{Heterogeneous Preference Learning.}
Heterogeneous preference learning has been widely studied under the clustering of bandits~\citep{gentile2014online, li2018online, li2019improved} and multi-task learning~\citep{duan2023adaptive}, where data from users with distinct preference vectors could be used to accelerate learning. Later works investigate privacy~\citep{liu2022federated}, model misspecification~\citep{wang2023onlinea}, and robustness to corrupted users~\citep{wang2023onlineb}. More recent studies by~\citet{liu2025offline} and ~\citet{wang2025online} are closely related to our setting, respectively providing offline and online algorithms for clustering of bandits, whereas we study the preference learning from pairwise feedback under the offline and active-data augmented settings. With growing interest in RLHF, recent efforts have addressed scenarios involving users with diverse preferences, which are often referred to as personalized RLHF~\citep{kirk2023personalisation, li2024personalized, conitzer2024social, jang2023personalized, poddar2024personalizing, ramesh2024group}. 
Theoretically, \citet{liu2024dual} study heterogeneous user rationality, \citet{zhong2024provable} focus on meta-learning and social welfare aggregation, and \citet{park2024rlhf} analyze representation-based aggregation under assumptions on uniqueness, diversity, and concentrability.
Compared to these directions, our work is the first to establish a general clustering-based framework for heterogeneous preference learning without imposing assumptions on the underlying clustering structure or data coverage, and to extend beyond the conventional pure offline setting by incorporating an active-data augmentation mechanism that adaptively improves underrepresented dimensions.

\section{Setting}\label{sec: problem_setting}

\paragraph{Notations.} Throughout this paper, we use $[s] = \{1, 2, \ldots, s\}$ to denote the set of integers from $1$ to $s$. For any matrix $M \in \mathbb{R}^{d \times d}$, we write $\lambda_{\min}(M) = \lambda_1(M)$ to denote its smallest eigenvalue, and $\lambda_i(M)$ to denote its $i$-th smallest eigenvalue. For vector norms, we use $\| \cdot \|_2$ to denote the Euclidean ($\ell 2$) norm, and $\| \cdot \|_M$ to denote the Mahalanobis norm defined with respect to matrix $M$.

\subsection{Problem Formulation}

We consider a set of $U$ users, denoted by $\mathcal{U} = [U]$, where each user $u \in \mathcal{U}$ is associated with a preference vector $\bm\theta_u \in \Theta$, with $\Theta \coloneqq \left\{\bm\theta \in \mathbb{R}^d \mid \left\|\bm\theta\right\|_2 \leq 1\right\}$. To model preference heterogeneity, the users are partitioned into $J$ clusters ($J\leq U$), where all users within the same cluster $j \in [J]$ share a common preference vector $\bm\theta^j$. Specifically, let $\mathcal{U}(j)$ denote the set of users in cluster $j$, so that $\mathcal{U} = \bigcup_{j=1}^J \mathcal{U}(j)$ and $\mathcal{U}(j) \cap \mathcal{U}(j') = \emptyset$ for any $j \neq j'$. By construction, users in the same cluster share the same preference vector\footnote{In practice, users within a cluster may have similar but not identical preferences (e.g., individuals from similar backgrounds often exhibit minor differences). Our results remain valid under such variations, as discussed in \Cref{remark_robustness_of_al1} and verified in \Cref{sec:simulations}. For clarity and consistency with prior works~\citep{gentile2014online, li2018online, li2019improved, liu2025offline}, we still assume identical preferences in each cluster.}, i.e., $\bm\theta_u = \bm\theta_{u'}$ if and only if there exists a cluster $j$ such that $u, u' \in \mathcal{U}(j)$.
We further denote by $j_u$ the cluster index to which user $u$ belongs. Note that both the true clustering and the number of clusters are \textbf{unknown} to the learner. For a given user $u$, we refer to users in the same cluster as \emph{homogeneous users} and those in different clusters as \emph{heterogeneous users}.

In the offline preference learning setting, each user $u \in \mathcal{U}$ is provided with an offline dataset 
\(\mathcal{D}_u=\left\{\left(\bm x_u^i,\bm a_u^i, {\bm a'}_u^{i},y_u^i\right)\right\}_{i=1}^{N_u}\)
where $N_u$ denotes the number of samples for each user, and we further define $N_{\mathcal{S}} = \sum_{u \in \mathcal{S}} N_u$ as the total number of samples from all users in a set $\mathcal{S}$. Within each dataset $\mathcal{D}_u$, $\bm x_u^i\in\mathcal{X}$ represents a context for selecting actions (e.g., prompts in RLHF or specific user features in recommendation systems) randomly drawn from the context set $\mathcal{X}$, and $\bm a_u^i,{\bm a'}_u^i\in\mathcal{A}$ represent a pair of candidate actions (e.g., responses in RLHF or items in recommendation systems) randomly drawn from the action set $\mathcal{A}$. The binary feedback $y_u^i$ indicates user $u$'s preference: $y_u^i=1$ implies that user $u$ prefers action $\bm a_u^i$ over ${\bm a'}_u^i$ given context $\bm x_u^i$, whereas $y_u^i=0$ implies the opposite. Preferences $y_u^i$ are assumed to follow the Bradley–Terry–Luce (BTL) model~\citep{bradley1952rank, debreu1960individual, zhu2023principled}:
\begin{align*}
\mathbb{P}\left[y_u^i=1\mid u,\bm x_u^i,\bm a_u^i,{\bm a'}_u^i\right] & = \frac{1}{1 + \exp\left(-(r_u(\bm x_u^i,\bm a_u^i)-r_u(\bm x_u^i,{\bm a'}_u^i))\right)} \\
& = \sigma\left(\bm\theta_u^{\top}\left(\phi(\bm x_u^i, \bm a_u^i) - \phi(\bm x_u^i,{\bm a'}_u^i)\right)\right),
\end{align*}
where $r_u(\bm x, \bm a)=\bm\theta_u^{\top}\phi(\bm x,\bm a)$ is a linear reward function parameterized by an unknown vector $\bm\theta_u$ and a known feature mapping $\phi:\mathcal{X}\times\mathcal{A}\rightarrow\mathbb{R}^d$ with $\|\phi(\bm x,\bm a)\|_2\leq 1$ for all $(\bm x,\bm a)\in\mathcal{X}\times \mathcal{A}$, and $\sigma(x)=\frac{1}{1+e^{-x}}$ denotes the sigmoid function. The interpretations of the context, action, and feature map $\phi$ in practical applications are discussed in detail in \Cref{remark_real-world_applications}. Additionally, we define the feature difference $\bm z_u^i=\phi(\bm x_u^i,\bm a_u^i)-\phi(\bm x_u^i,{\bm a'}_u^i)$, noting that $(\bm\theta^{\top}\bm z)$ is 2-subgaussian for any $\bm\theta\in\Theta$. 



A policy \(\pi:\mathcal{X}\rightarrow\mathcal{A}\) is a mapping from contexts to actions.
Given an arbitrary test user \(u_t\in\mathcal{U}\), we define the \textit{suboptimality gap} of a policy \(\pi_{u_t}\) as:
\begin{align}\label{align_objective}
\text{SubOpt}_{u_t}\left(\pi_{u_t}\right) \coloneqq J_{u_t}\left(\pi_{u_t}^*\right) - J_{u_t}\left(\pi_{u_t}\right)
=\mathbb{E}_{\bm x\sim\rho_p}\left[ \bm\theta_{u_t}^{\top}\phi(\bm x,\pi_{u_t}^*(\bm x)) - \bm\theta_{u_t}^{\top}\phi(\bm x,\pi_{u_t}(\bm x)) \right],
\end{align}
where \(J_u(\pi)=\mathbb{E}_{\bm x\sim\rho_p}[ r_u(\bm x,\pi(\bm x))]\) denotes the expected reward for user \(u\) under policy \(\pi\), \(\pi_u^*=\arg\max_{\pi}J_u(\pi)\) is the optimal policy, and \(\rho_p\) denotes the distribution over contexts.

We consider two settings based on dataset availability:
\begin{itemize}
    \item \textbf{Pure Offline Model:} In this setting, the policy \(\pi_{u_t}\) for the test user \(u_t\) is derived from fixed, pre-collected offline datasets \(\mathcal{D}=\bigcup_{u\in\mathcal{U}}\mathcal{D}_u\). The objective is to minimize the suboptimality gap in \Cref{align_objective} using solely offline data.

    \item \textbf{Active-data Augmented Model:} In addition to the fixed offline dataset \(\mathcal{D}\), the learner actively selects \(N\) additional data points specifically for the test user \(u_t\). At each active selection round \(n\in[N]\), the learner chooses a data tuple \(\left(\mathring{\bm x}_{u_t}^n, \mathring{\bm a}_{u_t}^n, \mathring{\bm a}_{\ u_t}^{\prime n}\right)\in\mathcal{X}\times\mathcal{A}\times\mathcal{A}\), obtains preference feedback \(\mathring{y}_{u_t}^n\), and forms an active dataset \(\mathring{\mathcal{D}}=\left\{\left(\mathring{\bm x}_{u_t}^n, \mathring{\bm a}_{u_t}^n, \mathring{\bm a}_{\ u_t}^{\prime n}, \mathring{y}_{u_t}^n\right)\right\}_{n=1}^{N}\) after $N$ rounds. The objective is to minimize \Cref{align_objective} by leveraging both offline and actively collected datasets \(\mathcal{D}\cup\mathring{\mathcal{D}}\).
\end{itemize}

\begin{remark}[Distinctions from Classical Clustering of Bandits Works]\label{remark_distinctions_ClusBand}
In addition to the setting differences discussed in \Cref{section_related_works}, we highlight the differences in assumptions between this paper and classical clustering of bandits works~\citep{gentile2014online, li2018online, li2019improved, wang2025online, li2025demystifying}. Previous studies typically rely on three assumptions: (i) \textit{user randomness}, ensuring balanced data across users; (ii) \textit{sufficient data} with a large heterogeneity gap for correct clustering; and (iii) \textit{item regularity}, guaranteeing adequate coverage across all preference dimensions. While the only prior offline work~\citep{liu2025offline} relaxes user randomness and data sufficiency, it still depends on item regularity. However, this assumption is overly restrictive in our setting and real-life scenarios, as pairwise feedback may be interdependent and distort coverage. In contrast, we remove all three assumptions to develop a more general and practical framework, treating the setting with item regularity assumption only as a special case.
\end{remark}


\subsection{Representative Applications}\label{remark_real-world_applications}

Our framework is closely related to the reinforcement learning from human feedback (RLHF) paradigm~\citep{zhu2023principled, das2024active, li2025provably}. 
In this setting, $\bm x_u^i$ represents a prompt shown to labeler $u$, $(\bm a_u^i, \bm a_{\ u}^{\prime i})$ are two candidate responses, and $y_u^i$ indicates the labeler’s preference over two responses. 
The reward $r_u(\bm x, \bm a)$ reflects the labeler’s underlying evaluation, while $\phi(\bm x_u^i, \bm a_u^i)$ can be interpreted as the output of all but the final layer of a pre-trained language model and $\bm\theta_u$ as the personalized weights in its final layer~\citep{zhu2023principled, park2024rlhf, li2025provably}. 
In this view, the pure offline setting aims to aggregate offline pairwise preference data from multiple labelers to align the base model for the test labeler, whereas the active-data augmented setting focuses on the test labeler by carefully selecting prompt–response pairs based on the offline data. For instance, the learner may target prompts where the model’s responses are more uncertain or diverse, and pair them with contrasting candidate responses, so that the resulting preference feedback provides additional information for refining the user’s preference estimate.

Beyond RLHF, our framework also applies to recommendation systems~\citep{aramayo2023multiarmed, li2010contextual, yan2022dynamic}, where $u$ denotes a user, $\bm x_u^i$ captures contextual information (e.g., time, recommendation category, or interface variant), $(\bm a_u^i, \bm a_{\ u}^{\prime i})$ are two candidate items (such as movies or products), and $y_u^i$ indicates which item was preferred. 
The pure offline case models cold-start recommendation, estimating the test user’s preferences from historical interactions of similar users. 
The active-data augmented setting extends this by interactively querying the user with designed contextual features and item pairs, collecting feedback to improve preference estimation.

\section{Algorithm for \offmodel{}}
\label{sec:pure_offline}

To address the first research question in \Cref{sec:intro} on how to learn cluster structures under fixed and imbalanced offline data without coverage assumptions, we begin with the \offmodel{}. 
In this section, we introduce our algorithm, \textit{Offline Connection-based Clustering of Preference Learning} (\offalgo{}) in \Cref{subsection_offC^2PL}, followed by the theoretical analysis in \Cref{subsection_theoretical_results_al1}. 
We further examine a special case under the commonly adopted \textit{item regularity} assumption (\Cref{assumption_item_regularity}) from the clustering of bandits literature~\citep{gentile2014online, li2018online, wang2023onlinea, li2025demystifying, liu2025offline}, connecting our framework to prior studies.


\subsection{Algorithm Design: \offalgo{}}\label{subsection_offC^2PL}

\begin{algorithm}[t]
\caption{Offline Connection-based Clustering of Preference Learning}
\label{alg_Off-C$^2$PL}
\begin{algorithmic}[1]\label{al1}
\STATE \label{line_input_al1} \textbf{Input:} Test user $u_{t} \in \mathcal{U}$; offline dataset $\mathcal{D} = \bigcup_{u \in \mathcal{U}} \mathcal{D}_u$; parameters $\alpha \geq 1$, $\lambda > 0$, $\delta > 0$, $\kappa>0$, $\hat\gamma \geq 0$; and reference vector $\bm w$.
\STATE \label{line_initialization_al1} \textbf{Initialization:} Construct a null graph $\mathcal{G} = (\mathcal{V}, \emptyset)$ where $\mathcal{V} = \mathcal{U}$. For each user $u \in \mathcal{V}$, compute $\hat{\bm{\theta}}_u$ and $\text{CI}_u$ as in \Cref{align_statistics_initialization_al1}.
\STATE \texttt{// Offline Cluster Learning}
\FOR{each pair of users $u_1, u_2 \in \mathcal{V}$}
  \STATE Connect $(u_1, u_2)$ if the condition in \Cref{align_clutering_condition_for_similarity} holds.
  \label{line_connecting_al1}
\ENDFOR
\STATE Let $\mathcal{G}_{\hat\gamma} = (\mathcal{V}, \mathcal{E}_{\hat\gamma})$ denote the updated graph.
\STATE \texttt{// Data Aggregation}
\FOR{each user $u \in \mathcal{V}$}
  \STATE \label{line_aggregate_data_al1} Aggregate data and update statistics:
  \begin{equation*}
  \mathcal{V}_{\hat\gamma}(u) = \left\{ v \mid (u,v) \in \mathcal{E}_{\hat\gamma} \right\} \cup \{ u \}, \quad
  \tilde M_u = \frac{\lambda}{\kappa} I + \sum_{v \in \mathcal{V}_{\hat\gamma}(u)} \sum_{i=1}^{N_v} \bm z_v^i (\bm z_v^i)^{\top}, \quad
  \tilde N_u = \sum_{v \in \mathcal{V}_{\hat\gamma}(u)} N_v,
  \end{equation*}
  \begin{equation*}
  \tilde{\bm{\theta}}_u = \arg\min_{\bm\theta} \bigg[ - \sum_{v \in \mathcal{V}_{\hat\gamma}(u)} \sum_{i=1}^{N_v} \left( y_v^i \log \sigma(\bm\theta^{\top} \bm z_v^i) + (1 - y_v^i) \log \sigma(-\bm\theta^{\top} \bm z_v^i) \right) + \frac{\lambda}{2} \| \bm\theta \|_2^2 \bigg].
  \end{equation*}
\ENDFOR
\STATE \texttt{// Policy Output}
\STATE Calculate the pessimistic value estimate $\tilde J_{u_{t}}(\pi)$ for any policy $\pi$ as in \Cref{align_calculation_of_value_function_al1}.
\label{line_value_function_calculation_al1}
\STATE \textbf{Output:} $\pi_{u_{t}} = \arg\max_{\pi} \tilde J_{u_t}(\pi)$.
\end{algorithmic}
\end{algorithm}


We detail the procedure of Off-C$^2$PL in Algorithm~\ref{al1}. 
To address scenarios without any coverage assumption, \offalgo{} constructs confidence intervals for each user’s estimated preference vector based on the minimum eigenvalue of the user’s information (Gramian) matrix, enabling reliable confidence estimation even with uneven data coverage across dimensions. 
The algorithm initializes a null graph and connects edges only between users whose estimated preferences are confidently identified as similar, ensuring safe data aggregation. 
To handle binary pairwise feedback $(\bm a_u^i, \bm a_{\ u}^{\prime i}, y_u^i)$ under a logistic model, \offalgo{} adopts a maximum likelihood estimation (MLE) approach, estimating $\hat{\bm\theta}_u$ by minimizing the regularized negative log-likelihood of observed comparisons.

\textbf{Input and Initialization.} The inputs (line~\ref{line_input_al1}) include test user $u_t$, offline dataset $\mathcal{D}=\bigcup_{u\in\mathcal{U}}\mathcal{D}_u$, parameters ($\alpha, \lambda, \delta, \kappa, \hat\gamma$) explained later, and a reference vector $\bm w\in\mathbb{R}^d$ used for theoretical simplification which does not affect the induced policy~\citep{zhu2023principled, li2025provably}. 
The algorithm initializes a null graph $\mathcal{G}$, representing each user in $\mathcal{U}$ as an isolated node (line~\ref{line_initialization_al1}), and then computes key statistics:
\begin{align}\label{align_statistics_initialization_al1}
\begin{split}
\hat{\bm\theta}_u &= \argmin_{\bm\theta}\bigg[ -\sum_{i=1}^{N_u}\left(y_u^i\log\sigma(\bm\theta^{\top}\bm z_u^i)+(1-y_u^i)\log\sigma(-\bm\theta^{\top}\bm z_u^i)\right)+\frac{\lambda}{2}\|\bm\theta\|_2^2\bigg],\\
M_u &= \frac{\lambda}{\kappa}I+\sum_{i=1}^{N_u}\bm z_u^i(\bm z_u^i)^{\top},\,\text{CI}_u = \frac{\sqrt{\lambda\kappa}+2\sqrt{ d\log\big( 1+\frac{4\kappa N_u}{\lambda d}\big)+2\log\left(\frac{2U}{\delta}\right)}}{\kappa\sqrt{\lambda_{\min}(M_u)}}.
\end{split}
\end{align}
Here, $\hat{\bm\theta}_u$ estimates user preferences under pairwise feedback, $M_u$ is a Gramian matrix regularized by $\lambda/\kappa$, and $\text{CI}_u$ denotes the confidence interval constructed based on the minimum eigenvalue of $M_u$, rather than the number of available samples, making it more suitable for scenarios without coverage assumptions in our setting.

\textbf{Offline Cluster Learning.} Unlike traditional online clustering of bandits algorithms~\citep{gentile2014online, li2018online, li2019improved, wang2025online} which typically begin with a complete user graph and iteratively delete edges based on online feedback, our algorithm starts with a null graph $\mathcal{G}$ and incrementally connects users whose preferences are sufficiently similar. This connection-based strategy is better suited to offline settings, where limited data per user make edge deletion unreliable and prone to bias. To determine similarity, we use the key threshold parameter $\hat\gamma$, which controls whether two users should be clustered together. Specifically, as shown in line~\ref{line_connecting_al1}, the algorithm connects two users $u_1$ and $u_2$ if they satisfy:
\begin{align}\label{align_clutering_condition_for_similarity}
\left\|\hat{\bm\theta}_{u_1}-\hat{\bm\theta}_{u_2}\right\|_2<\hat\gamma-\alpha\left(\text{CI}_{u_1}+\text{CI}_{u_2}\right),
\end{align}
where the parameter $\alpha$ controls the conservativeness of clustering: a larger $\alpha$ inflates confidence intervals, making the algorithm less likely to mistakenly cluster users with noisy estimates.
This condition guarantees that the estimated difference between the preference vectors of $u_1$ and $u_2$ remains within the acceptable range $\hat\gamma$ with high confidence (see \Cref{subsection_theoretical_results_al1} for details). In this way, the algorithm only connects users whose behaviors are similar enough under the offline data, progressively building a graph that accurately reflects the underlying cluster structure.

\textbf{Data Aggregation.} Let $\mathcal{G}_{\hat\gamma}$ denote the graph obtained after the cluster learning phase. Based on this graph, the algorithm aggregates data from users who are identified to have similar preferences (line~\ref{line_aggregate_data_al1}). Specifically, we define $\mathcal{V}_{\hat\gamma}(u)$ as the set containing user $u$ and its one-shot neighbors, representing all users estimated to share similar preferences with $u$. Using this set, the algorithm constructs the aggregated Gramian matrix $\tilde M_u$ by combining samples from all users in $\mathcal{V}_{\hat\gamma}(u)$ and calculates the total number of samples $\tilde N_u$ within this set. The preference estimate for user $u$ is then refined by applying MLE to the aggregated data, yielding $\tilde{\bm\theta}_u$.

\textbf{Policy Output.} In the final step, the algorithm computes a pessimistic estimate~\citep{jin2021pessimism, rashidinejad2021bridging, li2022pessimism} of the value function for any policy $\pi$ for the test user $u_t$ which downweights underrepresented dimensions and emphasizes directions with sufficient data coverage, thereby mitigating the risk of overestimating performance in poorly explored dimensions:
\begin{align}\label{align_calculation_of_value_function_al1}
\tilde J_{u_{t}}(\pi) = \left( \mathbb{E}_{\bm x \sim \rho_p} \left[ \phi(\bm x, \pi(\bm x)) \right] - \bm w \right)^{\top} \tilde{\bm\theta}_{u_{t}} 
- \tilde\beta_{u_{t}} \left\| \mathbb{E}_{\bm x \sim \rho_p} \left[ \phi(\bm x, \pi(\bm x)) \right] - \bm w \right\|_{ \tilde M_{u_{t}}^{-1} },
\end{align}
where the confidence term $\tilde{\beta}_{u} = \Big( 2 \sqrt{ d \log \big( 1 + \frac{4 \tilde N_{u} \kappa}{\lambda d} \big) + 2 \log \left( \frac{2U}{\delta} \right) } + \sqrt{ \lambda \kappa } \Big)\big/{ \kappa }$
accounts for estimation uncertainty (line~\ref{line_value_function_calculation_al1}). The algorithm outputs the final policy $\pi_{u_t}$ that maximizes a pessimistic estimate $\tilde J_{u_t}(\pi)$, following the principle of pessimism in offline learning~\citep{jin2021pessimism, li2022pessimism}. This estimate is designed to down-weight underrepresented dimensions and prioritize actions in regions of the feature space where the data provides more reliable information. Note that obtaining the exact $\pi_{u_t}$ in \Cref{al1} requires an exhaustive search, which is feasible for small context and action spaces $\mathcal{X}$ and $\mathcal{A}$. For large-scale settings, one can instead employ policy optimization methods such as PPO~\citep{schulman2017proximal, das2024active} to efficiently approximate $\pi_{u_t}$.

\subsection{Theoretical Results for Algorithm \ref{al1}}\label{subsection_theoretical_results_al1}


We present the theoretical results for Algorithm~\ref{al1} (Off-C$^2$PL), with detailed proofs in \Cref{appendix_detailed_proofs} and key notations summarized in \Cref{tab:notation_neighbor_sets}. 
\Cref{lemma_confidence_ellipsoid_of_hat_theta} bounds the estimation error of each user’s preference vector $\hat{\bm\theta}_u$ based on individual data (line~\ref{line_initialization_al1}); 
\Cref{lemma_cardinality_of_R_and_W} characterizes the homogeneous and heterogeneous neighbor sets ($\mathcal{R}_{\hat\gamma}(u)$ and $\mathcal{W}_{\hat\gamma}(u)$), quantifying data aggregation quality in the learned graph $\mathcal{G}_{\hat\gamma}$; 
and \Cref{lemma_confidence_ellipsoid_of_tilde_theta} extends this analysis to the aggregated estimator $\tilde{\bm\theta}_u$ (line~\ref{line_aggregate_data_al1}). 
Finally, \Cref{th1} provides the main suboptimality bound. We begin by introducing the minimum heterogeneity gap between different clusters in \Cref{definition_minimum_heterogeneity_gap}.


\begin{table}[h]
\centering
\caption{Summary of neighbor set notations.}\label{tab:notation_neighbor_sets}
\renewcommand{\arraystretch}{1.1}
\resizebox{\textwidth}{!}{%
\begin{tabular}{@{}ccc@{}}
\toprule
\textbf{Notation} & \textbf{Definition} & \textbf{Interpretation} \\
\midrule
\large$\mathcal{V}_{\hat\gamma}(u)$ & 
\large$\{u\} \cup \{v \mid (u,v) \in \mathcal{E}_{\hat\gamma}\}$ & 
Set containing user $u$ and all its neighbors in the graph $\mathcal{G}_{\hat\gamma}$. \\
\midrule
\large$\mathcal{R}_{\hat\gamma}(u)$ & 
\large$\{v \mid v \in \mathcal{V}_{\hat\gamma}(u), \bm\theta_u = \bm\theta_v\}$ & 
\begin{tabular}[c]{@{}c@{}}Set of \emph{homogeneous neighbors} of $u$, i.e., users in $\mathcal{V}_{\hat\gamma}(u)$ sharing the same preference vector.\\
Their data can be safely aggregated with $u$’s without introducing bias.\end{tabular} \\
\midrule
\large$\mathcal{W}_{\hat\gamma}(u)$ & 
\large$\{v \mid v \in \mathcal{V}_{\hat\gamma}(u), \bm\theta_u \neq \bm\theta_v\}$ & 
\begin{tabular}[c]{@{}c@{}}Set of \emph{heterogeneous neighbors} of $u$, i.e., users in $\mathcal{V}_{\hat\gamma}(u)$ with different preference vectors.\\
Aggregating their data with $u$’s may introduce bias and should be carefully controlled.\end{tabular} \\
\bottomrule
\end{tabular}
}
\end{table}

\begin{definition}[Minimum Heterogeneity Gap]\label{definition_minimum_heterogeneity_gap}
The preference vectors of users from different clusters are separated by at least a gap of $\gamma$. Specifically, for any two users $u$ and $v$ belonging to different clusters (i.e., $j_u \neq j_v$), it holds that \( \left\| \bm\theta_u - \bm\theta_v \right\|_2 \geq \gamma.\)
\end{definition}

\begin{lemma}[Confidence Ellipsoid of $\hat{\bm\theta}_u$]\label{lemma_confidence_ellipsoid_of_hat_theta}
For any user $u$, under the initialization in \Cref{align_statistics_initialization_al1} with $\kappa=1/(2+e^2+e^{-2})$, it holds with probability at least $1 - \delta$ that
\begin{align*}
\left\| \hat{\bm\theta}_u - \bm\theta_u \right\|_2
\leq
\frac{ \sqrt{ \lambda \kappa } + 2 \sqrt{ 2 \log \left( \frac{1}{\delta} \right) + d \log \left( 1 + \frac{4 N_u \kappa}{ d \lambda } \right) } }{ \kappa \sqrt{ \lambda_{\min}(M_u) } }.
\end{align*}
\end{lemma}

\Cref{lemma_confidence_ellipsoid_of_hat_theta} provides a high-probability bound on the estimation error of $\hat{\bm\theta}_u$ defined in \Cref{align_statistics_initialization_al1}, which guarantees that the estimation error for each $\hat{\bm\theta}_u$ is controlled by the minimum eigenvalue of the information matrix $M_u$ for user $u$. Note that the estimate $\tilde{\bm\theta}_u$ is obtained by aggregating all data from users in the neighborhood $\mathcal{V}_{\hat\gamma}(u)$, which includes both homogeneous and heterogeneous neighbors. Since Algorithm~\ref{al1} relies on $\tilde{\bm\theta}_u$ to determine the final policy, it is crucial to analyze the cardinality of both sets $\mathcal{R}_{\hat\gamma}(u)$ and $\mathcal{W}_{\hat\gamma}(u)$, since the former provides additional homogeneous samples that help reduce the estimation error, while the latter may introduce biased samples that can sometimes adversely affect the estimate. We formalize this in the following lemma:

\begin{lemma}[Cardinality of $\mathcal{R}_{\hat\gamma}(u)$ and $\mathcal{W}_{\hat\gamma}(u)$]\label{lemma_cardinality_of_R_and_W}
Let parameter inputs in Algorithm \ref{al1} satisfy $\alpha \geq 1$, $\lambda$ and $\delta$ be such that \(
\lambda \leq 2 \log\left( \tfrac{2U}{\delta} \right) + d \log\big( 1 + \tfrac{4\kappa\, \min_v \{ N_v \} }{ d \lambda } \big),
\
\delta \leq \tfrac{ d \lambda }{ 4 \kappa\, \min_v \{ N_v \} + d \lambda }
\), and $\kappa=1/(2+e^2+e^{-2})$.
Define $\varepsilon = \hat\gamma - \gamma$ as the gap between the selected clustering threshold and the true minimum heterogeneity gap. Then there exist some $\alpha_r \in \Big( \frac{\kappa}{3(\alpha+1)\sqrt{2\max\{2,d\}\log(2U/\delta)}}, \frac{\kappa}{2(\alpha - 1)\sqrt{2\log(2U/\delta)}} \Big)$ and $\alpha_w \in \Big( 0, \frac{\kappa}{2(\alpha-1)\sqrt{2\log(2U/\delta)}} \Big)$ such that for any user $u$, with probability at least $1-\delta$, the cardinalities of the homogeneous and heterogeneous neighbor sets can be characterized as:
\begin{align}
& \mathcal{R}_{\hat\gamma}(u) 
= \Big\{ v \,\Big|\, \bm\theta_u = \bm\theta_v 
\text{\emph{ and} }
\frac{1}{ \sqrt{ \lambda_{\min}(M_u) } } + \frac{1}{ \sqrt{ \lambda_{\min}(M_v) } }
< \alpha_r \hat\gamma 
\Big\} \cup \{ u \}, \label{align_cardinality_of_R}
\\& \mathcal{W}_{\hat\gamma}(u)
= \left\{ v \,\left|\, \gamma\leq
\| \bm\theta_u - \bm\theta_v \|_2 < \hat\gamma 
\text{\emph{ and} }
\frac{1}{ \sqrt{ \lambda_{\min}(M_u) } } + \frac{1}{ \sqrt{ \lambda_{\min}(M_v) } }
< \alpha_w \varepsilon \right.
\right\}. \label{align_cardinality_of_W}
\end{align}
\end{lemma}

In \Cref{lemma_cardinality_of_R_and_W}, the notation $\lambda_{\min}(M_u)$, denoting the minimum eigenvalue of the information matrix $M_u$, quantifies the sufficiency of data in the dataset $\mathcal{D}_u$. Since the preference vector $\bm\theta_u$ lies in $\mathbb{R}^d$, the dataset must provide adequate coverage along each dimension to ensure a sufficiently large $\lambda_{\min}(M_u)$, i.e., an informative Gramian matrix. 

By definition, $\mathcal{R}_{\hat\gamma}(u)$ consists of user $u$ and its homogeneous neighbors, indicating those samples that are beneficial for accurately estimating the true preference vector $\tilde{\bm\theta}_u$. The first condition in \Cref{align_cardinality_of_R} ensures the homogeneity of users within $\mathcal{R}_{\hat\gamma}(u)$, while the second condition shows that only when both $u$ and $v$ have sufficiently informative data can $v$ be identified as a reliable neighbor. Moreover, the right-hand side of \Cref{align_cardinality_of_R} depends linearly on $\hat\gamma$, implying that increasing the clustering threshold $\hat\gamma$ allows more homogeneous neighbors to be included. On the other hand, $\mathcal{W}_{\hat\gamma}(u)$ captures the heterogeneous neighbors of $u$, which may introduce bias. The first condition in \Cref{align_cardinality_of_W} shows that only users with a preference difference smaller than $\hat\gamma$ may be mistakenly clustered together, while the second condition imposes a stricter data sufficiency requirement for these heterogeneous neighbors. Notably, since $\varepsilon = \hat\gamma - \gamma$, the required information level for heterogeneous connections is more stringent than that for homogeneous ones.

With \Cref{lemma_cardinality_of_R_and_W}, we can now bound the estimation error of $\tilde{\bm\theta}_u$ in terms of the total number of aggregated samples, denoted by $N_{\mathcal{V}_{\hat\gamma}(u)}$, and the number of samples coming from heterogeneous neighbors, denoted by $N_{\mathcal{W}_{\hat\gamma}(u)}$. This is formalized in the following lemma.

\begin{lemma}[Confidence Ellipsoid of $\tilde{\bm\theta}_u$]
\label{lemma_confidence_ellipsoid_of_tilde_theta}
For any user $u$, under the data aggregation step of Algorithm~\ref{al1} and the same conditions as in \Cref{lemma_cardinality_of_R_and_W}, it holds with probability at least $1 - \delta$ that
\[
\left\| \tilde{\bm\theta}_u - \bm\theta_u \right\|_{ \tilde M_u }
\leq 
\frac{ \sqrt{ \lambda \kappa } + 2 \sqrt{ 2 \log\left( \frac{2U}{\delta} \right) + d\, \log\Big( 1 + \frac{4 \kappa N_{ \mathcal{V}_{\hat\gamma}(u) } }{ d \lambda } \Big) } }{ \kappa }
+ \frac{ \hat\gamma\, \sqrt{ d\, N_{ \mathcal{W}_{\hat\gamma}(u) } } }{ 2 }.
\]
\end{lemma}

\Cref{lemma_confidence_ellipsoid_of_tilde_theta} shows that the estimation error of $\tilde{\bm\theta}_u$ with respect to the information matrix built from its local neighborhood in $\mathcal{G}_{\hat\gamma}$ can be decomposed into two sources: the noise term (the first term), which captures the randomness due to finite samples, and the bias term (the second term), which reflects the heterogeneity arising from including neighbors in $\mathcal{W}_{\hat\gamma}(u)$. Building on this result, our first theorem characterizes the suboptimality gap of Algorithm~\ref{al1} in the offline setting.

\begin{theorem}\label{th1}
Under the same conditions as in \Cref{lemma_cardinality_of_R_and_W}, the suboptimality gap of Algorithm~\ref{al1} for any test user $u_t$ can be bounded with probability at least $1-\delta$ as:
\begin{align}
\textnormal{SubOpt}_{u_{t}}(\pi_{u_{t}}) 
& \leq \tilde O\left(\sqrt{d} \left(1 + \hat\gamma \sqrt{N_{\mathcal{W}_{\hat\gamma}(u_t)}} \right)
\, \Big\| \mathbb{E}_{\bm x\sim\rho_p}[\phi(\bm x,\pi_{u_{t}}(\bm x))] - \bm w \Big\|_{\tilde M_{u_t}^{-1}} \right) \label{align_first_form_of_th1} 
\\ & \leq \tilde O\left( \frac{ \sqrt d \left( 1 + \hat\gamma \sqrt{N_{\mathcal{W}_{\hat\gamma}(u_{t})}} \right) }{ \sqrt{ \lambda_{\min}(\tilde M_{u_{t}}) } } \right), \label{align_second_form_of_th1} 
\end{align}
where $\tilde O$ hides absolute constants and logarithmic factors. The matrix \(
\tilde M_{u_t} = \frac{\lambda}{\kappa} I + \sum_{v \in \mathcal{V}_{\hat\gamma}(u_t)} \sum_{i=1}^{N_v} \bm z_v^i (\bm z_v^i)^{\top}\)
denotes the information matrix constructed from the aggregated data of all users in $\mathcal{V}_{\hat\gamma}(u_t)$. Furthermore, when the threshold satisfies $\hat\gamma \leq \gamma$, the heterogeneous set becomes empty according to \Cref{lemma_cardinality_of_R_and_W}, i.e., $\mathcal{W}_{\hat\gamma}(u_t) = \emptyset$, and the suboptimality bound simplifies to
\begin{align}\label{align_special_form_of_th1}
\textnormal{SubOpt}_{u_t}(\pi_{u_t})
&\leq \tilde O\left( 
\sqrt{d}\big\|\mathbb{E}_{\bm x \sim \rho_p}[\phi(\bm x,\pi_{u_t}(\bm x))] - \bm w \big\|_{\tilde M_{u_t}^{-1}}
\right)
\leq 
\tilde O\!\left(
\sqrt{\frac{d}{\lambda_{\min}(\tilde M_{u_t})}}
\right).
\end{align}
\end{theorem}

The suboptimality gap in \Cref{align_first_form_of_th1} of \Cref{th1} consists of the product of two interpretable terms. The first term, $\sqrt{d}\left(1+\hat\gamma\sqrt{N_{\mathcal{W}_{\hat\gamma}(u_t)}}\right)$, can be further decomposed into two parts. Up to logarithmic factors, the first part can be seen as a fundamental term that arises from the inherent sample noise and reflects the baseline statistical uncertainty. The second part, which grows linearly with $\hat\gamma$ and sublinearly with the number of samples from heterogeneous neighbors $N_{\mathcal{W}_{\hat\gamma}(u_t)}$, captures the bias introduced by potential heterogeneity from the neighbors. As represented in \Cref{lemma_cardinality_of_R_and_W}, choosing a larger $\hat\gamma$ can increase $N_{\mathcal{W}_{\hat\gamma}(u_t)}$, thereby amplifying this bias term. 

The second term in \Cref{align_second_form_of_th1}, $\Big\| \mathbb{E}_{\bm x\sim\rho_p}[\phi(\bm x,\pi_{u_{t}}(\bm x))] - \bm w \Big\|_{\tilde M_{u_t}^{-1}}$, is known as the \emph{concentratability coefficient}, a concept widely used in offline learning and policy evaluation~\citep{jin2021pessimism, zhu2023principled, li2025provably}. This term quantifies the mismatch between the context-action distribution induced by the optimal policy and the distribution supported by the offline data from $u_t$ and its neighbors in the resulting graph $\mathcal{G}_{\hat\gamma}$. A smaller concentratability coefficient implies that the offline data provides better coverage of the distribution under the optimal policy. Furthermore, choosing the reference vector $\bm w$ as a representative feature (e.g., the most frequent feature vector $\phi$ observed in the data)~\citep{zhu2023principled, li2025provably} aligns the concentratability term with the data-supported subspace, leading to a tighter suboptimality.

The dependence on $\tilde M_{u_t}$ in \Cref{align_second_form_of_th1} through its minimum eigenvalue indicates that the overall sample efficiency is constrained by how well the data covers each dimension of the parameter space. Specifically, $\lambda_{\min}(\tilde M_{u_t})$ appearing in the denominator implies that, in the worst case, the effective number of samples per dimension is determined by the least informative direction. As indicated in \Cref{lemma_cardinality_of_R_and_W}, increasing $\hat\gamma$ can expand the neighborhood $\mathcal{V}_{\hat\gamma}(u_t)$, enlarging $\tilde M_{u_t}$ and potentially improving coverage, though at the cost of introducing more heterogeneity bias.

\begin{remark}[Selection of $\hat\gamma$]\label{remark_selection_of_hatgamma}
As shown in \Cref{lemma_cardinality_of_R_and_W}, the cardinalities of both $\mathcal{R}_{\hat\gamma}(u)$ and $\mathcal{W}_{\hat\gamma}(u)$ depend critically on the choice of $\hat\gamma$. Increasing $\hat\gamma$ generally enlarges both sets: a larger $\mathcal{R}_{\hat\gamma}(u)$ provides more homogeneous samples that can improve the accuracy of estimating $\bm\theta_u$, whereas a larger $\mathcal{W}_{\hat\gamma}(u)$ may introduce greater bias due to the inclusion of heterogeneous neighbors (as analyzed in \Cref{lemma_confidence_ellipsoid_of_tilde_theta} and \Cref{th1}). Therefore, careful selection of $\hat\gamma$ is crucial. Notably, \Cref{align_special_form_of_th1} shows that choosing $\hat\gamma \leq \gamma$ simplifies the suboptimality bound to a bias-free form. This provides a practical strategy to avoid large bias when a lower bound of $\gamma$ is available, but at the cost of reducing $\mathcal{R}_{\hat\gamma}(u_t)$ and thus increasing the noise due to fewer aggregated samples.
Due to space limitations, we defer detailed guidelines on selecting this parameter in practice to Appendix \ref{appendix_discussions_on_hatgamma}.
\end{remark}

\begin{remark}[Comparison with Single User Case]\label{remark_comparison_with_single_user_case}
When we choose $\hat\gamma = 0$, Algorithm~\ref{al1} reduces to the special case where no clustering is learned and only the data from the test user, $\mathcal{D}_{u_t}$, is used for estimation. In this scenario, the bound in \Cref{th1} specializes to \(\tilde O \left( \sqrt{d} \Big\| \mathbb{E}_{\bm x \sim \rho_p} [\phi(\bm x, \pi_{u_t}(\bm x))] - \bm w \Big\|_{M_{u_t}^{-1}} \right),\)
which matches the suboptimality bound derived for the single-user case in previous works~\citep{zhu2023principled, li2025provably}. 
\end{remark}

\begin{remark}[Discussions on Parameter $\kappa$] 
The input parameter $\kappa$ in Algorithm \ref{al1} serves as a non-linearity coefficient, lower bounding the minimum slope of the sigmoid function, i.e.,
\begin{align}\label{align_kappa}
\min_{\substack{(\bm x,\bm a,\bm a')\in\mathcal{X}\times \mathcal{A}\times\mathcal{A},\ \bm\theta\in\Theta}}
\nabla\sigma\left( \phi(\bm x,\bm a)^{\top}\bm\theta - \phi(\bm x,\bm a')^{\top}\bm\theta \right) \;\geq\; \kappa > 0.
\end{align} 
In our setting, $\kappa$ can be safely fixed to the constant $1/(2+e^2+e^{-2})$, which guarantees the validity of our theoretical results (e.g., \Cref{th1}). This is because we assume $\|\bm\theta_u\|_2 \leq 1$ and $\|\phi(\bm x,\bm a)\|_2 \leq 1$, following prior works on contextual logistic bandits~\citep{chen2020dynamic, oh2019thompson, lee2024nearly} and clustering of bandits literature~\citep{gentile2014online, wang2023onlinea, wang2025online, liu2025offline}. 
In more general scenarios where the $\ell_2$-norm of either $\bm\theta_u$ or $\phi(\bm x,\bm a)$ is not bounded by a constant, the margin can become arbitrarily large, and $1/\kappa$ may grow exponentially. In such cases, as shown in \Cref{lemma_confidence_ellipsoid_of_tilde_theta} and \Cref{appendix_proof_of_th1} (proof of \Cref{th1}), our suboptimality bound scales linearly with $1/\kappa$. By contrast, prior work in the single-user setting exploits mirror-descent techniques to improve this dependence to $1/\sqrt{\kappa}$~\citep{li2025provably}, which is argued to be tight~\citep{das2024active, li2025provably}. Extending this improved $\sqrt{\kappa}$ dependence to our heterogeneous multi-user setting with clustering remains an interesting open problem. 
\end{remark}

\subsection{Further Results and Comparisons under Item Regularity Assumption}\label{subsection_item_regularity}

In the traditional clustering of bandits literature~\citep{gentile2014online, li2018online, wang2023onlinea, dai2024conversational, wang2025online, liu2025offline}, a common assumption is that the offline datasets provide sufficient coverage across all dimensions of the preference vector. This condition ensures that the information matrix is well-conditioned, which is crucial for accurate estimation. We first introduce this standard requirement, known as the \emph{item regularity assumption}, and then discuss how our algorithm and theoretical results change under this setting.

\begin{assumption}[Item Regularity]\label{assumption_item_regularity}
Let $\rho$ be a distribution over $\{(\bm x,\bm a,\bm a')\in\mathcal{X}\times\mathcal{A}\times\mathcal{A} : \|\phi(\bm x,\bm a)\|_2 \leq 1,\, \|\phi(\bm x,\bm a')\|_2 \leq 1\}$ where coveriance matrix \(\mathbb{E}_{(\bm x,\bm a,\bm a')\sim\rho_a}\!\left[(\phi(\bm x,\bm a) - \phi(\bm x,\bm a'))(\phi(\bm x,\bm a) - \phi(\bm x,\bm a'))^{\top}\right]\) is full rank with minimum eigenvalue $\lambda_a > 0$. For any fixed unit vector $\bm\theta \in \mathbb{R}^d$, the random variable $\left(\bm\theta^{\top}(\phi(\bm x,\bm a) - \phi(\bm x,\bm a'))\right)^2$, with $(\bm x, \bm a, \bm a')\sim\rho$, has sub-Gaussian tails with variance upper bounded by $\sigma^2$. Each context-action pair $(\bm x_u^i,\bm{a}_u^i,\bm a_{\ u}^{\prime i})$ in $\mathcal{D}_u$ is selected from a finite candidate set $\mathcal{S}_u^i$ with size $|\mathcal{S}_u^i|\leq S$ for any $i\in[N_u]$, where the actions in $\mathcal{S}_u^i$ are independently drawn from $\rho$. 
Moreover, we assume the \textit{smoothed regularity parameter} \(\tilde\lambda_a=\int_0^{\lambda_a}\Big( 1 - e^{-\frac{(\lambda_a-x)^2}{2\sigma^2}} \Big)^S \mathrm{d}x\)  
is known to the algorithm.
\end{assumption}

\Cref{assumption_item_regularity} ensures that the data distribution is sufficiently rich to provide informative samples in all directions of the preference vector $\bm\theta_u$. This assumption is especially relevant when offline data are collected from finite action spaces with bounded size, such as datasets generated by logging policies in online bandits~\citep{dudik2015contextual, wang2025online}. Under this condition, preference estimates become accurate once enough data are observed, since the minimum eigenvalue of the information matrix grows directly with the number of samples. Consequently, our confidence bounds decrease with the amount of offline data rather than depending solely on the minimum eigenvalue itself. \Cref{lemma_cardinality_under_item_regularity} summarizes the modified clustering conditions and resulting characterizations.

\begin{lemma}[Extension of \Cref{lemma_cardinality_of_R_and_W}]\label{lemma_cardinality_under_item_regularity}
Under \Cref{assumption_item_regularity}, replace the confidence interval by
\(
\textnormal{CI}_u = \Big(\sqrt{\lambda\kappa} + 2\sqrt{\, d\, \log\big( 1 + \frac{4\kappa N_u}{\lambda d} \big) + 2\log\left( \frac{2U}{\delta} \right) }\Big)\Big/\Big(\kappa\sqrt{\tilde\lambda_a N_u/2}\Big),
\)
and adjust the condition in \Cref{align_clutering_condition_for_similarity} to:
\[
\big\|\hat{\bm\theta}_{u_1} - \hat{\bm\theta}_{u_2}\big\|_2 < \hat\gamma - \alpha(\textnormal{CI}_{u_1} + \textnormal{CI}_{u_2})
\quad \text{and} \quad 
\min\{N_{u_1},\, N_{u_2}\} \geq N_{\min},
\]
where $N_{\min} = \frac{16}{\tilde\lambda_a^2}\log\left( \frac{8Ud}{\tilde\lambda_a^2\delta} \right)$. All other conditions remain as in \Cref{lemma_cardinality_of_R_and_W}. Then there exist some $\alpha_r' \in \left( \frac{\kappa\sqrt{\tilde\lambda_a}}{3(\alpha+1)\sqrt{\max\{2,d\}\log(2U/\delta)}}, \frac{\kappa\sqrt{\tilde\lambda_a}}{2(\alpha - 1)\sqrt{2\log(2U/\delta)}} \right)$ and $\alpha_w' \in \left( 0, \frac{\kappa\sqrt{\tilde\lambda_a}}{2(\alpha-1)\sqrt{\log(2U/\delta)}} \right)$ such that the cardinalities of $\mathcal{R}_{\hat\gamma}(u)$ and $\mathcal{W}_{\hat\gamma}(u)$ are given by:
\begin{align}
\mathcal{R}_{\hat\gamma}(u) 
&= 
\begin{cases}
\Big\{
v\,\Big|\,
\bm\theta_u = \bm\theta_v,\;
\frac{1}{\sqrt{N_u}} + \frac{1}{\sqrt{N_v}} < \alpha_r' \hat\gamma,\;
N_v \geq N_{\min}
\Big\} \cup \{u\}, & N_u \geq N_{\min} \\ 
\{u\}, & \text{otherwise}
\end{cases}, 
\label{align_cardinality_of_R_regularity}
\\
\mathcal{W}_{\hat\gamma}(u)
&=
\begin{cases}
\Big\{
v\,\Big|\,
\gamma \leq \|\bm\theta_u - \bm\theta_v\|_2 < \hat\gamma,\;
\frac{1}{\sqrt{N_u}} + \frac{1}{\sqrt{N_v}} < \alpha_w' \varepsilon
\Big\}, & N_u \geq N_{\min} \\
\emptyset, & \text{otherwise}
\end{cases}.
\label{align_cardinality_of_W_regularity}
\end{align}
\end{lemma}

The expressions above show that, under \Cref{assumption_item_regularity}, the ability to correctly identify homogeneous and heterogeneous neighbors depends explicitly on the sample size rather than the conditioning of the Gramian matrix. This aligns with the results in standard offline clustering of bandits frameworks~\citep{liu2025offline}. Below we present \Cref{th1_item_regularity}, which characterizes the suboptimality of our algorithm under \Cref{assumption_item_regularity}.

\begin{corollary}\label{th1_item_regularity}
Under the same conditions as in \Cref{lemma_cardinality_under_item_regularity}, the suboptimality of the algorithm is bounded with probability at least $1-\delta$ as:
\[
\textnormal{SubOpt}_{u_t}(\pi_{u_t}) \leq 
\tilde O\left( \sqrt\frac{d}{\tilde\lambda_a} \left( \sqrt{ \frac{1}{ N_{\mathcal{V}_{\hat\gamma}(u_t)} } } 
+ \hat\gamma \sqrt{\eta_{ \mathcal{W}_{\hat\gamma}(u_t) }}
\right)\right),
\]
where 
$\eta_{ \mathcal{W}_{\hat\gamma}(u_t) }
= 
\frac{ N_{ \mathcal{W}_{\hat\gamma}(u_t) } }{ N_{ \mathcal{V}_{\hat\gamma}(u_t) } }$
denotes the fraction of samples from heterogeneous neighbors among all samples aggregated for $u_t$ in the graph $\mathcal{G}_{\hat\gamma}$.
\end{corollary}

\Cref{th1_item_regularity} takes a form similar to the suboptimality bounds in classical offline clustering of bandits~\citep{liu2025offline}. Specifically, the term $\sqrt{1/N_{\mathcal{V}_{\hat\gamma}(u_t)}}$ captures the \emph{noise}, arising from the inherent variance in estimating the preference vector. This term decreases as the number of aggregated samples $N_{\mathcal{V}_{\hat\gamma}(u_t)}$ increases, implying that a larger $\hat\gamma$, which connects more users, reduces the noise. In contrast, the term $\hat\gamma\sqrt{\eta_{\mathcal{W}_{\hat\gamma}(u_t)}}$ captures the \emph{bias}, introduced by aggregating data from neighbors whose preferences differ from $u_t$. This bias grows linearly with $\hat\gamma$ and depends on the fraction of heterogeneous samples included. Thus, while increasing $\hat\gamma$ reduces noise, it also risks introducing greater bias. This tradeoff underscores the importance of carefully tuning $\hat\gamma$ to balance sample efficiency with robustness against heterogeneity, as discussed in \Cref{remark_selection_of_hatgamma}. Finally, the scaling factor $\sqrt{d\big/\tilde\lambda_a}$ arises from \Cref{assumption_item_regularity}, reflecting that each offline sample contributes only partial information across dimensions. As a result, the overall suboptimality must be scaled by $\sqrt{d/\tilde\lambda_a}$ to capture performance across all preference dimensions.

\begin{remark}[Robustness of Algorithm \ref{al1}]\label{remark_robustness_of_al1}
As noted in prior works on clustering of bandits~\citep{wang2023onlinea, dai2024conversational}, it can be restrictive to assume that users within the same cluster share exactly identical preferences, as small gaps may exist even among users with similar backgrounds. To address this, those works developed additional algorithms to handle intra-cluster bias, often based on edge-deletion strategies~\citep{gentile2014online, li2018online}. In contrast, we argue that our proposed Algorithm~\ref{al1} is inherently robust to such cases. Specifically, when small preference gaps exist within a cluster, the setting can be interpreted as if each user forms its own cluster (i.e., $U=J$). In this case, $\mathcal{R}_{\hat\gamma}(u)=\{u\}$ in \Cref{lemma_cardinality_of_R_and_W}, while other users with similar (though not identical) preferences may be included in $\mathcal{W}_{\hat\gamma}(u)$ when $\hat\gamma$ is chosen larger than this gap, provided their data sufficiency satisfies the second condition in \Cref{align_cardinality_of_W} or \Cref{align_cardinality_of_W_regularity}. According to \Cref{th1} and \Cref{th1_item_regularity}, such users still contribute to the aggregated information matrix $\tilde M_{u_t}$ and to the neighbor set $\mathcal{V}_{\hat\gamma}(u_t)$ which helps decrease noise with some additional bias, reflected in larger $N_{\mathcal{W}_{\hat\gamma}(u_t)}$ and thus $\eta_{\mathcal{W}_{\hat\gamma}(u_t)}$ (noting that $\eta_{\mathcal{W}_{\hat\gamma}(u_t)} \leq 1$ always holds). Therefore, in practice, when small intra-cluster gaps exist, it is often preferable to select a relatively small $\hat\gamma$ to better control the bias.
\end{remark}

\section{Algorithm for \hybmodel{}}\label{sec:active_data_selection}

In \Cref{sec:pure_offline}, we analyzed the algorithm designed for clustering-based preference learning under the pure offline setting. However, as shown in \Cref{th1}, a key limitation of the pure offline case is its reliance on the distribution of the available datasets. More specifically, if the data collected from a user's neighbors fail to adequately cover the distribution induced by the optimal policy, the resulting concentratability coefficient may become large, which can significantly degrade performance. This phenomenon corresponds to the second research question introduced in \Cref{sec:intro}: \textit{how to mitigate the impact of insufficient coverage in offline datasets}.
In many real-world applications, it is often feasible to collect a small amount of additional online or interactive data to complement existing offline datasets.
Motivated by this, we extend the our offline algorithm in \Cref{sec:pure_offline} to the \hybmodel{} defined in \Cref{sec: problem_setting}, which aims to address the distributional limitation challenge of the \offmodel{} by combining offline clustering with active-data augmentation.

\begin{algorithm}[tbhp]
\caption{Active-data Augmented - Offline Connection-based Clustering of Preference Learning}
\label{alg_PersonalizedRLHF}
\begin{algorithmic}[1]\label{al2}
\STATE \textbf{Input:} Test user $u_t \in \mathcal{U}$, offline dataset $\mathcal{D} = \bigcup_{u \in \mathcal{U}} \mathcal{D}_u$, and online rounds $N$; Graph $\mathcal{G}_{\hat\gamma}$, neighbor set $\mathcal{V}_{\hat\gamma}(u_{t})$, aggregated Gramian matrix $\tilde M_{u_t}$, and initial preference estimate $\tilde{\bm\theta}_{u_t}$ from Algorithm~\ref{al1} \label{line_input_al2}.
\STATE \textbf{Initialization:} Set $\tilde M_{u_{t}}^0 \leftarrow \tilde M_{u_{t}}$ and $\tilde{\bm\theta}_{u_{t}}^0 \leftarrow \tilde{\bm\theta}_{u_{t}}$.\label{line_initialization_al2}
\STATE \texttt{// Active-data Augmentation}
\FOR{$n = 1, \dots, N$}
    \STATE Select $\left(\mathring{\bm x}_{u_{t}}^n,\, \mathring{\bm a}_{u_{t}}^n,\, \mathring{\bm a}_{\ u_t}^{\prime n}\right)$ according to \Cref{align_active_data_selection}.
    \STATE Observe feedback $\mathring{y}_{u_{t}}^n$.
    \STATE Compute $\mathring{\bm z}_{u_{t}}^n = \phi(\mathring{\bm x}_{u_{t}}^n, \mathring{\bm a}_{u_{t}}^n) - \phi(\mathring{\bm x}_{u_{t}}^n, \mathring{\bm a}_{\ u_t}^{\prime n})$.
    \STATE Update \(\tilde M_{u_{t}}^n = \tilde M_{u_{t}}^{n-1} + \mathring{\bm z}_{u_{t}}^n \left( \mathring{\bm z}_{u_{t}}^n \right)^{\top}\) and \(\tilde{\bm\theta}_{u_{t}}^n\) \text{as in} \Cref{align_calculation_of_tilde_theta_t}.
\ENDFOR
\STATE \texttt{// Policy Output}
\STATE Construct \(\overline{\bm\theta}_{u_t}\) as \Cref{align_constructing_overline_theta}.
\STATE \textbf{Output:} 
\(\pi_{u_{t}}(\bm x) = \argmax_{\bm a \in \mathcal{A}} \phi(\bm x, \bm a)^{\top} \overline{\bm\theta}_{u_{t}}.\)
\end{algorithmic}
\end{algorithm}

\subsection{Algorithm Design: \hybalgo{}}

We now introduce our algorithm for the \hybmodel{}, which extends the cluster structure learned in \offalgo{} (\Cref{al1}). 
Recall from \Cref{sec: problem_setting} that in \hybmodel{}, the learner can interact with the environment for a limited number of rounds to collect additional feedback. Specifically, it is allowed to select $N$ rounds of active data for the target user $u_t$ to mitigate the poor coverage of the offline datasets.
We refer our algorithm in this setting as \textit{Active-data Augmented - Offline Connection-based Clustering of Preference Learning} (\hybalgo{}). 
The core idea of \hybalgo{} is to actively select $N$ rounds of data for the test user to complement the offline data by improving the coverage of the feature space (e.g. in conversational recommendation systems the website adopts $N$ rounds of further dialogues to identify the users' preferences).
Since the clustering structure has been learned offline, the active-data augmentation phase should be based on the aggregated Gramian matrix $\tilde M_{u_t}$, which summarizes the information from the test user's neighborhoods. As shown by the suboptimality bound in \Cref{th1}, the estimation error is largely determined by the minimum eigenvalue of $\tilde M_{u_t}$. Therefore, the goal of this phase is to actively collect new data to increase this eigenvalue, ensuring that each dimension is sufficiently covered. The detailed procedure is summarized in \Cref{al2}.

\textbf{Input and Initialization.}
The inputs and initialization directly use the results from Algorithm~\ref{al1}. Specifically, in addition to test user $u_t$ and offline dataset $\mathcal{D}$, the algorithm also takes the learned cluster graph $\mathcal{G}_{\hat\gamma}$, the neighbor set $\mathcal{V}_{\hat\gamma}(u_t)$, and the initial Gramian matrix $\tilde M_{u_t}$ and preference estimate $\tilde{\bm\theta}_{u_t}$ (Line~\ref{line_input_al2}). These are used to initialize the active-data augmentation phase (Line~\ref{line_initialization_al2}).

\textbf{Active-data Augmentation.}
The key component of Algorithm~\ref{al2} is the active-data augmentation procedure. In each round $n$, the algorithm selects the context-action pair on the most underrepresented dimensions to broaden the information matrix:
\begin{align}\label{align_active_data_selection}
\left(\mathring{\bm x}_{u_{t}}^n,\, \mathring{\bm a}_{u_{t}}^n,\, \mathring{\bm a}_{\ u_t}^{\prime n}\right)
= \argmax_{(\bm x,\, \bm a,\, \bm a') \in \mathcal{X} \times \mathcal{A} \times \mathcal{A}} 
\Big\{ \big\| \phi(\bm x, \bm a) - \phi(\bm x, \bm a') \big\|_{ \left( \tilde M_{u_{t}}^{n-1} \right)^{-1} } \Big\}.
\end{align}
After selection, the feedback $\mathring{y}_{u_t}^n$ is observed, and the difference feature $\mathring{\bm z}_{u_t}^n$ is computed. The Gramian matrix is then updated as \(\tilde M_{u_t}^n = \tilde M_{u_t}^{n-1} + \mathring{\bm z}_{u_t}^n \left( \mathring{\bm z}_{u_t}^n \right)^{\top},\) and the preference estimate is refined by solving the regularized maximum likelihood problem (regularized by the same $\lambda$ as that in \Cref{al1}) that combines both the offline aggregated data and all active-data up to round $n$:
\begin{align}\label{align_calculation_of_tilde_theta_t}
\tilde{\bm\theta}_{u_{t}}^n = \argmin_{\bm\theta} \Bigg(
& -\sum_{v \in \mathcal{V}_{\hat\gamma}(u_{t})} \sum_{i=1}^{N_v} 
\Big[ y_v^i \log \sigma\big(\bm\theta^{\top} \bm z_v^i\big)
+ (1 - y_v^i) \log \sigma\big(-\bm\theta^{\top} \bm z_v^i\big) \Big] \notag\\
& - \sum_{s=1}^{n} \Big[ \mathring{y}_{u_{t}}^{s} \log \sigma\big(\bm\theta^{\top} \mathring{\bm z}_{u_{t}}^{s}\big)
+ (1 - \mathring{y}_{u_{t}}^{s}) \log \sigma\big(-\bm\theta^{\top} \mathring{\bm z}_{u_{t}}^{s}\big) \Big]
+ \frac{\lambda}{2} \big\| \bm\theta \big\|_2^2 
\Bigg).
\end{align}

\textbf{Policy Output.}
Finally, the algorithm constructs the final preference estimate $\overline{\bm\theta}_{u_t}$ by taking a weighted average of all historical estimates $\tilde{\bm\theta}_{u_{t}}^n$ for $n=1,\cdots, N$:
\begin{align}\label{align_constructing_overline_theta}
\overline{\bm\theta}_{u_{t}} = 
\frac{1}{ d\, \lambda_{\min}\left( \tilde M_{u_{t}}^{N} \right) + N }
\left(
d\, \lambda_{\min}\left( \tilde M_{u_{t}}^{N} \right) \tilde{\bm\theta}_{u_{t}}^{N} 
+ \sum_{n=1}^{N} \tilde{\bm\theta}_{u_{t}}^n
\right).
\end{align}
This weighting places more emphasis on the final estimate, extending prior approach in~\citet{das2024active} which only uses a simple average for the pure active setting. The learned policy then selects the action that maximizes the expected reward as: \(\pi_{u_t}(\bm x) = \argmax_{\bm a \in \mathcal{A}} \phi(\bm x, \bm a)^{\top} \overline{\bm\theta}_{u_t}.\)

\subsection{Theoretical Results for Algorithm \ref{al2}}

We now present the theoretical guarantee for Algorithm~\ref{al2}, \hybalgo{}, in \Cref{th2}.

\begin{theorem}\label{th2}
Under the same assumptions as in \Cref{lemma_cardinality_of_R_and_W} and \Cref{th1}, the suboptimality gap of Algorithm~\ref{al2} for the test user $u_t$ can be bounded with probability at least $1-\delta$ as:
\[
\textnormal{SubOpt}_{u_{t}}(\pi_{u_{t}})
\leq 
\tilde O\left(
\frac{
\sqrt d \left( 1 + \hat\gamma \sqrt{ N_{\mathcal{W}_{\hat\gamma}(u_{t})} } \right)
}{
\sqrt{ \lambda_{\min}\left( \tilde M_{u_{t}}^{N} \right) + N/d }
}
\right),
\]
where $\tilde M_{u_t}^N=\frac{\lambda}{\kappa} I + \sum_{v \in \mathcal{V}_{\hat\gamma}(u_t)} \sum_{i=1}^{N_v} \bm z_v^i (\bm z_v^i)^{\top}+\sum_{i=1}^N\mathring{\bm z}_{u}^i(\mathring{\bm z}_u^i)^{\top}$ denotes the final Gramian matrix that combines both the offline aggregated data and the actively selected data in Algorithm~\ref{al2}.
\end{theorem}

In \Cref{th2}, the numerator mirrors the structure of \Cref{th1}: it is composed of two parts, where the first one representing the inherent sample noise, and the other capturing the bias introduced by heterogeneous neighbors. The term inside the square root of the denominator, $\lambda_{\min}(\tilde M_{u_t}^N) + N/d$, quantifies the effective number of "useful" samples that contribute to accurately estimating the preference vector for each dimension, just as in \Cref{th1}. Specifically, $\lambda_{\min}(\tilde M_{u_t}^N)$ reflects the normal contribution of the aggregated information matrix, $\tilde M_{u_t}^N$, from both pure offline samples and active selected samples in each dimension, while $N/d$ corresponds to the additional contribution of the $N$ active samples, distributed across $d$ dimensions.

\begin{remark}[Comparison with Prior Results]
When $N = 0$, the setting reduces to the pure offline scenario, and the suboptimality bound in \Cref{th2} naturally recovers the bound from \Cref{th1}. Additionally, as discussed in \Cref{remark_comparison_with_single_user_case}, setting $\hat\gamma = 0$ to only use samples from the test user itself allows us to specialize our result to the single-user case. Building on this, our framework can be further specialized to scenarios involving only active data without any offline data when $\mathcal{D}=\emptyset$, as explored in prior work~\citep{das2024active}. In this case, $\tilde M_{u_t}^N$ consists solely of active samples, and the suboptimality bound in \Cref{th2} outperforms the result in \citet{das2024active} (which achieves $\tilde O(d/\sqrt{N})$) by incorporating $\lambda_{\min}(\tilde M_{u_t}^N)$ into the denominator, yielding a more refined bound. 
\end{remark}

As shown in \Cref{th2}, the final Gramian matrix under active-data augmentation, denoted by $\tilde M_{u_t}^N$, differs from $\tilde M_{u_t}$ in that it not only aggregates the offline samples but also includes the actively selected samples. According to the selection rule in \Cref{align_active_data_selection}, the algorithm deliberately targets the dimensions with the sparsest information, which is fundamentally different from passively using the given offline dataset. In scenarios where the offline data is imbalanced (i.e. with some dimensions severely underrepresented while others are sufficiently covered), this active selection allows the algorithm to focus additional samples on the least informative directions, effectively ``filling in'' the gaps and improving the estimation. 

Therefore, a key quantity of interest is the improvement in the information matrix through our actively selected data, captured by the gap $\lambda_{\min}\big( \tilde M_{u_t}^N \big) - \lambda_{\min}\big( \tilde M_{u_t} \big)$. We first give \Cref{definition_sample_extreme_matrix} that characterizes such cases where active selection brings significant improvement.

\begin{definition}[$(d^*,N)$-Sample Imbalanced Gramian Matrix]
\label{definition_sample_extreme_matrix}
A Gramian matrix $M$ is called $(d^*,N)$-sample imbalanced if $d^*$ is the smallest value in $\{1,\cdots,d\}$ such that 
\(
\lambda_{d^*+1}\left( M \right) - \lambda_{\min}\left( M \right) \geq \lceil N/d^* \rceil.
\)
By convention, any matrix is at least $(d,N)$-sample imbalanced, since there are only $d$ dimensions and we treat $\lambda_{d+1}(M)$ as $+\infty$.
\end{definition}

Intuitively, this definition implies that there is a large discrepancy in sample sufficiency between the least well-informed dimension and the $(d^*+1)$-th dimension. For a $(d^*,N)$-sample imbalanced matrix, actively selecting samples according to \Cref{align_active_data_selection} can substantially boost the minimum eigenvalue by concentrating new samples where they are most needed. This is formalized in the following lemma.

\begin{lemma}[Quantification of the Minimum Eigenvalue Improvement]
\label{lemma_difference_of_eigenvalues}
Assume that the feature difference vector $\bm z = \phi(\bm x, \bm a) - \phi(\bm x, \bm a')$ can span the entire Euclidean unit ball $\{\bm z \in \mathbb{R}^d : \|\bm z\|_2 \leq 1\}$ for all $(\bm x, \bm a, \bm a')\in\mathcal{X} \times \mathcal{A} \times \mathcal{A}$. Further suppose that $\tilde M_{u_t}$ is $(d^*,N)$-sample imbalanced as defined in \Cref{definition_sample_extreme_matrix}. Then, under the active selection rule in \Cref{align_active_data_selection} for a total of $N$ rounds, it holds that
\[
\lambda_{\min}\big( \tilde M_{u_t}^N \big) - \lambda_{\min}\big( \tilde M_{u_t} \big) \geq \lfloor N/d^* \rfloor.
\]
\end{lemma}

Combining \Cref{lemma_difference_of_eigenvalues} with \Cref{th2}, we can explicitly show how the active sampling improves the bound relative to the pure offline setting.

\begin{corollary}\label{corollary_extreme_datasets}
Suppose that the assumptions in \Cref{lemma_difference_of_eigenvalues} hold. Then the suboptimality gap in \Cref{th2} can be rewritten as: 
\[
\textnormal{SubOpt}_{u_{t}}(\pi_{u_{t}})
\leq
\tilde O\left(
\frac{
\sqrt{d} \left( 1 + \hat\gamma \sqrt{ N_{\mathcal{W}_{\hat\gamma}(u_{t})} } \right)
}{
\sqrt{\lambda_{\min}\left( \tilde M_{u_{t}} \right) + N /d^*}
}
\right).
\]
Moreover, the bound can be simplified to: \(\tilde O\bigg(\frac{\sqrt d \big( 1 + \hat\gamma \sqrt{N_{\mathcal{W}_{\hat\gamma}(u_{t})} } \big)}{\sqrt{\lambda_{\min}\big( \tilde M_{u_{t}} \big) + N}}\bigg)\) when $\tilde M_{u_t}$ is $(1,N)$-sample imbalanced.
\end{corollary}


As shown in \Cref{lemma_difference_of_eigenvalues} and \Cref{corollary_extreme_datasets}, when the offline Gramian matrix $\tilde M_{u_t}$ is highly imbalanced (i.e., well covered in some dimensions but sparse in others) our active-data selection rule explicitly targets the underrepresented dimensions. In this case, each actively selected sample can contribute more than a single effective observation. Specifically, comparing \Cref{th1} with \Cref{corollary_extreme_datasets}, the denominator improves by $\tilde O(N/d^*)$ for some $d^*\leq d$, rather than the $O(N/d)$ scaling in the general case. Intuitively, the active samples only need to be distributed across $d^*$ dimensions instead of all $d$ dimensions. Consequently, for a $(d^*,N)$-sample imbalanced matrix $\tilde M_{u_t}$, one actively selected sample is equivalent to $d/d^*$ fully informative samples and yields a suboptimality gain. \Cref{fig:eigenvalue_improvement} depicts this phenomenon. This result highlights how active-data augmentation can effectively mitigate imbalance in offline coverage by reinforcing the sparse directions of the preference.

\begin{figure}[t]
    \centering
    \includegraphics[width=0.75\textwidth]{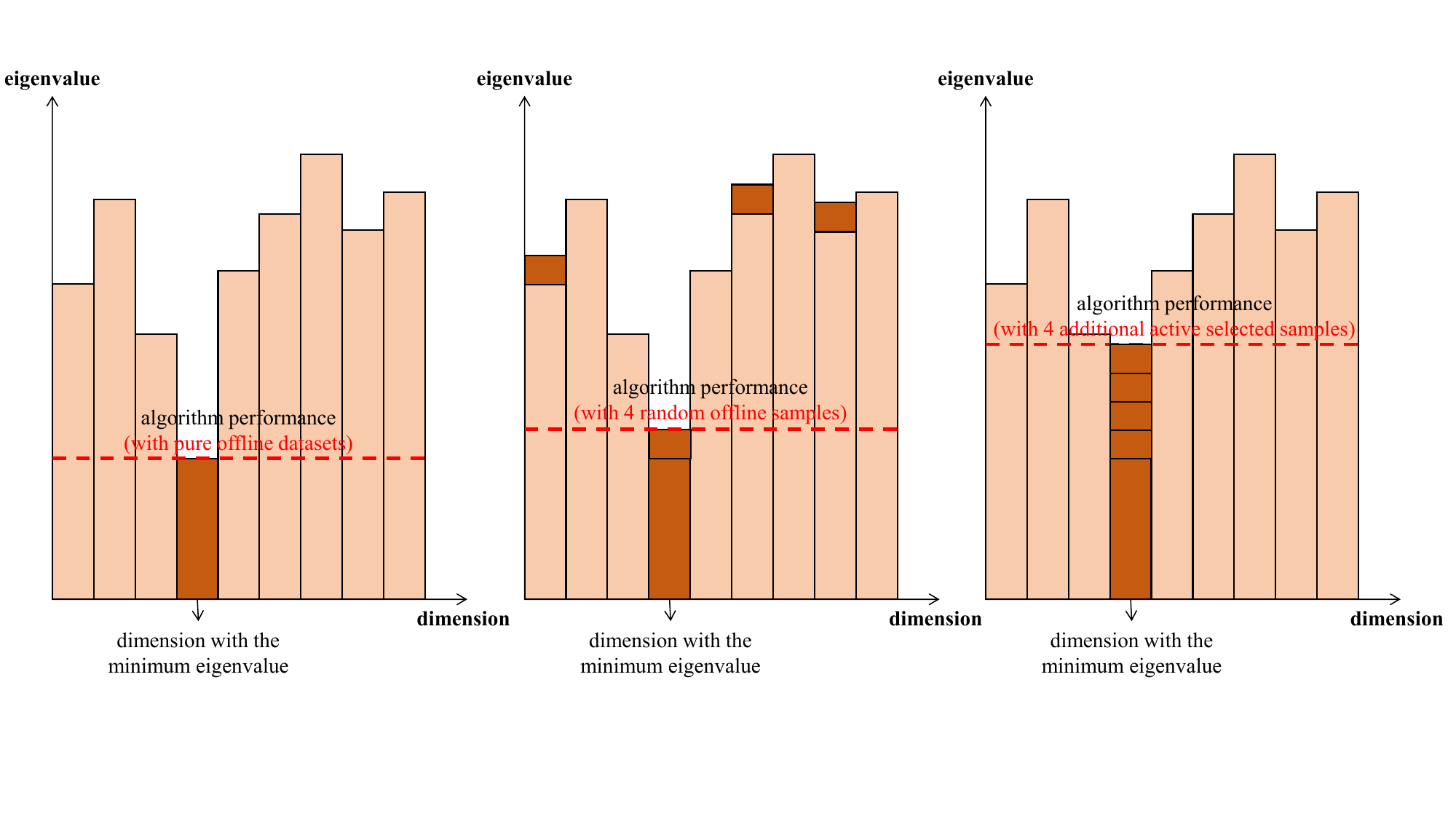}
    \vspace{-10mm}
    \caption{Illustration of how active-data augmentation enhances performance by increasing the minimum eigenvalue of the information matrix.  
    \textbf{Left:} Pure offline data suffers from underrepresented dimensions, limiting performance.  
    \textbf{Middle:} Adding random offline samples offers limited improvement.  
    \textbf{Right:} Actively selected samples focus on underrepresented dimensions, substantially increasing the minimum eigenvalue and improving performance.}
\label{fig:eigenvalue_improvement}

\end{figure}

Finally, we present a special-case result under the item regularity assumption (\Cref{assumption_item_regularity}) and the condition that $\tilde M_{u_t}$ is $(d^*, N)$-sample imbalanced, which illustrates the benefit of active-data augmentation even in a traditional bandit context:

\begin{corollary}\label{corollary_imbalanced_regularity_datasets}
Suppose \Cref{assumption_item_regularity} holds and that $\tilde M_{u_t}$ is $(d^*, N)$-sample imbalanced. Following the proof of \Cref{th1_item_regularity}, it holds that
$$
\textnormal{SubOpt}_{u_{t}}(\pi_{u_{t}})
\;\leq\;
\tilde O\left( \sqrt{\frac{d}{\tilde\lambda_a}} \left( \frac{1}{\sqrt{N_{\mathcal{V}_{\hat\gamma}(u_t)}+N/(d^*\tilde\lambda_a)}} + \frac{\hat\gamma\sqrt{N_{\mathcal{W}_{\hat\gamma}(u_t)}}}{\sqrt{N_{\mathcal{V}_{\hat\gamma}(u_t)}+N/(d^*\tilde\lambda_a)}} \right)
\right).
$$
\end{corollary}

\Cref{corollary_imbalanced_regularity_datasets} can be interpreted in terms of \emph{noise} (the first term) and \emph{bias} (the second term). Importantly, under \Cref{assumption_item_regularity}, each actively selected sample is equivalent to at least $1/(d^*\tilde\lambda_a)$ offline samples (which is strictly greater than one, since $\tilde\lambda_a \leq 1/d \leq 1/d^*$ holds by~\citet{wang2023onlinea}). This advantage arises because active samples offer better coverage through the active selection rule than the coverage offered by \Cref{assumption_item_regularity} for offline samples. Consequently, this result strengthens \Cref{th1_item_regularity}, yielding a strictly better suboptimality bound by reducing both noise and bias.
\section{Experiments}
\label{sec:simulations}

In this section, we evaluate the performance of Off-C$^2$PL and \hybalgo{} using synthetic and real-world data. All experiments are averaged over $20$ independent rounds.

\begin{figure}[t]
  \centering

  \begin{subfigure}[t]{\linewidth}
    \centering
    \includegraphics[width=\linewidth]{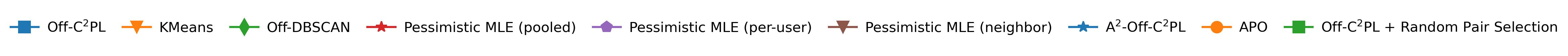}
\end{subfigure}

  \centering
  \begin{subfigure}[t]{0.24\linewidth}
    \centering
    \includegraphics[width=\linewidth]{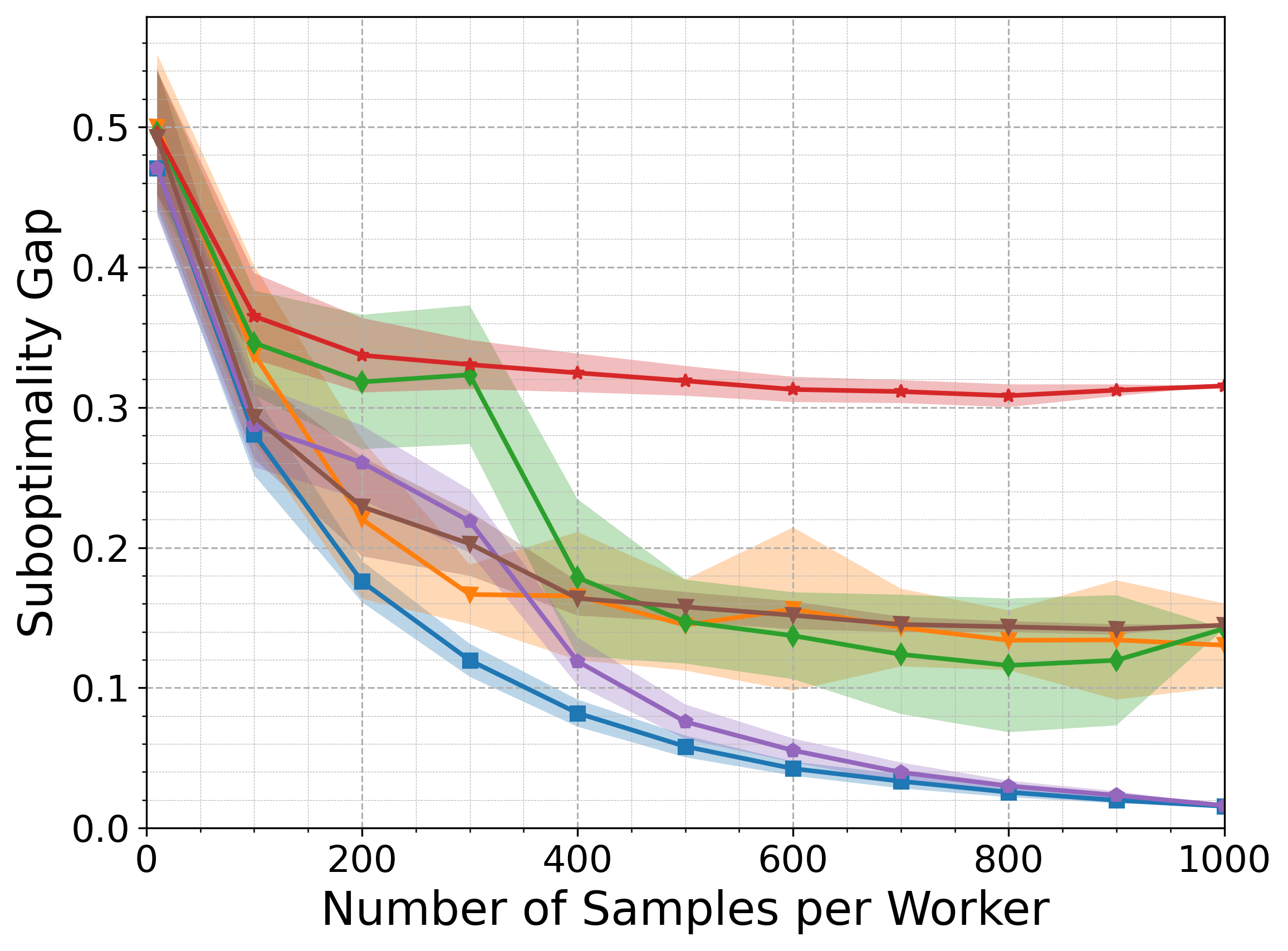}
    \caption{Offline: Synthetic}
    \label{fig:off_synth}
  \end{subfigure}
  \hfill
  \begin{subfigure}[t]{0.24\linewidth}
    \centering
    \includegraphics[width=\linewidth]{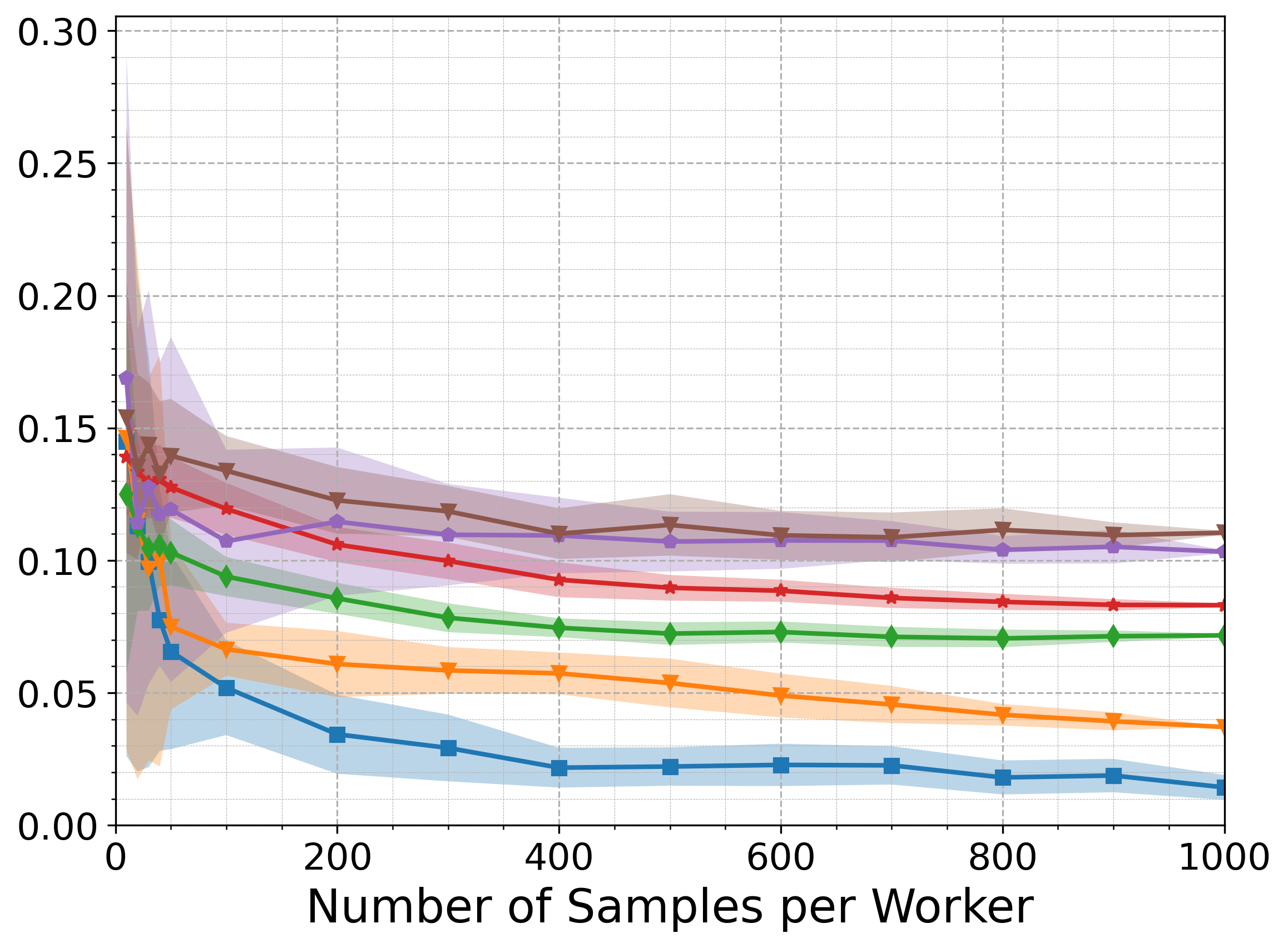}
    \caption{Offline: Reddit}
    \label{fig:off_Reddit}
  \end{subfigure}
  \hfill
  \begin{subfigure}[t]{0.24\linewidth}
    \centering
    \includegraphics[width=\linewidth]{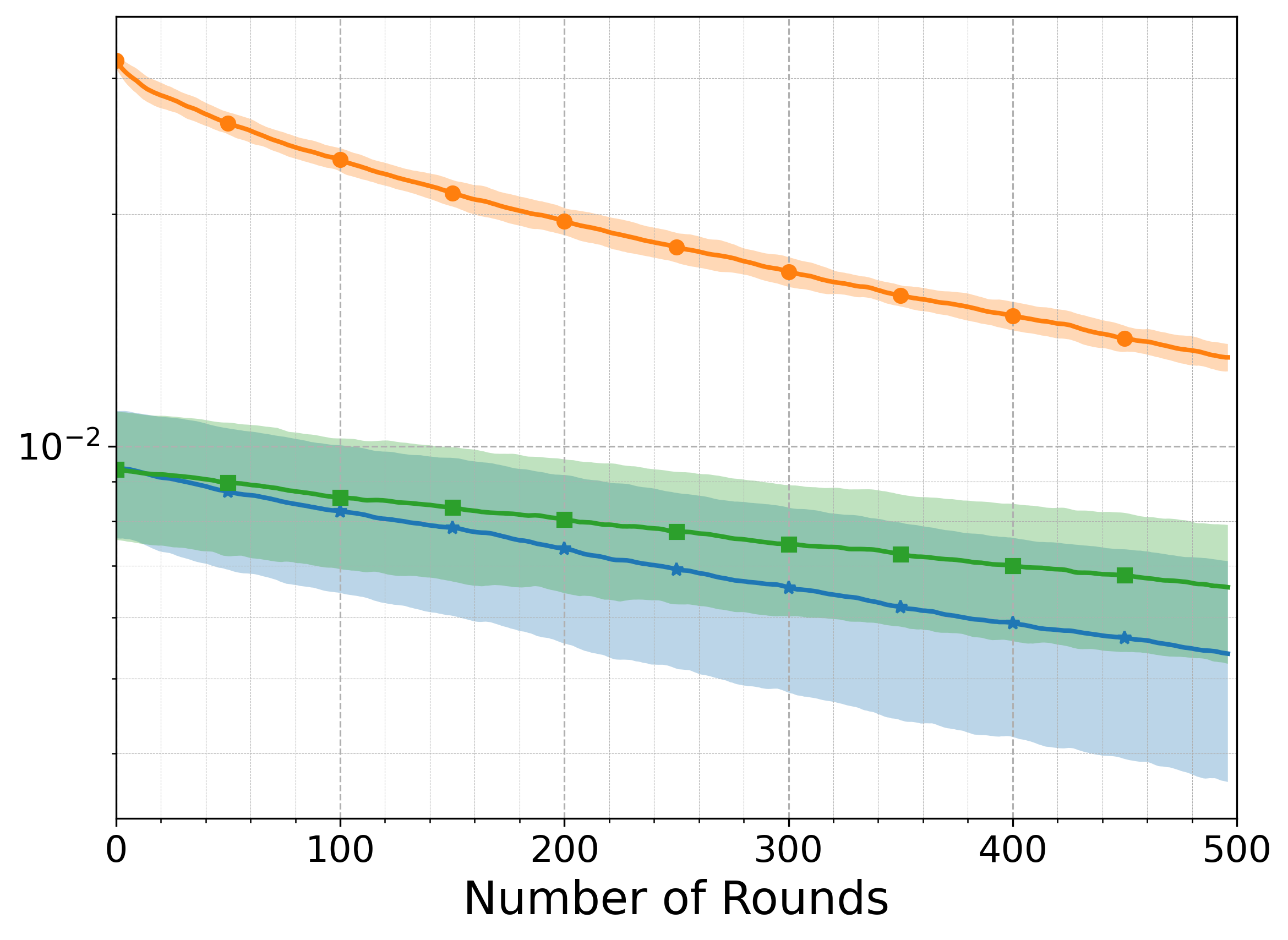}
    \caption{Hybrid: Synthetic}
    \label{fig:hyb_synth}
  \end{subfigure}
  \hfill
  \begin{subfigure}[t]{0.24\linewidth}
    \centering
    \includegraphics[width=\linewidth]{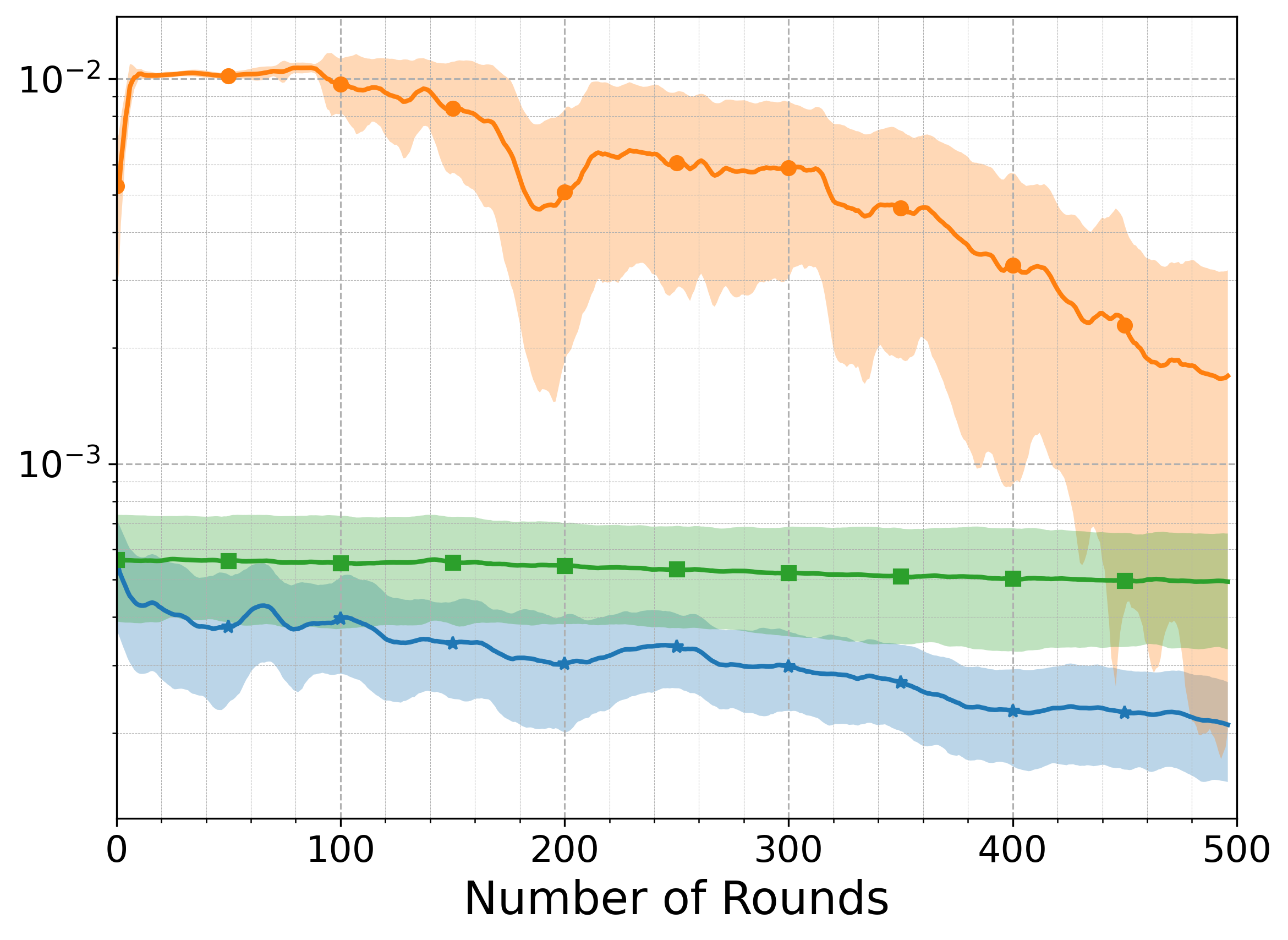}
    \caption{Hybrid: Reddit}
    \label{fig:hyb_Reddit}
  \end{subfigure}

  \centering

  \begin{subfigure}[t]{0.24\linewidth}
    \centering
    \includegraphics[width=\linewidth]{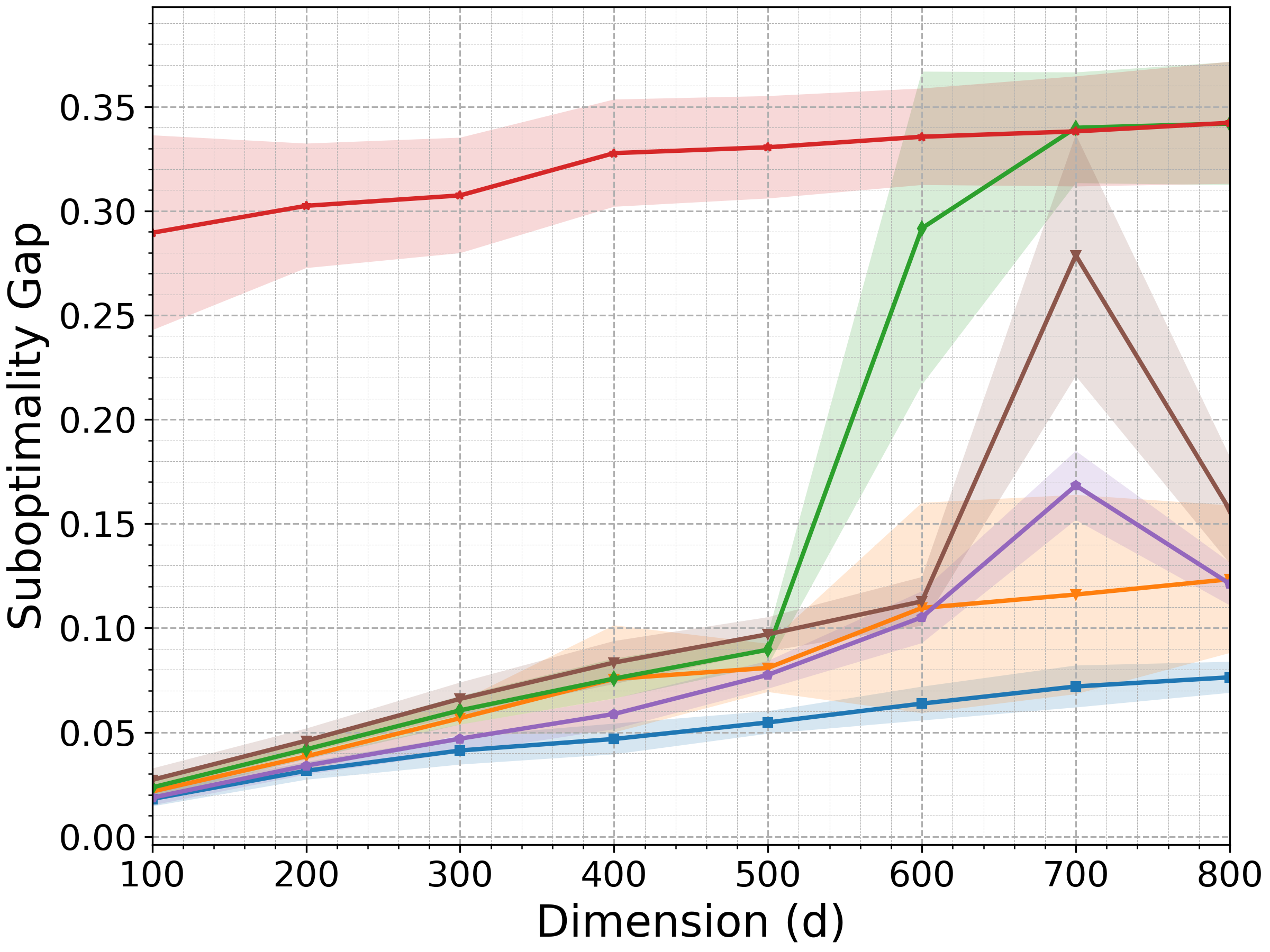}
    \caption{$d$: Synthetic}
    \label{fig:dim_synth}
  \end{subfigure}
  \hfill
  \begin{subfigure}[t]{0.24\linewidth}
    \centering
    \includegraphics[width=\linewidth]{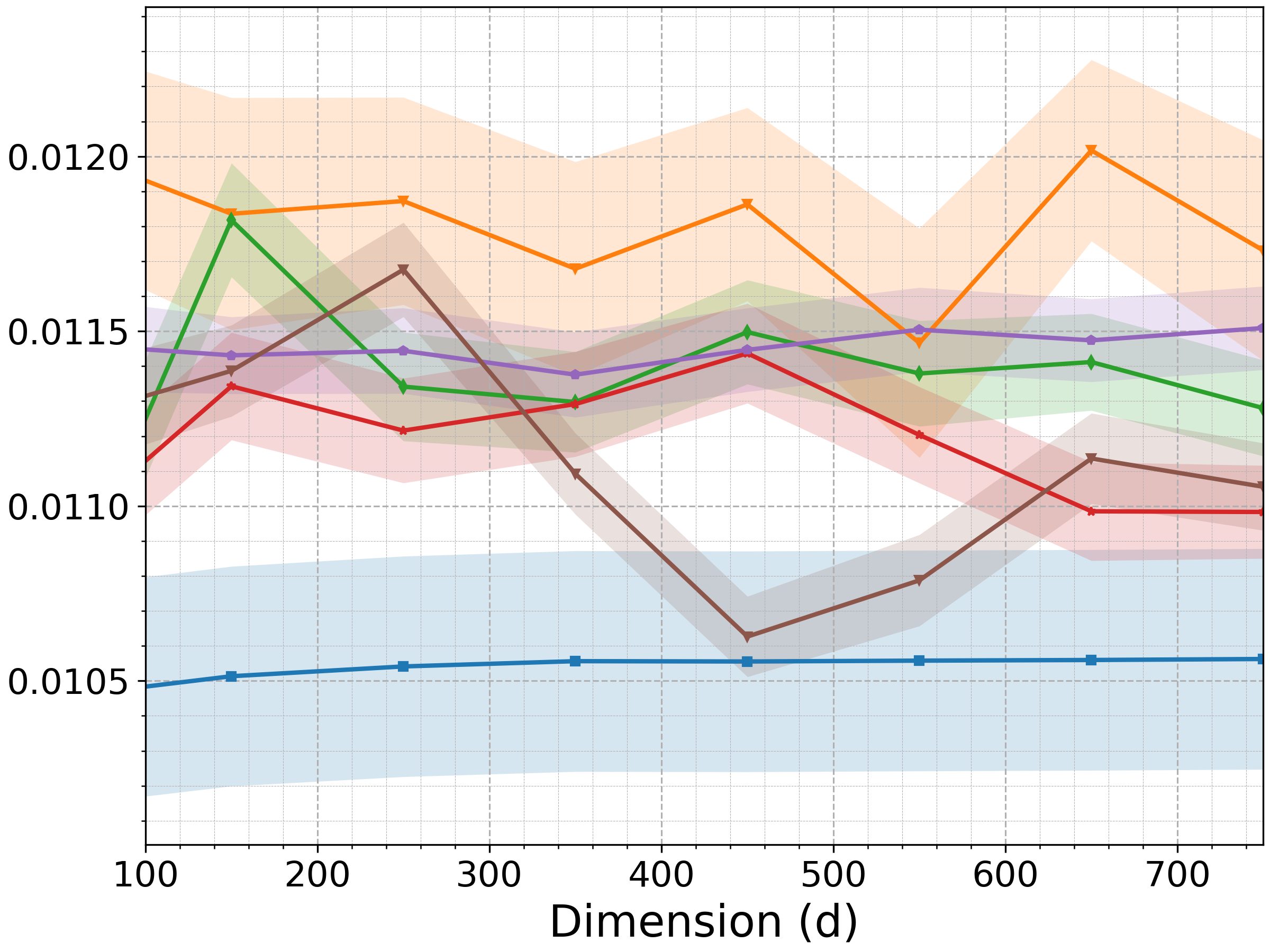}
    \caption{$d$: Reddit}
    \label{fig:dim_Reddit}
  \end{subfigure}
  \hfill
  \begin{subfigure}[t]{0.24\linewidth}
    \centering
    \includegraphics[width=\linewidth]{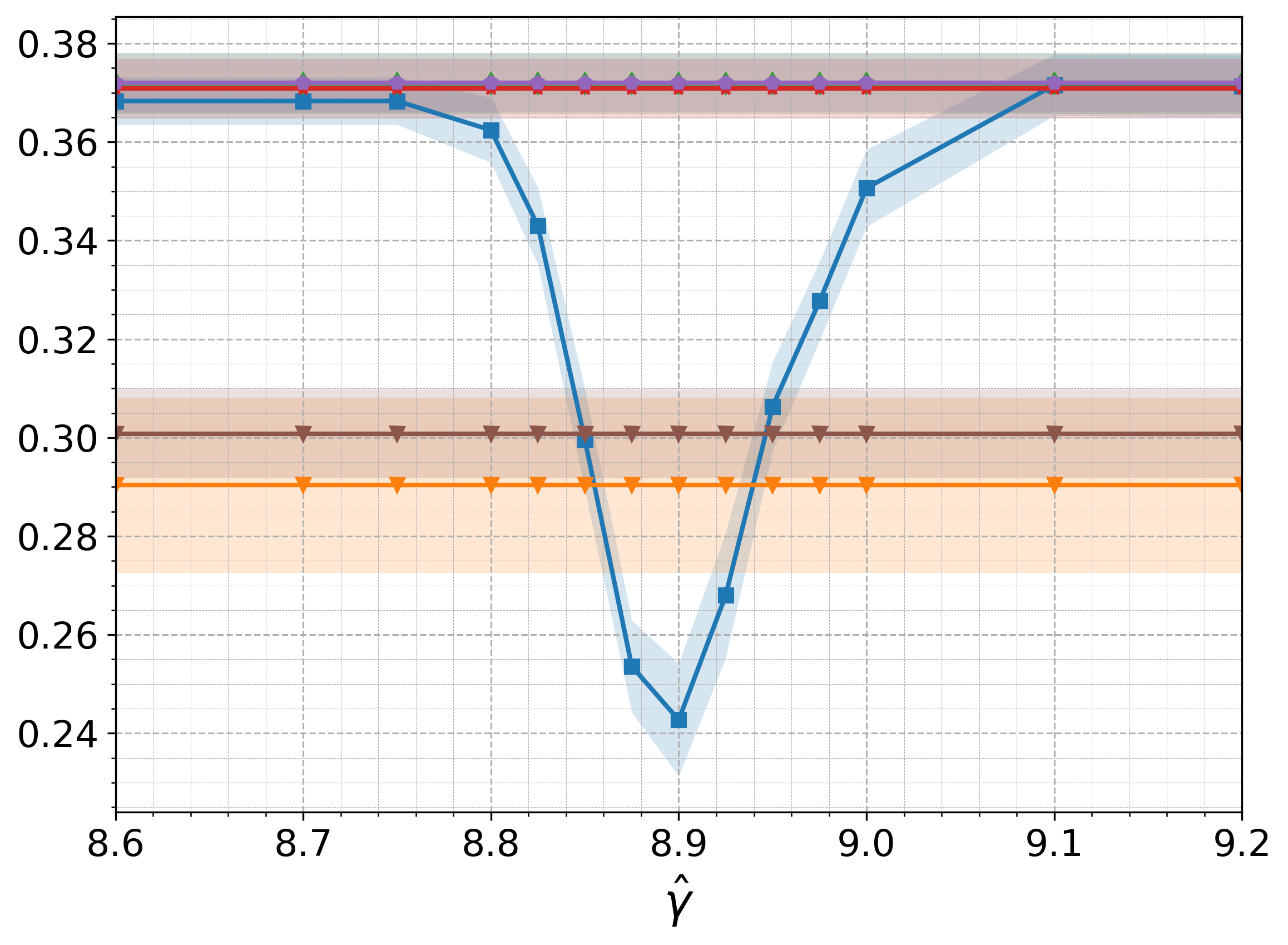}
    \caption{$\hat{\gamma}$: Synthetic}
    \label{fig:gamma_synth}
  \end{subfigure}
  \hfill
  \begin{subfigure}[t]{0.24\linewidth}
    \centering
    \includegraphics[width=\linewidth]{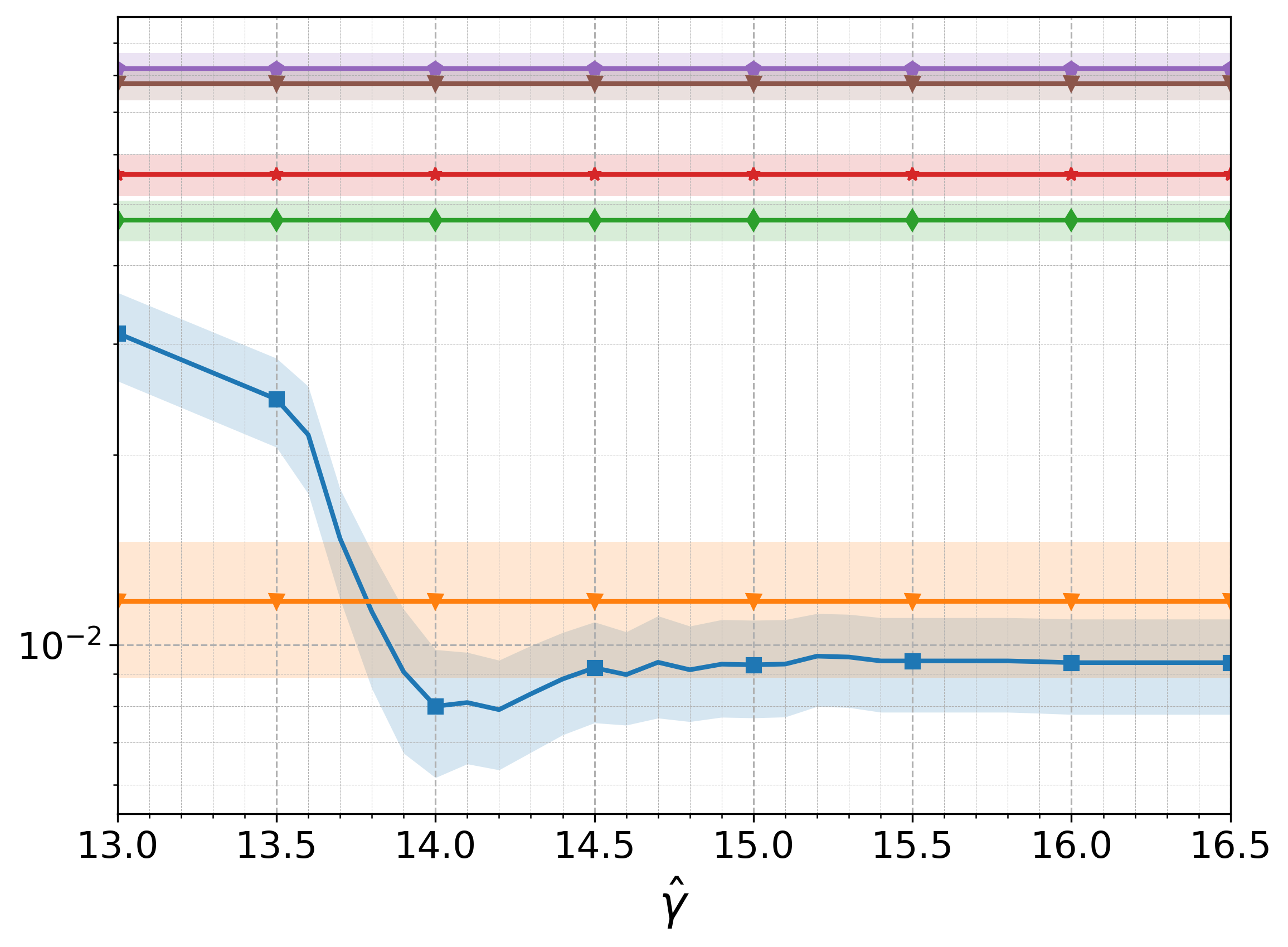}
    \caption{$\hat{\gamma}$: Reddit}
    \label{fig:gamma_Reddit}
  \end{subfigure}

 \caption{Figures~\ref{fig:off_synth} and~\ref{fig:off_Reddit} correspond to \emph{performance in offline setting with insufficient data}, 
Figures~\ref{fig:hyb_synth} and~\ref{fig:hyb_Reddit} correspond to performance in hybrid setting, 
Figures~\ref{fig:gamma_synth} and~\ref{fig:gamma_Reddit} correspond to \emph{the impact of dimension $d$}, and
Figures~\ref{fig:dim_synth} and~\ref{fig:dim_Reddit} correspond to \emph{the impact of clustering-threshold \(\hat{\gamma}\)}.}
  \label{fig:ablation}
\end{figure}

\emph{Baselines.} 
We compare \offalgo{} with both enhanced versions of traditional clustering algorithms and prior methods for contextual logistic bandits. Specifically, we adapt classical clustering algorithms such as KMeans~\citep{mcqueen1967some} (with $\sqrt{\# \text{ of users}}$ as cluster number) and DBSCAN~\citep{schubert2017dbscan} to our setting by incorporating the same policy output phase as in \Cref{al1} with their clustering procedures. We also include variants of Pessimistic MLE ~\cite{zhu2023principled} for contextual logistic bandits: \textit{Pessimistic MLE (per-user)} uses only the test user’s data, \textit{Pessimistic MLE (pooled)} aggregates data from all users, and \textit{Pessimistic MLE (neighbor)} leverages data from the test user’s neighbors identified by a KNN algorithm using cosine similarity on $\bm\theta$. 
For evaluating \hybalgo{}, we compare against the pure offline algorithm \offalgo{} trained on randomly generated offline samples and the pure active learning algorithm Active Preference Optimization (APO) from~\citet{das2024active} that operates without any offline data.

\emph{Synthetic Dataset.}
We construct a synthetic pairwise–preference dataset with \(U = 40\) users partitioned into \(J = 8\) clusters uniformly at random. 
Each cluster \(j\) has a ground-truth vector \(\bm{\theta}^{j}\in\mathbb{R}^{d}\) with \(d=768\), matching the dimensionality of the real-world embeddings used in our experiments. 
For a user \(u\) in cluster \(c\), we set \(\bm{\theta}_{u}=\bm{\theta}^{j}+\bm{\epsilon}_{u}\), where \(\bm{\epsilon}_{u}\sim\mathcal{N}(0, s^{2} I_{d})\). 
This adds mild within-cluster heterogeneity so users are similar but not identical, better reflecting real data. 
We then generate \(1000\) pairwise comparisons per user under a Bradley–Terry–Luce model: for a pair-difference feature \(\bm{z}\sim\mathcal{N}(0, I_{d})\), the preferred item is sampled with probability \(\sigma\left(\beta\,\bm{\theta}_{u}\cdot\bm{z}\right)\), where \(\sigma(x)=(1+e^{-x})^{-1}\) and \(\beta\) controls noise (larger \(\beta\) implies cleaner preferences).

\emph{Real-World Dataset.}
We use the Reddit TL;DR summarization \cite{volske2017tl} alongside human preferences collected by \citet{stiennon2020learning}. 
Each sample in our dataset consists of a forum post from Reddit, paired with two distinct summaries generated by the GPT-2 language model. 
Human annotators then indicate their preference for one of the summaries. 
This dataset contains preference annotations from $76$ users, with individual contributions ranging from as few as $2$ to more than \(18{,}000\) prompts. For evaluation, we focus on 42 annotators who each provide more than \(1{,}000\) annotations, and from each of these, we uniformly sample \(1{,}000\) preferences for testing. In order to calculate the suboptimality gap, it is necessary to have access to an optimal policy. 
However, the true optimal policy is unknown when working with real-world data.
Therefore, we must rely on the available dataset to approximate the most optimal policy. 
Thus, we leverage maximum likelihood estimation (MLE) regression through a gradient descent on the full dataset, to ensure that the derived optimal policy is optimal relative to the given dataset. 


\emph{Experiment 1: Performance under \offmodel{}.}
We examine \offalgo{} against a suite of baselines on both the synthetic and the Reddit dataset, varying the per-user sample budget from \(100\) to \(1000\) pairs, considering \(40\) users. 
On the synthetic data (Figure~\ref{fig:off_synth}), \offalgo{} has the smallest suboptimality gap across the entire range. 
Relative to the baselines in this run, it improves performance by \(88.1\%\) over KMeans, \(89.1\%\) over Off-DBSCAN, and \(95.1\%\), \(89.2\%\), and \(3.39\%\) over Pessimistic MLE (pooled), (neighbor), and (per-user).
Pessimistic MLE (per-user) becomes competitive only after using more than \(80\%\) of the samples and remains clearly worse in the low-sample regime.
On the Reddit dataset (Figure~\ref{fig:off_Reddit}), no baseline matches Off-$C^{2}$PL. 
With only \(\approx 400\) pairs per user it achieves a near-zero suboptimality gap and delivers relative improvements of \(61.5\%\) over KMeans, \(80.1\%\) over Off-DBSCAN, \(82.8\%\) over Pessimistic MLE (pooled), \(87.1\%\) over the neighbor, and \(86.2\%\) over the per-user variant.

\emph{Experiment 2: Performance under \hybmodel{}.} 
We compare \hybalgo{} against APO and an algorithm which uses \offalgo{} as offline initialization but replaces our active-data augmentation strategy with \emph{random} pair selection.
We allocate \(20\%\) of the data to the offline phase and then run \(500\) rounds of active-data selection. 
On the Reddit dataset, \hybalgo{} yields relative improvements of \(87.6\%\) over the online-only baseline and \(57.5\%\) over the random-selection baseline. 
On the synthetic dataset, the corresponding improvements are \(58.7\%\) and \(18.0\%\). 
As shown in Figures~\ref{fig:hyb_synth} and~\ref{fig:hyb_Reddit}, the pure active method begins with a large suboptimality gap due to the missing offline head start.
Although the active phase reduces this gap over rounds, it remains substantially worse. 
The random-selection baseline starts at the same gap as \hybalgo{} but fails to discover sufficiently informative pairs and therefore makes little progress. 
In contrast, \hybalgo{} consistently drives the gap downward across active rounds, achieving the best performance throughout.

\emph{Experiment 3: The impact of dimension $d$.}
We vary dimension \(d\) from \(100\) to \(800\) on synthetic data and from \(100\) to \(768\) on Reddit. 
For Reddit, we obtain lower-dimensional features by applying PCA to the original \(768\)-dimensional embeddings, so \(768\) is the maximum. 
On the synthetic dataset (Figure~\ref{fig:dim_synth}), the gap increases with \(d\) at a fixed sample size, as expected from higher estimation complexity. 
Notably, \offalgo{} degrades the slowest as it uses data across users within clusters and regularizes effectively in high dimensions. 
On Reddit, however (Figure~\ref{fig:dim_Reddit}), there is no noticeable trend in performance across \(d\), which is consistent with PCA preserving the dominant variance directions. 
Truncating to lower \(d\) primarily removes low-variance components that contribute little to preference prediction.

\emph{Experiment 4: The impact of clustering-threshold \(\hat{\gamma}\).}
Sweeping the clustering-threshold \(\hat{\gamma}\) reveals a bias–variance trade-off: overly small values merge unrelated users, while overly large values prevent pooling users in true clusters (Figures~\ref{fig:gamma_synth} and~\ref{fig:gamma_Reddit}). 
With a well-calibrated \(\hat{\gamma}\), \offalgo{} recovers the correct cluster structure and substantially reduces the suboptimality gap, demonstrating that accurate control of cluster connectivity is crucial when data is scarce.

\section{Conclusion}\label{sec:conclusion}

In this paper, we introduce and systematically study the Offline Clustering of Preference Learning problem, where user preferences naturally vary. We propose Algorithm~\ref{al1} (\offalgo{}), which leverages maximum likelihood estimation to cluster users with similar preferences without relying on any coverage assumption, enabling accurate aggregation of heterogeneous offline data. Our theoretical analysis characterizes the tradeoff between variance reduction from data aggregation and bias introduced by heterogeneity. We further extend this framework with active-data augmentation in Algorithm~\ref{al2} (\hybalgo{}), which selectively samples underrepresented dimensions, achieving notable theoretical and empirical gains over purely offline methods.

A promising direction for future work is to refine our suboptimality bounds in cases where the $\ell_2$ norm of $\bm\theta_u$ is not constant. While prior single-user analyses improve the dependence on the nonlinearity parameter from $1/\kappa$ to $1/\sqrt{\kappa}$, extending this improvement to heterogeneous multi-user clustering remains open. Developing techniques to achieve a $1/\sqrt{\kappa}$ dependency within our framework would mark a significant theoretical advancement.

\newpage
\bibliographystyle{ACM-Reference-Format}
\bibliography{main}

\newpage
\appendix
\clearpage
\onecolumn

\appendix

\section*{Appendix}

\begin{table}[t]
\centering
\caption{Summary of key notations.}
\label{tab:notations}
\renewcommand{\arraystretch}{1.1}
\begin{tabularx}{\textwidth}{@{}lX@{}}
\toprule
\textbf{Notation} & \textbf{Description} \\
\midrule
$U$, $\mathcal{U}$ & Number of users and the user set $\{1,\dots,U\}$. \\
$J$ & Number of clusters which is unknown to the learner. \\
$\bm\theta_u$ & True $d$-dimensional preference vector of user $u$ with $\|\bm\theta_u\|_2 \le 1$. \\
$\bm\theta^{j}$ & Preference shared by all users in cluster $j$. \\
$\mathcal{U}(j)$ & Users in cluster $j$. \\
$\phi(\bm x,\bm a)$ & Feature map $\phi:\mathcal{X}\times\mathcal{A}\to\mathbb{R}^d$ with $\|\phi(\bm x,\bm a)\|_2 \le 1$. \\
$\mathcal{D}_u$ & Offline data of user $u$: $\{(\bm x_u^i,\bm a_u^i,{\bm a'}_u^i,y_u^i)\}_{i=1}^{N_u}$. \\
$\bm z_u^i$ & Feature difference $\phi(\bm x_u^i,\bm a_u^i)-\phi(\bm x_u^i,{\bm a'}_u^i)$. \\
$\sigma(\cdot)$ & Sigmoid function in the BTL preference model. \\
$\kappa$ & Non-linearity coefficient (\Cref{align_kappa}); lower bound on $\nabla\sigma(\cdot)$ across comparisons. \\
$M_u$ & Regularized Gramian from $\mathcal{D}_u$: $\frac{\lambda}{\kappa}I+\sum_{i\in\mathcal{D}_u} \bm z_u^i(\bm z_u^i)^\top$. \\
$\lambda_{\min}(M)$ & Minimum eigenvalue of matrix $M$. \\
$\mathrm{CI}_u$ & Confidence radius for the MLE $\hat{\bm\theta}_u$. \\
$\hat\gamma$ & Clustering threshold controlling when two users are connected. \\
$\mathcal{V}_{\hat\gamma}(u)$ & Set of user $u$ and its neighbors connected under threshold $\hat\gamma$. \\
$\tilde M_u$, $\tilde N_u$ & Aggregated Gramian and sample count over $\mathcal{V}_{\hat\gamma}(u)$. \\
$\pi_u^*$ & Optimal policy for user $u$. \\
$\text{SubOpt}_u(\pi)$ & Suboptimality gap of policy $\pi$ for user $u$ (\Cref{align_objective}). \\
\bottomrule
\end{tabularx}
\end{table}

\section{Detailed Discussion of Remark~\ref{remark_selection_of_hatgamma}}\label{appendix_discussions_on_hatgamma}
This appendix elaborates practical policies for choosing the clustering threshold $\hat\gamma$. Our treatment closely follows the guidance in~\citet{liu2025offline}; we include their spirit here for completeness and refer readers there for additional discussion.

\subsection{Case 1: Known $\gamma$}
When the minimum heterogeneity gap $\gamma$ (defined in \Cref{definition_minimum_heterogeneity_gap}) is known, a natural choice is $\hat\gamma=\gamma$, which exactly separates users across clusters.


\begin{remark}[Discussions on $\gamma$ Known Cases]
Setting $\hat\gamma=\gamma$ eliminates bias from heterogeneous neighbors because the graph connects only users with the same preference vectors, implying $\mathcal{W}_{\hat\gamma}(u_t)=\emptyset$. The bound thus reflects only sampling noise from the homogeneous neighborhood $\mathcal{V}_{\hat\gamma}(u_t)$. \Cref{lemma_cardinality_of_R_and_W} and \Cref{align_special_form_of_th1} together show that setting $\hat\gamma = \gamma$ allows \Cref{al1} to maximize $\mathcal{R}_{\hat\gamma}(u_t)$ while still ensuring zero bias, making this choice practical. Notably, choosing $\hat\gamma<\gamma$ would also make $\mathcal{W}_{\hat\gamma}(u_t)=\emptyset$, but at the cost of potentially shrinking $\mathcal{R}_{\hat\gamma}(u_t)$ and losing valuable homogeneous samples which leads to smaller $\mathcal{V}_{\hat\gamma}(u_t)$ and thus increases the noise.
\end{remark}

\subsection{Case 2: Unknown $\gamma$}
When $\gamma$ is unknown, the threshold $\hat\gamma$ must be estimated from the offline data. We define
\begin{align}\label{align_definition_of_Gamma}
\Gamma(u, v) = \|\hat{\bm{\theta}}_{u} - \hat{\bm{\theta}}_{v}\|_2 - \alpha(\text{CI}_{u} + \text{CI}_{v}),
\qquad 
M(u) = \{v \in \mathcal{U} \setminus \{u\} \,:\, \Gamma(u, v) > 0\},
\end{align}
where $\text{CI}_u$ is given in \Cref{align_statistics_initialization_al1}. For $\alpha\ge 1$, $\Gamma(u,v)\le \|\bm\theta_u-\bm\theta_v\|_2$ is a lower bound on the true preference gap, and $M(u)$ collects users deemed heterogeneous relative to $u$. We consider two complementary policies.

\begin{definition}[Underestimation policy]
The underestimation policy is defined as:
\begin{align}\label{align_policy_for_choosing_unknown_hat_gamma}
\hat{\gamma} = \mathbb{I}\{M(u_{t}) \neq \emptyset\} \cdot \min_{v \in M(u_{t})} \Gamma(u_{t}, v).
\end{align}
\end{definition}

\begin{theorem}[Effect of the underestimation policy]\label{lemma_effect_of_underestimation_policy}
With $\hat\gamma$ chosen by \Cref{align_policy_for_choosing_unknown_hat_gamma} and $\alpha_w' = \tfrac{\kappa}{3(\alpha+1)\sqrt{2\max\{2,d\}\log(2U/\delta)}}$, any user $v$ in the heterogeneous neighbor set $\mathcal{W}_{\hat\gamma}(u_t)$ of \Cref{lemma_cardinality_of_R_and_W} also satisfies
\[
\frac{1}{\sqrt{\lambda_{\min}(M_{u_{t}})}} + \frac{1}{\sqrt{\lambda_{\min}(M_v)}} 
\;\geq\; \alpha_w' \, \|\bm\theta_{u_{t}} - \bm\theta_v \|_2.
\]
\end{theorem}

\begin{remark}[When an underestimation policy is preferable]
This conservative choice keeps $\mathcal{W}_{\hat\gamma}(u_t)$ small—only users with limited information enter—thereby controlling bias. The tradeoff is fewer homogeneous neighbors ($\mathcal{R}_{\hat\gamma}(u_t)$ and $\mathcal{V}_{\hat\gamma}(u_t)$ may shrink), which can increase noise. It is therefore preferable when bias is the primary concern—for example, in RLHF with annotators from diverse regions where mis-clustering can inject systematic preference bias or in fairness-sensitive applications (e.g., healthcare or education) where even small cross-group bias is more harmful than the extra noise from using fewer neighbors.
\end{remark}

\begin{definition}[Overestimation policy]\label{definition_of_overestimation_policy}
The overestimation policy is defined as:
\begin{align}\label{align_policy_for_choosing_optimistic_hatgamma}
\hat{\gamma} = \mathbb{I}\{M(u_{t}) \neq \emptyset\} \cdot \min_{v \in M(u_{t})} \tilde{\Gamma}(u_{t}, v),
\end{align}
where $\tilde\Gamma(u_{t}, v) = \|\hat{\bm\theta}_{u_{t}} - \hat{\bm\theta}_v\|_2 + \alpha(\text{CI}_{u_{t}} + \text{CI}_v)$ is an \emph{upper} bound on the gap between users $u_t$ and $v$.
\end{definition}

\begin{theorem}[Effect of the overestimation policy]\label{lemma_effect_of_overestimation_policy}
Under the policy in \Cref{definition_of_overestimation_policy}, if $M(u_t)\neq\emptyset$ then $\hat\gamma\geq\gamma$.
\end{theorem}

\begin{remark}[When an overestimation policy is preferable]
Ensuring $\hat\gamma \ge \gamma$ expands both the homogeneous neighbor set $\mathcal{R}_{\hat\gamma}(u_t)$ and the heterogeneous neighbor set $\mathcal{W}_{\hat\gamma}(u_t)$. This typically reduces noise but may also increase bias through more heterogeneous neighbors. This policy is therefore well-suited to noise-dominated regimes, such as recommendation cohorts with sparse but relatively homogeneous histories; or high-dimension scenarios where the number of dimensions $d$ is large.
\end{remark}

Both policies introduced here have their advantages and disadvantages. Underestimation reduces bias at the expense of higher noise; while overestimation does the opposite. In practice, the preferred policy depends on whether bias or noise is the main bottleneck. For additional discussion and complementary proofs of Lemmas~\ref{lemma_effect_of_underestimation_policy} and~\ref{lemma_effect_of_overestimation_policy}, see~\citet{liu2025offline}.

\section{Detailed Proofs}\label{appendix_detailed_proofs}

\subsection{Proof of \Cref{lemma_confidence_ellipsoid_of_hat_theta}}

\begin{proof}
First, for any $\bm\theta_s \in \mathbb{R}^d$, define
\[
G_u(\bm\theta_s) \coloneqq \sum_{i=1}^{N_u} \left( \sigma\left( \bm\theta_s^{\top} \bm z_u^i \right) - \sigma\left( \bm\theta_u^{\top} \bm z_u^i \right) \right) \bm z_u^i + \lambda \bm\theta_s.
\]

By the mean value theorem, for any two parameter vectors $\bm\theta_{s_1}$ and $\bm\theta_{s_2}$, we have
\[
G_u(\bm\theta_{s_1}) - G_u(\bm\theta_{s_2})
= \left( \sum_{i=1}^{N_u} \nabla \sigma\left( \bm\theta_{\overline{s}}^{\top} \bm z_u^i \right) \bm z_u^i (\bm z_u^i)^{\top} + \lambda I \right) (\bm\theta_{s_1} - \bm\theta_{s_2})
= W_u (\bm\theta_{s_1} - \bm\theta_{s_2}),
\]
where we define
\[
W_u \coloneqq \sum_{i=1}^{N_u} \nabla \sigma\left( \bm\theta_{\overline{s}}^{\top} \bm z_u^i \right) \bm z_u^i (\bm z_u^i)^{\top} + \lambda I
\quad \text{and} \quad
\bm\theta_{\overline{s}} = \xi \bm\theta_{s_1} + (1 - \xi) \bm\theta_{s_2}, \; \xi \in [0,1].
\]
In particular, for each user $u \in \mathcal{U}$, the mean value theorem implies that there exists $\xi_u \in [0,1]$ such that the intermediate point is given by $\bm\theta_{\overline{u}} = \xi_u \bm\theta_u + (1 - \xi_u) \hat{\bm\theta}_u$.

Furthermore, we define 
\[
W_u \coloneqq \sum_{i=1}^{N_u} \nabla \sigma\left( \bm\theta_{\overline u}^{\top} \bm z_u^i \right) \bm z_u^i (\bm z_u^i)^{\top} + \lambda I.
\]
Recall that 
\[
M_u = \sum_{i=1}^{N_u} \bm z_u^i (\bm z_u^i)^{\top} + \frac{\lambda}{\kappa} I.
\]
By \Cref{align_kappa}, we have $W_u \succeq \kappa M_u$ and $M_u^{-1} \succeq \kappa W_u^{-1}$ since $\nabla \sigma( \bm\theta_{\overline u}^{\top} \bm z_u^i ) \geq \kappa$. Here, for two symmetric matrices $A_1$ and $A_2$, the notation $A_1 \succeq A_2$ means that $A_1 - A_2$ is positive semi-definite.

Using these properties, we can show that
\begin{align}\label{align_bounding_G_minus_lambda_theta}
\left\| G_u(\hat{\bm\theta}_u) - \lambda \bm\theta_u \right\|_{M_u^{-1}}^2 
&= \left\| G_u(\hat{\bm\theta}_u) - G_u(\bm\theta_u) \right\|_{M_u^{-1}}^2 
= \left\| W_u (\bm\theta_u - \hat{\bm\theta}_u) \right\|_{M_u^{-1}}^2 \notag \\
&= (\bm\theta_u - \hat{\bm\theta}_u)^{\top} W_u M_u^{-1} W_u (\bm\theta_u - \hat{\bm\theta}_u) \notag \\
&\overset{(a)}{\geq} \kappa\, (\bm\theta_u - \hat{\bm\theta}_u)^{\top} W_u (\bm\theta_u - \hat{\bm\theta}_u) \notag \\
&\overset{(b)}{\geq} \kappa^2\, (\bm\theta_u - \hat{\bm\theta}_u)^{\top} M_u (\bm\theta_u - \hat{\bm\theta}_u) 
= \kappa^2 \left\| \bm\theta_u - \hat{\bm\theta}_u \right\|_{M_u}^2,
\end{align}
where (a) follows from $M_u^{-1} \succeq \kappa W_u^{-1}$ and (b) from $W_u \succeq \kappa M_u$.

Moreover, observe that
\begin{align}\label{align_bounding_lambda_theta}
\left\| \lambda \bm\theta_u \right\|_{M_u^{-1}} 
= \lambda \sqrt{ \bm\theta_u^{\top} M_u^{-1} \bm\theta_u }
\leq \sqrt{ \lambda \kappa } \| \bm\theta_u \|_2 
\leq \sqrt{ \lambda \kappa },
\end{align}
where the first inequality uses $M_u \succeq \frac{\lambda}{\kappa} I$ and the second follows from $\| \bm\theta_u \|_2 \leq 1$.

Combining these results, we have
\begin{align}\label{align_bounding_theta_minus_hattheta}
\left\| \bm\theta_u - \hat{\bm\theta}_u \right\|_{M_u} 
&\overset{(a)}{\leq} \frac{1}{\kappa} \left\| G_u( \hat{\bm\theta}_u ) - \lambda \bm\theta_u \right\|_{M_u^{-1}} \notag \\
&\overset{(b)}{\leq} \frac{1}{\kappa} \left\| G_u( \hat{\bm\theta}_u ) \right\|_{M_u^{-1}} + \frac{1}{\kappa} \left\| \lambda \bm\theta_u \right\|_{M_u^{-1}} \notag \\
&\overset{(c)}{\leq} \frac{1}{\kappa} \left\| G_u( \hat{\bm\theta}_u ) \right\|_{M_u^{-1}} + \sqrt{ \frac{\lambda}{\kappa} },
\end{align}
where (a) follows from \eqref{align_bounding_G_minus_lambda_theta}, (b) uses the triangle inequality, and (c) applies \eqref{align_bounding_lambda_theta}.

We then bound the term $\left\| G_u(\hat{\bm\theta}_u) \right\|_{M_u^{-1}}$ as follows:
\begin{align}\label{align_bounding_G}
\left\| G_u(\hat{\bm\theta}_u) \right\|_{M_u^{-1}} 
&= \left\| \sum_{i=1}^{N_u} \left( \sigma( \hat{\bm\theta}_u^{\top} \bm z_u^i ) - \sigma( \bm\theta_u^{\top} \bm z_u^i ) \right) \bm z_u^i + \lambda \hat{\bm\theta}_u \right\|_{M_u^{-1}} \notag \\
&= \left\| \sum_{i=1}^{N_u} \left( \sigma( \hat{\bm\theta}_u^{\top} \bm z_u^i ) - ( y_u^i - \varepsilon_u^i ) \right) \bm z_u^i + \sum_{i=1}^{N_u} \varepsilon_u^i \bm z_u^i + \lambda \hat{\bm\theta}_u \right\|_{M_u^{-1}} \notag \\
&\overset{(a)}{\leq} \left\| \sum_{i=1}^{N_u} \varepsilon_u^i \bm z_u^i \right\|_{M_u^{-1}},
\end{align}
where inequality (a) follows from the fact that $\hat{\bm\theta}_u$ is chosen to minimize the regularized log-likelihood:
\begin{align}\label{align_calculation_of_hat_theta}
\hat{\bm\theta}_u = \arg\min_{\bm\theta} 
\left[ - \sum_{i=1}^{N_u} \left( y_u^i \log \sigma( \bm\theta^{\top} \bm z_u^i ) + (1 - y_u^i) \log \sigma( -\bm\theta^{\top} \bm z_u^i ) \right) + \frac{\lambda}{2} \| \bm\theta \|_2^2 \right],
\end{align}
and thus its gradient satisfies
\[
\sum_{i=1}^{N_u} \left( \sigma( \hat{\bm\theta}_u^{\top} \bm z_u^i ) - y_u^i \right) \bm z_u^i + \lambda \hat{\bm\theta}_u = 0.
\]
Therefore, it follows from \eqref{align_bounding_G} that
\[
\frac{1}{\kappa} \left\| G_u( \hat{\bm\theta}_u ) \right\|_{M_u^{-1}}
\leq 
\frac{1}{\kappa} \left\| \sum_{i=1}^{N_u} \varepsilon_u^i \bm z_u^i \right\|_{M_u^{-1}}.
\]

Next, let $V = \frac{\lambda}{\kappa} I$. Since $\varepsilon_u^i$ is $2$-subgaussian, we apply Theorem~1 in~\citet{abbasi2011improved} to obtain
\begin{align}\label{align_bounding_epsilon_z}
\left\| \sum_{i=1}^{N_u} \varepsilon_u^i \bm z_u^i \right\|_{M_u^{-1}}^2
\leq 8 \log \left( \frac{ \det(M_u)^{1/2} }{ \delta\, \det(V)^{1/2} } \right)
\end{align}
with probability at least $1 - \delta$. Since $\| \bm z_u^i \|_2 \leq 2$, we have
\[
\det(M_u) \leq \left( \frac{\lambda}{\kappa} + \frac{4N_u}{d} \right)^d,
\quad
\det(V) = \left( \frac{\lambda}{\kappa} \right)^d,
\quad
\text{and thus} \quad
\sqrt{ \frac{ \det(M_u) }{ \det(V) } } 
\leq \left( 1 + \frac{4N_u \kappa}{d \lambda} \right)^{ d/2 }.
\]
Therefore,
\[
\left\| \sum_{i=1}^{N_u} \varepsilon_u^i \bm z_u^i \right\|_{M_u^{-1}}^2
\leq 8 \log \left( \tfrac{1}{\delta} \right) + 4d \log \left( 1 + \tfrac{4N_u \kappa}{d \lambda} \right)
\quad \text{with probability at least } 1 - \delta.
\]

Putting everything together, we conclude that
\[
\left\| \bm\theta_u - \hat{\bm\theta}_u \right\|_{M_u}
\leq 
\frac{ \sqrt{ \lambda \kappa } + 2 \sqrt{ 2 \log(1/\delta) + d \log( 1 + 4N_u \kappa / (d \lambda) ) } }{ \kappa }
\quad \text{with probability at least } 1 - \delta,
\]
which follows from combining \eqref{align_bounding_G_minus_lambda_theta}, \eqref{align_bounding_theta_minus_hattheta}, \eqref{align_bounding_G}, and \eqref{align_bounding_epsilon_z}.
\end{proof}

\subsection{Proof of \Cref{lemma_cardinality_of_R_and_W}}

\begin{proof}

In order to prove \Cref{lemma_cardinality_of_R_and_W}, it suffices to show the following statement: under the same conditions as in \Cref{lemma_cardinality_of_R_and_W}, both sets can be characterized as
\begin{align*}
\mathcal{R}_{\hat\gamma}(u) 
&= \Big\{ v \,\Big|\, \bm\theta_u = \bm\theta_v 
\text{ and }
\frac{1}{ \sqrt{ \lambda_{\min}(M_u) } } + \frac{1}{ \sqrt{ \lambda_{\min}(M_v) } }
<  \alpha_r \hat\gamma
\Big\} \cup \{ u \},\\
\mathcal{W}_{\hat\gamma}(u)
&= \Big\{ v \,\Big|\, 
 \gamma \leq \| \bm\theta_u - \bm\theta_v \|_2 < \hat\gamma 
\text{ and }
\frac{1}{ \sqrt{ \lambda_{\min}(M_u) } } + \frac{1}{ \sqrt{ \lambda_{\min}(M_v) } }
< \alpha_w \varepsilon
\Big\}
\end{align*}
for some $\alpha_r \in \left( \frac{\kappa}{3(\alpha+1)\sqrt{2\max\{2,d\}\log(2U/\delta)}}, \frac{\kappa}{2(\alpha - 1)\sqrt{2\log(2U/\delta)}} \right)$ and $\alpha_w \in \left( 0, \frac{\kappa}{2(\alpha-1)\sqrt{2\log(2U/\delta)}} \right)$ with probability at least $1-\delta$.

First, by applying \Cref{lemma_confidence_ellipsoid_of_hat_theta} and a union bound, we have that the event 
\[
\mathcal{E} \coloneqq \bigcap_{u \in \mathcal{U}} \Big\{ \| \hat{\bm\theta}_u - \bm\theta_u \|_2 \leq \text{CI}_u \Big\}
\]
holds with probability at least $1 - \delta/2$.

Recall that the connection condition in Algorithm~\ref{al1} is given by
\[
\left\| \hat{\bm\theta}_{u_1} - \hat{\bm\theta}_{u_2} \right\|_2 
< \hat\gamma - \alpha \left( \text{CI}_{u_1} + \text{CI}_{u_2} \right),
\]
which implies
\begin{align*}
\hat\gamma 
&> \left\| \hat{\bm\theta}_{u_1} - \hat{\bm\theta}_{u_2} \right\|_2 + \alpha \left( \text{CI}_{u_1} + \text{CI}_{u_2} \right) \\
&\geq \left\| \hat{\bm\theta}_{u_1} - \hat{\bm\theta}_{u_2} \right\|_2 + \text{CI}_{u_1} + \text{CI}_{u_2} \\
&\overset{(a)}{\geq} \left\| \hat{\bm\theta}_{u_1} - \hat{\bm\theta}_{u_2} \right\|_2 
+ \left\| \hat{\bm\theta}_{u_1} - \bm\theta_{u_1} \right\|_2 
+ \left\| \hat{\bm\theta}_{u_2} - \bm\theta_{u_2} \right\|_2 \\
&\overset{(b)}{\geq} \left\| \bm\theta_{u_1} - \bm\theta_{u_2} \right\|_2,
\end{align*}
where (a) follows from the event $\mathcal{E}$ and (b) follows by the triangle inequality. Therefore, any pair of connected users must have preference vectors whose difference is no greater than $\hat\gamma$.

Next, we calculate the cardinality of $\mathcal{R}_{\hat\gamma}(u)$.  
Note that for any user $v \in \mathcal{R}_{\hat\gamma}(u)$, it holds that $\bm\theta_u = \bm\theta_v$.  
To prove the claim for $\mathcal{R}_{\hat\gamma}(u)$ in \Cref{lemma_cardinality_of_R_and_W}, it suffices to show the following two conditions under event $\mathcal{E}$:  
\begin{itemize}
  \item[(i)] If 
  $\frac{1}{ \sqrt{ \lambda_{\min}(M_u) } } + \frac{1}{ \sqrt{ \lambda_{\min}(M_v) } } 
  < \frac{ \kappa \hat\gamma }{ 3(\alpha+1) \sqrt{ 2 \max\{2,d\} \log(2U/\delta) } }$  
  then $v$ must be included in $\mathcal{R}_{\hat\gamma}(u)$.
  \item[(ii)] If 
  $\frac{1}{ \sqrt{ \lambda_{\min}(M_u) } } + \frac{1}{ \sqrt{ \lambda_{\min}(M_v) } } 
  \geq \frac{ \kappa \hat\gamma }{ 2(\alpha-1) \sqrt{ 2 \log(2U/\delta) } }$  
  then $v$ must not be included in $\mathcal{R}_{\hat\gamma}(u)$.
\end{itemize}

\paragraph{For (i).}  
Given 
\[
\frac{1}{ \sqrt{ \lambda_{\min}(M_u) } } + \frac{1}{ \sqrt{ \lambda_{\min}(M_v) } }
< \frac{ \kappa \hat\gamma }{ 3(\alpha+1)\sqrt{2 \max\{2,d\} \log(2U/\delta)} },
\]
we have
\begin{align}\label{align_bounding_CI_summation}
&(\alpha+1) \left( \text{CI}_u + \text{CI}_v \right)\\
&\leq \frac{3(\alpha+1)\sqrt{2\log(2U/\delta) + d\log(1 + 4N_u\kappa/(d\lambda))}}{\kappa \sqrt{ \lambda_{\min}(M_u) }}
+ \frac{3(\alpha+1)\sqrt{2\log(2U/\delta) + d\log(1 + 4N_v\kappa/(d\lambda))}}{\kappa \sqrt{ \lambda_{\min}(M_v) }} \notag\\
&\leq \frac{3(\alpha+1)\sqrt{2\max\{2,d\} \log(2U/\delta)}}{\kappa} 
\left( \frac{1}{\sqrt{ \lambda_{\min}(M_u) }} + \frac{1}{\sqrt{ \lambda_{\min}(M_v) }} \right) 
< \hat\gamma,
\end{align}
where the second last inequality holds if $\lambda$ and $\delta$ satisfy 
$\lambda\kappa \leq 2\log(2U/\delta) + d\log(1 + 4N_s\kappa/(d\lambda))$ and $\delta \leq d\lambda/(4N_s\kappa + d\lambda)$ for all $s \in \mathcal{U}$.

Therefore, under event $\mathcal{E}$, we obtain
\[
\left\| \hat{\bm\theta}_{u} - \hat{\bm\theta}_{v} \right\|_2 
\leq \left\| \bm\theta_{u} - \bm\theta_{v} \right\|_2 + \text{CI}_{u} + \text{CI}_{v}
\overset{(a)}{=} \text{CI}_{u} + \text{CI}_{v} 
\overset{(b)}{\leq} \hat\gamma - \alpha(\text{CI}_{u} + \text{CI}_{v}),
\]
where (a) uses $\bm\theta_u = \bm\theta_v$, and (b) follows from \eqref{align_bounding_CI_summation}.  
Hence the connection condition in \Cref{align_clutering_condition_for_similarity} holds, which implies that $v$ will be connected to $u$ with probability at least $1-\delta$.

\paragraph{For (ii).}  
If 
\[
\frac{1}{\sqrt{\lambda_{\min}(M_u)}} + \frac{1}{\sqrt{\lambda_{\min}(M_v)}} 
\geq \frac{ \kappa \hat\gamma }{ 2(\alpha-1)\sqrt{2\log(2U/\delta)} },
\]
then we have
\begin{align}\label{align_bounding_alpha_-1_CI_summation}
(\alpha-1)\left( \text{CI}_u + \text{CI}_v \right) 
&\geq \frac{2(\alpha-1)}{\kappa} \sqrt{ \frac{2\log(2U/\delta)}{ \lambda_{\min}(M_u) } }
+ \frac{2(\alpha-1)}{\kappa} \sqrt{ \frac{2\log(2U/\delta)}{ \lambda_{\min}(M_v) } } \notag\\
&\geq \hat\gamma.
\end{align}
Therefore, it follows that
\[
\hat\gamma - \alpha (\text{CI}_u + \text{CI}_v) 
\leq - (\text{CI}_u + \text{CI}_v) 
= \| \bm\theta_u - \bm\theta_v \|_2 - (\text{CI}_u + \text{CI}_v) 
\leq \| \hat{\bm\theta}_u - \hat{\bm\theta}_v \|_2.
\]
Hence, the connection condition in \Cref{align_clutering_condition_for_similarity} does not hold under event $\mathcal{E}$.  
This verifies that any $v$ satisfying this bound cannot be included in $\mathcal{R}_{\hat\gamma}(u)$, implying 
\[
\alpha_r \in \left( \frac{\kappa}{3(\alpha+1)\sqrt{2\max\{2,d\}\log(2U/\delta)}},\, \frac{\kappa}{2(\alpha-1)\sqrt{2\log(2U/\delta)}} \right].
\]

For the cardinality of $\mathcal{W}_{\hat\gamma}(u)$, note that since both $\lambda_{\min}(M_u)$ and $\lambda_{\min}(M_v)$ are positive, we trivially have $\alpha_w > 0$.  
It remains to show that any heterogeneous user $v$ with
\[
\frac{1}{\sqrt{\lambda_{\min}(M_u)}} + \frac{1}{\sqrt{\lambda_{\min}(M_v)}} 
\geq \frac{ \kappa \hat\gamma }{ 2(\alpha-1)\sqrt{2\log(2U/\delta)} }
\]
cannot be included in $\mathcal{W}_{\hat\gamma}(u)$ under event $\mathcal{E}$.  
By the same argument as in \eqref{align_bounding_alpha_-1_CI_summation}, we have 
$(\alpha-1)(\text{CI}_u + \text{CI}_v) \geq \varepsilon$.  
This yields
\[
\varepsilon - \alpha (\text{CI}_u + \text{CI}_v) 
\leq - (\text{CI}_u + \text{CI}_v) 
\leq \| \bm\theta_u - \bm\theta_v \|_2 - (\text{CI}_u + \text{CI}_v) - \gamma 
\leq \| \hat{\bm\theta}_u - \hat{\bm\theta}_v \|_2 - \gamma,
\]
which implies 
\[
\hat\gamma - \alpha (\text{CI}_u + \text{CI}_v) \leq \| \hat{\bm\theta}_u - \hat{\bm\theta}_v \|_2.
\]
Thus, the connection condition in \Cref{align_clutering_condition_for_similarity} does not hold for such $v$, confirming that it cannot be included in $\mathcal{W}_{\hat\gamma}(u)$.

\end{proof}

\subsection{Proof of \Cref{lemma_confidence_ellipsoid_of_tilde_theta}}\label{appendix_proof_of_lemma_tilde_theta}

\begin{proof}
First, we define
\[
\tilde{G}_u(\bm\theta_s)
= \sum_{v \in \mathcal{V}_{\hat\gamma}(u)} \sum_{i=1}^{N_v} 
\left( \sigma\left( \bm\theta_s^{\top} \bm z_v^i \right)
- \sigma\left( \bm\theta_u^{\top} \bm z_v^i \right) \right) \bm z_v^i 
+ \lambda \bm\theta_s,
\quad \forall\, \bm\theta_s \in \mathbb{R}^d.
\]

By the mean value theorem, for any $\bm\theta_{s_1}$ and $\bm\theta_{s_2}$, we have
\[
\tilde{G}_u(\bm\theta_{s_1}) - \tilde{G}_u(\bm\theta_{s_2})
= \left(
\sum_{v \in \mathcal{V}_{\hat\gamma}(u)} \sum_{i=1}^{N_v}
\nabla \sigma\left( \bm\theta_{\overline{s}}^{\top} \bm z_v^i \right)
\bm z_v^i \bm z_v^{i\top} + \lambda I 
\right)
\left( \bm\theta_{s_1} - \bm\theta_{s_2} \right),
\]
for some intermediate point 
$\bm\theta_{\overline{s}} = \xi \bm\theta_{s_1} + (1 - \xi) \bm\theta_{s_2}$ 
with $\xi \in [0, 1]$.  
In particular, for each $u \in \mathcal{U}$, we let 
$\xi_u \in [0, 1]$ and define the corresponding intermediate point 
$\bm\theta_{\overline{u}} = \xi_u \bm\theta_u + (1 - \xi_u) \tilde{\bm\theta}_u$.

We further define
\[
\tilde{W}_u = 
\sum_{v \in \mathcal{V}_{\hat\gamma}(u)} \sum_{i=1}^{N_v}
\nabla \sigma\left( \bm\theta_{\overline{u}}^{\top} \bm z_v^i \right)
\bm z_v^i \bm z_v^{i\top} + \lambda I
\quad \text{and} \quad
\tilde{M}_u = 
\sum_{v \in \mathcal{V}_{\hat\gamma}(u)} \sum_{i=1}^{N_v}
\bm z_v^i \bm z_v^{i\top} + \frac{\lambda}{\kappa} I.
\]
By construction, it holds that $\tilde{W}_u \succeq \kappa \tilde{M}_u$ and thus 
$\tilde{M}_u^{-1} \succeq \kappa \tilde{W}_u^{-1}$ for all $u \in \mathcal{U}$.

Then, we have
\begin{align}\label{align_bounding_tilde_G}
\big\| \tilde{G}_u(\tilde{\bm\theta}_u) - \lambda \bm\theta_u \big\|_{ \tilde{M}_u^{-1} }^2
&= \big\| \tilde{G}_u(\tilde{\bm\theta}_u) - \tilde{G}_u(\bm\theta_u) \big\|_{ \tilde{M}_u^{-1} }^2
= \big\| \tilde{W}_u ( \bm\theta_u - \tilde{\bm\theta}_u ) \big\|_{ \tilde{M}_u^{-1} }^2 \notag \\
&= ( \bm\theta_u - \tilde{\bm\theta}_u )^{\top}
\tilde{W}_u \tilde{M}_u^{-1} \tilde{W}_u 
( \bm\theta_u - \tilde{\bm\theta}_u ) \notag \\
&\overset{(a)}{\geq}
\kappa\, ( \bm\theta_u - \tilde{\bm\theta}_u )^{\top}
\tilde{W}_u ( \bm\theta_u - \tilde{\bm\theta}_u )
\notag \\
&\overset{(b)}{\geq}
\kappa^2\, ( \bm\theta_u - \tilde{\bm\theta}_u )^{\top}
\tilde{M}_u ( \bm\theta_u - \tilde{\bm\theta}_u )
= \kappa^2\, \big\| \bm\theta_u - \tilde{\bm\theta}_u \big\|_{ \tilde{M}_u }^2,
\end{align}
where (a) follows from $\tilde{M}_u^{-1} \succeq \kappa \tilde{W}_u^{-1}$ and (b) follows from $\tilde{W}_u \succeq \kappa \tilde{M}_u$.

Moreover, since $\tilde{M}_u \succeq \frac{\lambda}{\kappa} I$, we have
\begin{align}\label{align_bounding_lambdatheta}
\big\| \lambda \bm\theta_u \big\|_{ \tilde{M}_u^{-1} }
= \lambda\, \sqrt{ \bm\theta_u^{\top} \tilde{M}_u^{-1} \bm\theta_u }
\leq \lambda\, \sqrt{ \bm\theta_u^{\top} \left( \frac{\kappa}{\lambda} I \right) \bm\theta_u }
= \sqrt{ \lambda \kappa }\, \big\| \bm\theta_u \big\|_2
\leq \sqrt{ \lambda \kappa }.
\end{align}

Hence, we obtain
\begin{align}
\big\| \bm\theta_u - \tilde{\bm\theta}_u \big\|_{ \tilde{M}_u }
&\overset{(a)}{\leq} 
\frac{1}{\kappa} 
\big\| \tilde{G}_u( \tilde{\bm\theta}_u ) - \lambda \bm\theta_u \big\|_{ \tilde{M}_u^{-1} } 
\notag \\
&\overset{(b)}{\leq} 
\frac{1}{\kappa} 
\big\| \tilde{G}_u( \tilde{\bm\theta}_u ) \big\|_{ \tilde{M}_u^{-1} }
+ \frac{1}{\kappa} 
\big\| \lambda \bm\theta_u \big\|_{ \tilde{M}_u^{-1} }
\notag \\
&\overset{(c)}{\leq} 
\frac{1}{\kappa} 
\big\| \tilde{G}_u( \tilde{\bm\theta}_u ) \big\|_{ \tilde{M}_u^{-1} }
+ \sqrt{ \frac{ \lambda }{ \kappa } },
\label{align_bounding_theta_minus_tilde_theta}
\end{align}
where (a) follows from~\Cref{align_bounding_tilde_G}, (b) applies the triangle inequality, and (c) uses the bound in~\Cref{align_bounding_lambdatheta}.

Furthermore, we can bound $\tilde{G}_u(\tilde{\bm\theta}_u)$ as follows:
\begin{align}
& \frac{1}{\kappa^2} \Big\| \tilde{G}_u( \tilde{\bm\theta}_u ) \Big\|_{ \tilde{M}_u^{-1} }^2 \notag \\
&\overset{(a)}{=} \frac{1}{\kappa^2} 
\Big\| \sum_{v \in \mathcal{V}_{\hat\gamma}(u)} \sum_{i=1}^{N_v} 
\left( \sigma( \tilde{\bm\theta}_u^{\top} \bm z_v^i ) - \sigma( \bm\theta_u^{\top} \bm z_v^i ) \right) \bm z_v^i 
+ \lambda \tilde{\bm\theta}_u \Big\|_{ \tilde{M}_u^{-1} }^2 \notag \\
&= \frac{1}{\kappa^2} 
\Big\| \sum_{v} \sum_{i} 
\left( \sigma( \tilde{\bm\theta}_u^{\top} \bm z_v^i ) - y_v^i + y_v^i - \sigma( \bm\theta_u^{\top} \bm z_v^i ) \right) \bm z_v^i 
+ \lambda \tilde{\bm\theta}_u \Big\|_{ \tilde{M}_u^{-1} }^2 \notag \\
&= \frac{1}{\kappa^2} 
\Big\| \sum_{v} \sum_{i} 
\left( \sigma( \tilde{\bm\theta}_u^{\top} \bm z_v^i ) - y_v^i \right) \bm z_v^i 
+ \lambda \tilde{\bm\theta}_u 
+ \sum_{v} \sum_{i} \left( y_v^i - \sigma( \bm\theta_u^{\top} \bm z_v^i ) \right) \bm z_v^i \Big\|_{ \tilde{M}_u^{-1} }^2 \notag \\
&\overset{(b)}{=} \frac{1}{\kappa^2} 
\Big\| \sum_{v} \sum_{i} \left( y_v^i - \sigma( \bm\theta_v^{\top} \bm z_v^i ) + \sigma( \bm\theta_v^{\top} \bm z_v^i ) - \sigma( \bm\theta_u^{\top} \bm z_v^i ) \right) \bm z_v^i \Big\|_{ \tilde{M}_u^{-1} }^2 \notag \\
&= \frac{1}{\kappa^2} 
\Big\| \underbrace{ \sum_{v} \sum_{i} \varepsilon_v^i \bm z_v^i }_{ \text{noise} } 
+ \underbrace{ \sum_{v} \sum_{i} \left( \sigma( \bm\theta_v^{\top} \bm z_v^i ) - \sigma( \bm\theta_u^{\top} \bm z_v^i ) \right) \bm z_v^i }_{ \text{bias} } 
\Big\|_{ \tilde{M}_u^{-1} }^2 \notag \\
&\overset{(c)}{\leq} 
\left( \frac{1}{\kappa} \Big\| \sum_{v} \sum_{i} \varepsilon_v^i \bm z_v^i \Big\|_{ \tilde{M}_u^{-1} }
+ \frac{1}{\kappa} \Big\| \sum_{v} \sum_{i} \left( \sigma( \bm\theta_v^{\top} \bm z_v^i ) - \sigma( \bm\theta_u^{\top} \bm z_v^i ) \right) \bm z_v^i \Big\|_{ \tilde{M}_u^{-1} } \right)^2 \notag \\
&\overset{(d)}{=} 
\left( \frac{1}{\kappa} \Big\| \sum_{v} \sum_{i} \varepsilon_v^i \bm z_v^i \Big\|_{ \tilde{M}_u^{-1} }
+ \frac{1}{\kappa} \Big\| \sum_{v \in \mathcal{W}_{\hat\gamma}(u)} \sum_{i} \left( \sigma( \bm\theta_v^{\top} \bm z_v^i ) - \sigma( \bm\theta_u^{\top} \bm z_v^i ) \right) \bm z_v^i \Big\|_{ \tilde{M}_u^{-1} } \right)^2.
\label{align_bounding_tilde_G}
\end{align}

Here, (a) follows from the definition of $\tilde{G}_u(\tilde{\bm\theta}_u)$;  
(b) holds since $\tilde{\bm\theta}_u$ minimizes the negative log-likelihood regularized by $\lambda$, implying
\[
\sum_{v} \sum_{i} \left( \sigma( \tilde{\bm\theta}_u^{\top} \bm z_v^i ) - y_v^i \right) \bm z_v^i + \lambda \tilde{\bm\theta}_u = 0;
\]
(c) uses the triangle inequality; and (d) uses the fact that for any homogeneous neighbor $v \in \mathcal{R}_{\hat\gamma}(u)$, we have $\bm\theta_u = \bm\theta_v$, so only the heterogeneous neighbors contribute to the bias term.

Next, we bound the term 
\[
\Big\| \sum_{v \in \mathcal{W}_{\hat\gamma}(u)} \sum_{i=1}^{N_v} 
\left( \sigma( \bm\theta_v^{\top} \bm z_v^i ) - \sigma( \bm\theta_u^{\top} \bm z_v^i ) \right) \bm z_v^i \Big\|_{ \tilde M_u^{-1} }.
\]
By the triangle inequality, we have
\begin{align}\label{align_bounding_sigma_minus_sigma}
&\quad \Big\| \sum_{v \in \mathcal{W}_{\hat\gamma}(u)} \sum_{i=1}^{N_v} 
\left( \sigma( \bm\theta_v^{\top} \bm z_v^i ) - \sigma( \bm\theta_u^{\top} \bm z_v^i ) \right) \bm z_v^i \Big\|_{ \tilde M_u^{-1} } \notag \\
&\leq \sum_{v \in \mathcal{W}_{\hat\gamma}(u)} \sum_{i=1}^{N_v} 
\Big| \sigma( \bm\theta_v^{\top} \bm z_v^i ) - \sigma( \bm\theta_u^{\top} \bm z_v^i ) \Big| 
\big\| \bm z_v^i \big\|_{ \tilde M_u^{-1} } \notag \\
&\overset{(a)}{\leq} \sum_{v} \sum_{i} 
\frac{1}{4} \big| \bm\theta_v^{\top} \bm z_v^i - \bm\theta_u^{\top} \bm z_v^i \big| \big\| \bm z_v^i \big\|_{ \tilde M_u^{-1} } 
\leq \frac{\hat\gamma}{4} \sum_{v} \sum_{i} \| \bm z_v^i \|_2 \| \bm z_v^i \|_{ \tilde M_u^{-1} } \notag \\
&\overset{(b)}{\leq} \frac{\hat\gamma}{2} 
\sum_{v \in \mathcal{W}_{\hat\gamma}(u)} \sum_{i=1}^{N_v} 
\big\| \bm z_v^i \big\|_{ \tilde M_u^{-1} },
\end{align}
where (a) follows from the Lipschitz continuity of the sigmoid function with constant $L_\sigma = \frac{1}{4}$, and (b) uses $\| \bm z_v^i \|_2 \leq 2$.

Furthermore, observe that
\[
\sum_{v \in \mathcal{W}_{\hat\gamma}(u)} \sum_{i=1}^{N_v} 
\big\| \bm z_v^i \big\|_{ \tilde M_u^{-1} }^2 
= \mathrm{tr}\left( \tilde M_u^{-1} \left( \tilde M_u - \frac{\lambda}{\kappa} I \right) \right) 
\leq d.
\]
By applying Cauchy--Schwarz inequality, we get
\begin{align}\label{align_bounding_summation_zvi}
\sum_{v} \sum_{i} \big\| \bm z_v^i \big\|_{ \tilde M_u^{-1} }
&\leq \sqrt{ 
\left( \sum_{v} N_v \right) 
\left( \sum_{v} \sum_{i} \big\| \bm z_v^i \big\|_{ \tilde M_u^{-1} }^2 \right) 
} 
\leq \sqrt{ d \cdot N_{ \mathcal{W}_{\hat\gamma}(u) } }.
\end{align}

Combining the above, the bias term due to heterogeneous neighbors is bounded accordingly.

Therefore, by applying \Cref{align_bounding_sigma_minus_sigma} and (\ref{align_bounding_summation_zvi}), we obtain
\begin{align}\label{align_bounding_sigma_theta_minus_sigma_theta}
\Big\|
\sum_{v \in \mathcal{W}_{\hat\gamma}(u)} \sum_{i=1}^{N_v}
\left( \sigma( \bm\theta_v^{\top} \bm z_v^i ) - \sigma( \bm\theta_u^{\top} \bm z_v^i ) \right) \bm z_v^i 
\Big\|_{ \tilde M_u^{-1} }
\leq \frac{ \hat\gamma }{ 2 } \sqrt{ d\, N_{ \mathcal{W}_{\hat\gamma}(u) } },
\end{align}
where $N_{ \mathcal{W}_{\hat\gamma}(u) } = \sum_{v \in \mathcal{W}_{\hat\gamma}(u) } N_v$.

Furthermore, for the noise term in \Cref{align_bounding_tilde_G}, by applying Theorem 1 in~\citet{abbasi2011improved} with $V = \frac{ \lambda }{ \kappa } I$, we have
\begin{align}\label{align_bounding_summation_epsilon_times_z}
\Big\|
\sum_{v \in \mathcal{V}_{\hat\gamma}(u)} \sum_{i=1}^{N_v} \varepsilon_v^i \bm z_v^i 
\Big\|_{ \tilde M_u^{-1} }
\leq 2 \sqrt{ 2 \log \left( \frac{ \det( \tilde M_u )^{1/2} }{ \delta\, \det(V)^{1/2} } \right) }
\leq 2 \sqrt{ 2 \log\left( \frac{1}{\delta} \right) + d\, \log\left( 1 + \frac{ 4\, N_{\mathcal{V}_{\hat\gamma}(u)} \kappa }{ d \lambda } \right) }
\end{align}
with probability at least $1 - \delta$, where $N_{\mathcal{V}_{\hat\gamma}(u)} = \sum_{v \in \mathcal{V}_{\hat\gamma}(u)} N_v$.

Combining \Cref{align_bounding_theta_minus_tilde_theta}, \Cref{align_bounding_tilde_G}, \Cref{align_bounding_sigma_theta_minus_sigma_theta}, and (\ref{align_bounding_summation_epsilon_times_z}), we finally have
\[
\Big\| \bm\theta_u - \tilde{\bm\theta}_u \Big\|_{ \tilde M_u }
\leq 
\frac{ \sqrt{ \lambda \kappa } + 2 \sqrt{ 2 \log\left( \frac{ 2U }{ \delta } \right) + d\, \log\left( 1 + \frac{ 4\, N_{\mathcal{V}_{\hat\gamma}(u)} \kappa }{ d \lambda } \right) } }{ \kappa }
+ \frac{ \hat\gamma }{ 2 } \sqrt{ d\, N_{ \mathcal{W}_{\hat\gamma}(u) } },
\]
which holds for all $u \in \mathcal{U}$ with probability at least $1 - \delta$. This completes the proof of Lemma~\ref{lemma_confidence_ellipsoid_of_tilde_theta}.
\end{proof}

\subsection{Proof of \Cref{th1}}\label{appendix_proof_of_th1}

\begin{proof}
By Lemmas~\ref{lemma_confidence_ellipsoid_of_hat_theta} and \ref{lemma_confidence_ellipsoid_of_tilde_theta}, we have
\begin{align}\label{align_result_from_lemma1_and_3}
\left\| \bm\theta_u - \tilde{\bm\theta}_u \right\|_{\tilde M_u}
\leq \tilde{\beta}_u + \frac{\hat\gamma}{2} \sqrt{ d\, N_{\mathcal{W}_{\hat\gamma}(u)} }
\end{align}
for all $u \in \mathcal{U}$ with probability at least $1-\delta$.

For simplicity, let $u = u_{t}$ denote the test user. Define $J_u'(\pi) = J_u(\pi) - \langle \bm\theta_u, \bm w \rangle$. Then, the suboptimality gap can be written as:
\begin{align*}
\text{SubOpt}_u(\pi_u)
& = J_u(\pi_u^*) - J_u(\pi_u) 
= J_u'(\pi_u^*) - J_u'(\pi_u) \\
& = \left( J_u'(\pi_u^*) - \tilde J_u(\pi_u^*) \right)
+ \left( \tilde J_u(\pi_u^*) - \tilde J_u(\pi_u) \right)
+ \left( \tilde J_u(\pi_u) - J_u'(\pi_u) \right).
\end{align*}

For the second term, since $\pi_u = \arg\max_{\pi} \tilde J_u(\pi)$, we have $\tilde J_u(\pi_u^*) - \tilde J_u(\pi_u) \leq 0$.

For the third term:
\begin{align*}
\tilde J_u(\pi_u) - J_u'(\pi_u)
&= \left( \mathbb{E}_{\bm x \sim \rho_p} [\phi(\bm x, \pi_u(\bm x))] - \bm w \right)^{\top} (\tilde{\bm\theta}_u - \bm\theta_u)
- \tilde{\beta}_u \Big\| \mathbb{E}_{\bm x \sim \rho_p} [\phi(\bm x, \pi_u(\bm x))] - \bm w \Big\|_{\tilde M_u^{-1}} \\
&\leq \Big\| \mathbb{E}_{\bm x \sim \rho_p} [\phi(\bm x, \pi_u(\bm x))] - \bm w \Big\|_{\tilde M_u^{-1}}
\left( \left\| \tilde{\bm\theta}_u - \bm\theta_u \right\|_{\tilde M_u} - \tilde{\beta}_u \right) \\
&\overset{(a)}{\leq} \frac{\hat\gamma}{2} \sqrt{ d\, N_{\mathcal{W}_{\hat\gamma}(u)} }
\, \Big\| \mathbb{E}_{\bm x \sim \rho_p} [\phi(\bm x, \pi_u(\bm x))] - \bm w \Big\|_{\tilde M_u^{-1}},
\end{align*}
where $(a)$ uses \eqref{align_result_from_lemma1_and_3}.

Similarly, for the first term:
\begin{align*}
J_u'(\pi_u^*) - \tilde J_u(\pi_u^*)
&= \left( \bm\theta_u - \tilde{\bm\theta}_u \right)^{\top} \left( \mathbb{E}_{\bm x \sim \rho_p} [\phi(\bm x, \pi_u^*(\bm x))] - \bm w \right)
+ \tilde{\beta}_u \Big\| \mathbb{E}_{\bm x \sim \rho_p} [\phi(\bm x, \pi_u^*(\bm x))] - \bm w \Big\|_{\tilde M_u^{-1}} \\
&\leq \left( \left\| \bm\theta_u - \tilde{\bm\theta}_u \right\|_{\tilde M_u} + \tilde{\beta}_u \right)
\Big\| \mathbb{E}_{\bm x \sim \rho_p} [\phi(\bm x, \pi_u^*(\bm x))] - \bm w \Big\|_{\tilde M_u^{-1}} \\
&\leq \left( 2\tilde{\beta}_u + \frac{\hat\gamma}{2} \sqrt{ d\, N_{\mathcal{W}_{\hat\gamma}(u)} } \right)
\Big\| \mathbb{E}_{\bm x \sim \rho_p} [\phi(\bm x, \pi_u^*(\bm x))] - \bm w \Big\|_{\tilde M_u^{-1}}.
\end{align*}

Putting everything together, we obtain:
\begin{align*}
\text{SubOpt}_u(\pi_u)
&\leq \left( 2\tilde{\beta}_u + \hat\gamma \sqrt{ d\, N_{\mathcal{W}_{\hat\gamma}(u)} } \right)
\Big\| \mathbb{E}_{\bm x \sim \rho_p} [\phi(\bm x, \pi_u^*(\bm x))] - \bm w \Big\|_{\tilde M_u^{-1}} \\
&\leq \left( \frac{ 2\sqrt{2\log\left( \frac{2U}{\delta} \right) + d\log\left( 1 + \frac{4\tilde N_u \kappa}{d\lambda} \right) } }{ \kappa }
+ \hat\gamma \sqrt{ d\, N_{\mathcal{W}_{\hat\gamma}(u)} } \right)
\Big\| \mathbb{E}_{\bm x \sim \rho_p} [\phi(\bm x, \pi_u^*(\bm x))] - \bm w \Big\|_{\tilde M_u^{-1}} \\
&= \tilde O\left( \sqrt{d} \left( 1 + \hat\gamma \sqrt{ N_{\mathcal{W}_{\hat\gamma}(u)} } \right)
\, \Big\| \mathbb{E}_{\bm x \sim \rho_p} [\phi(\bm x, \pi_u^*(\bm x))] - \bm w \Big\|_{\tilde M_u^{-1}} \right),
\end{align*}
which concludes the proof of Theorem~\ref{th1}.
\end{proof}

\subsection{Proof of \Cref{lemma_cardinality_under_item_regularity}}

\begin{proof}
In this proof, we define
\[
\text{CI}_u = 
\frac{
\sqrt{\lambda\kappa} + 2\sqrt{\, d\, \log\left( 1 + \frac{4\kappa N_u}{\lambda d} \right) + 2\log\left( \frac{2U}{\delta} \right)}
}{
\kappa\,\sqrt{\tilde\lambda_a N_u / 2\,}
}.
\]
By \Cref{lemma_confidence_ellipsoid_of_hat_theta}, Lemma J.1 in \citet{wang2023onlinea} and Lemma 7 in \citep{li2018online}, it holds that 
$\lambda_{\min}(M_u) \geq \tilde\lambda_a N_u / 2$ 
for all users connected to user $u$ with probability at least $1 - \delta/2$. Therefore, we have
\[
\left\| \hat{\bm\theta}_u - \bm\theta_u \right\|_2
\leq 
\frac{
\sqrt{ \lambda \kappa } + 2 \sqrt{ 2 \log \left( \frac{2U}{\delta} \right) + d \log \left( 1 + \frac{4 N_u \kappa}{ d \lambda } \right) }
}{
\kappa\, \sqrt{ \lambda_{\min}(M_u) }
}
\leq 
\text{CI}_u
\]
with probability at least $1-\delta$. 

Finally, by following the same argument used in the proof of \Cref{lemma_cardinality_of_R_and_W}, but replacing $\lambda_{\min}(M_u)$ with the explicit bound on $N_u$ under \Cref{assumption_item_regularity}, we obtain the desired result in \Cref{lemma_cardinality_under_item_regularity}.
\end{proof}

\subsection{Proof of \Cref{th1_item_regularity}}

\begin{proof}

We denote 
\(\eta_{ \mathcal{W}_{\hat\gamma}(u) } \coloneqq N_{\mathcal{W}_{\hat\gamma}(u)} / N_{ \mathcal{V}_{\hat\gamma}(u) }\) 
for clarity, then it follows that
\begin{align*}
\text{SubOpt}_u(\pi_u)
& \leq 
\tilde O\left(
\frac{\sqrt{d}\left(1+\hat\gamma\sqrt{N_{\mathcal{W}_{\hat\gamma}(u_t)}}\right)}{\sqrt{\lambda_{\min}(\tilde M_{u_t})}}
\right) \\ 
& \leq 
\tilde O\left( 
\sqrt{\tfrac{d}{\tilde\lambda_a}} \left( 
\sqrt{\tfrac{1}{ N_{ \mathcal{V}_{\hat\gamma}(u) } }} 
+ \frac{ \hat\gamma \sqrt{N_{ \mathcal{W}_{\hat\gamma}(u) }} }{ \sqrt{N_{ \mathcal{V}_{\hat\gamma}(u) }} } 
\right)
\right) \\
& \leq 
\tilde O\left( 
\sqrt{\tfrac{d}{\tilde\lambda_a}} \left( 
\sqrt{\tfrac{1}{ N_{ \mathcal{V}_{\hat\gamma}(u) } }} 
+ \hat\gamma \sqrt{\eta_{ \mathcal{W}_{\hat\gamma}(u) }} 
\right) \right).
\end{align*}
Here the first inequality follows directly from \Cref{th1}, while the second inequality applies Lemma~J.1 in~\citet{wang2023onlinea} and Lemma~7 in~\citet{li2018online}. This completes the proof of \Cref{th1_item_regularity}.

\end{proof}

\subsection{Proof of \Cref{th2}}

\begin{proof}
To simplify the notation, we write $u = u_{t}$. We define
\begin{align}
& \text{SubOpt}_u(\pi_u, \bm x)
\coloneqq 
\bm\theta_u^{\top} \left( \phi(\bm x, \pi_u^*(\bm x)) - \phi(\bm x, \pi_u(\bm x)) \right), \notag
\\& \overline\beta_{u}^n
\coloneqq 
\frac{
2\sqrt{\, d\, \log\left( 1 + \frac{4\kappa(\tilde N_u + n)}{\lambda d} \right) + 2 \log(2N/\delta) } 
+ \sqrt{\, \lambda \kappa\, }
}{\kappa}. \notag
\end{align}

First, note that by \Cref{lemma_confidence_ellipsoid_of_tilde_theta} and \Cref{le_confidence_ellipsoid}, since the cardinality of the heterogeneous neighbor set $\mathcal{W}_{\hat\gamma}(u)$ remains unchanged during the online phase, we have
\begin{align}\label{align_ellipsoid_ball_of_tilde_theta_t}
\Big\|\, \bm\theta_u - \tilde{\bm\theta}_u^n\, \Big\|_{\tilde M_u^n}
\;\leq\; 
\overline\beta_u^n 
+ 
\frac{\hat\gamma}{2}\, \sqrt{\, d\, N_{\mathcal{W}_{\hat\gamma}(u)}\, }
\quad 
\text{for each } n \in [N],
\end{align}
with probability at least $1 - \frac{\delta}{2N}$. By applying a union bound over all $n \in [N]$, this bound holds uniformly for all rounds with probability at least $1-\delta$.

We now bound $\text{SubOpt}_u(\pi_u, \bm x)$. It holds that
\begin{align}
& \text{SubOpt}_u(\pi_u, \bm x)
\\&= \bm\theta_u^{\top} \left( \phi(\bm x, \pi_u^*(\bm x)) - \phi(\bm x, \pi_u(\bm x)) \right) \notag \\
&\overset{(a)}{\leq}
\bm\theta_u^{\top} \left( \phi(\bm x, \pi_u^*(\bm x)) - \phi(\bm x, \pi_u(\bm x)) \right)
+ \overline{\bm\theta}_u^{\top} \left( \phi(\bm x, \pi_u(\bm x)) - \phi(\bm x, \pi_u^*(\bm x)) \right) \notag \\
&= \left( \bm\theta_u - \overline{\bm\theta}_u \right)^{\top}
\left( \phi(\bm x, \pi_u^*(\bm x)) - \phi(\bm x, \pi_u(\bm x)) \right) \notag \\
&= \left( \bm\theta_u -
\frac{1}{ d\, \lambda_{\min}\left( \tilde M_u^{N} \right) + N }
\left( d\, \lambda_{\min}\left( \tilde M_u^{N} \right)\, \tilde{\bm\theta}_u^{N} + \sum_{n=1}^{N} \tilde{\bm\theta}_u^n \right)
\right)^{\top}
\left( \phi(\bm x, \pi_u^*(\bm x)) - \phi(\bm x, \pi_u(\bm x)) \right) \notag \\
&= \frac{1}{ d\, \lambda_{\min}\left( \tilde M_u^{N} \right) + N }
\left(
d\, \lambda_{\min}\left( \tilde M_u^{N} \right)
\left( \bm\theta_u - \tilde{\bm\theta}_u^{N} \right)^{\top}
+ \sum_{n=1}^{N} \left( \bm\theta_u - \tilde{\bm\theta}_u^n \right)^{\top}
\right)
\left( \phi(\bm x, \pi_u^*(\bm x)) - \phi(\bm x, \pi_u(\bm x)) \right), 
\label{align_decomposing_SubOpt_u}
\end{align}
where $(a)$ holds due to the fact that $\pi_u$ maximizes the pessimistic value (line 9 in Algorithm~\ref{al2}).

Next, for the first term in \eqref{align_decomposing_SubOpt_u}, we have:
\begin{align}
& \left( \bm\theta_u - \tilde{\bm\theta}_u^{N} \right)^{\top}
\left( \phi(\bm x, \pi_u^*(\bm x)) - \phi(\bm x, \pi_u(\bm x)) \right) \notag \\
&\overset{(a)}{\leq}
\big\| \bm\theta_u - \tilde{\bm\theta}_u^{N} \big\|_2
\big\| \phi(\bm x, \pi_u^*(\bm x)) - \phi(\bm x, \pi_u(\bm x)) \big\|_2 \notag \\
&\overset{(b)}{\leq}
2 \frac{ \big\| \bm\theta_u - \tilde{\bm\theta}_u^{N} \big\|_{ \tilde M_u^{N} } }
{ \sqrt{ \lambda_{\min}\left( \tilde M_u^{N} \right) } } \notag \\
&\overset{(c)}{\leq}
\frac{ 2\, \overline\beta_u^{N} + \hat\gamma\, \sqrt{\, d\, N_{\mathcal{W}_{\hat\gamma}(u)}} }{ \sqrt{ \lambda_{\min}\left( \tilde M_u^{N} \right) } }.
\label{align_bounding_the_first_part_of_SubOpt}
\end{align}
Here, $(a)$ follows from the Cauchy–Schwarz inequality; $(b)$ uses the fact that feature vectors are bounded by 1 in norm and the definition of the minimum eigenvalue; $(c)$ follows from \eqref{align_ellipsoid_ball_of_tilde_theta_t}.

For the summation term in \eqref{align_decomposing_SubOpt_u}, we have:
\begin{align}
& \sum_{n=1}^{N} \left( \bm\theta_u - \tilde{\bm\theta}_u^n \right)^{\top}
\left( \phi(\bm x, \pi_u^*(\bm x)) - \phi(\bm x, \pi_u(\bm x)) \right) \notag \\
&\leq \sum_{n=1}^{N}
\big\| \bm\theta_u - \tilde{\bm\theta}_u^n \big\|_{ \tilde M_u^n }
\big\| \phi(\bm x, \pi_u^*(\bm x)) - \phi(\bm x, \pi_u(\bm x)) \big\|_{ ( \tilde M_u^n )^{-1} } \notag \\
&\overset{(a)}{\leq} 
\sum_{n=1}^{N} 
\big\| \bm\theta_u - \tilde{\bm\theta}_u^n \big\|_{ \tilde M_u^n }
\big\| \phi( \mathring{\bm x}_u^n, \mathring{\bm a}_u^n ) - \phi( \mathring{\bm x}_u^n, \mathring{\bm a}^{\prime n}_{\ u} ) \big\|_{ ( \tilde M_u^n )^{-1} } \notag \\
&\overset{(b)}{\leq} 
\sum_{n=1}^{N}
\left( 2\, \overline\beta_u^n + \hat\gamma\, \sqrt{\, d\, N_{\mathcal{W}_{\hat\gamma}(u)}} \right)
\big\| \phi( \mathring{\bm x}_u^n, \mathring{\bm a}_u^n ) - \phi( \mathring{\bm x}_u^n, \mathring{\bm a}^{\prime n}_{\ u} ) \big\|_{ ( \tilde M_u^n )^{-1} } \notag \\
&\overset{(c)}{\leq}
\left( 2\, \overline\beta_u^{N} + \hat\gamma\, \sqrt{\, d\, N_{\mathcal{W}_{\hat\gamma}(u)}} \right)
\sum_{n=1}^{N}
\big\| \phi( \mathring{\bm x}_u^n, \mathring{\bm a}_u^n ) - \phi( \mathring{\bm x}_u^n, \mathring{\bm a}^{\prime n}_{\ u} ) \big\|_{ ( \tilde M_u^n )^{-1} } \notag \\
&\overset{(d)}{\leq}
\left( 2\, \overline\beta_u^{N} + \hat\gamma\, \sqrt{\, d\, N_{\mathcal{W}_{\hat\gamma}(u)}} \right)
\sqrt{N} \,
\sqrt{
\sum_{n=1}^{N}
\big\| \phi( \mathring{\bm x}_u^n, \mathring{\bm a}_u^n ) - \phi( \mathring{\bm x}_u^n, \mathring{\bm a}^{\prime n}_{\ u} ) \big\|_{ ( \tilde M_u^n )^{-1} }^2
} \notag \\
&\overset{(e)}{\leq}
\left( 2\, \overline\beta_u^{N} + \hat\gamma\, \sqrt{\, d\, N_{\mathcal{W}_{\hat\gamma}(u)}} \right)
\sqrt{\, 2dN\, \log\left( 1 + \frac{4\kappa N}{\lambda d} \right) \, }.
\label{align_bounding_the_second_part_of_SubOpt}
\end{align}
Here, $(a)$ holds by the active data augmentation rule in line 4 of Algorithm~\ref{al2}; 
$(b)$ uses the ellipsoid bound \eqref{align_ellipsoid_ball_of_tilde_theta_t}; 
$(c)$ holds because $\overline\beta_u^n$ is non-decreasing in $n$;
$(d)$ applies the Cauchy–Schwarz inequality; 
and $(e)$ follows from the elliptical potential lemma (\Cref{le_elliptic_potential_lemma}).

Combining \Cref{align_decomposing_SubOpt_u}, \Cref{align_bounding_the_first_part_of_SubOpt}, and \Cref{align_bounding_the_second_part_of_SubOpt} yields:
\begin{align*}
\text{SubOpt}_u(\pi_u, \bm x)
&\leq 
\left( \frac{1}{\, d\, \lambda_{\min}\left( \tilde M_u^{N} \right) + N } \right)
\left( 2\, \overline\beta_u^{N} + \hat\gamma\, \sqrt{\, d\, N_{\mathcal{W}_{\hat\gamma}(u)}} \right)
\left( d\, \sqrt{\, \lambda_{\min}\left( \tilde M_u^{N} \right) } + \sqrt{\, 2dN\, \log\left( 1 + \tfrac{4\kappa N}{\lambda d} \right) } \right) \\
&\leq 
\left( \frac{1}{\, d\, \lambda_{\min}\left( \tilde M_u^{N} \right) + N } \right)
\left( 2\, \overline\beta_u^{N}\sqrt{d} + \hat\gamma\, d\, \sqrt{\, N_{\mathcal{W}_{\hat\gamma}(u)}} \right)
\sqrt{\, 2 \left( d\, \lambda_{\min}\left( \tilde M_u^{N} \right) + 2N\, \log\left( 1 + \tfrac{4\kappa N}{\lambda d} \right) \right) } \\
&= \tilde O\left(
\frac{\, d \left( 1 + \hat\gamma\, \sqrt{\, N_{\mathcal{W}_{\hat\gamma}(u)}} \right) }{
\sqrt{\, d\, \lambda_{\min}\left( \tilde M_u^{N} \right) + N }
}
\right).
\end{align*}

Since $\text{SubOpt}_u(\pi_u) = \mathbb{E}_{\bm x \sim \rho_p} [\, \text{SubOpt}_u(\pi_u, \bm x) ]$, it follows that
\[
\text{SubOpt}_u(\pi_u) 
\leq \tilde O\left(
\frac{\, d \left( 1 + \hat\gamma\, \sqrt{\, N_{\mathcal{W}_{\hat\gamma}(u)}} \right) }{
\sqrt{\, d\, \lambda_{\min}\left( \tilde M_u^{N} \right) + N }
}
\right),
\]
which completes the proof of Theorem~\ref{th2}.
\end{proof}

\subsection{Proof of \Cref{lemma_difference_of_eigenvalues}}
\begin{proof}
According to \Cref{le:multi_step_update}, under the active data augmentation rule in \Cref{align_active_data_selection}, it can be shown that in each block of $d^*$ rounds, the minimum eigenvalue of the Gramian matrix increases by at least $1$, that is, for any $i\in\big\{1,\cdots,\lfloor \frac{N}{d^*} \rfloor\big\}$, 
\[
\lambda_{\min}\left( \tilde M_{u_t}^{\,d^* i} \right) - \lambda_{\min}\left( \tilde M_{u_t}^{\,d^*(i-1)} \right) \,\geq\, 1.
\]
Therefore, we have:
\begin{align*}
\lambda_{\min}\left( \tilde M_{u_t}^N \right) - \lambda_{\min}\left( \tilde M_{u_t} \right)
\,&\geq\, 
\lambda_{\min}\left( \tilde M_{u_t}^{\,d^* \lfloor \frac{N}{d^*} \rfloor} \right)
- \lambda_{\min}\left( \tilde M_{u_t} \right)
\\&\geq\, 
\sum_{i=0}^{\,\lfloor \frac{N}{d^*} \rfloor - 1}
\left( 
\lambda_{\min}\left( \tilde M_{u_t}^{\,d^*(i+1)} \right) 
- 
\lambda_{\min}\left( \tilde M_{u_t}^{\,d^* i} \right)
\right)
\,\geq\, 
\Big\lfloor \frac{N}{d^*} \Big\rfloor,
\end{align*}
where we define $\lambda_{\min}\left( \tilde M_{u_t}^{0} \right) = \lambda_{\min}\left( \tilde M_{u_t} \right)$ to be the minimum eigenvalue of the Gramian matrix constructed from the aggregated offline data. This completes the proof of \Cref{lemma_difference_of_eigenvalues}.
\end{proof}

\section{Technical Lemmas}

\begin{lemma}[Confidence Ellipsoid]\label{le_confidence_ellipsoid}
Let $\{F_t\}_{t=0}^{\infty}$ be a filtration. Let $\{\varepsilon_t\}_{t=1}^{\infty}$ be a real-valued stochastic process such that $\varepsilon_t$ is $F_t$-measurable and $\varepsilon_t$ is conditionally $R$-subgaussian for some $R>0$. Moreover, let $\{X_t\}_{t=1}^{\infty}$ be an $\mathbb{R}^d$-valued stochastic process such that $X_t$ is $F_{t-1}$-measurable. Assume that $V=\lambda I$ for $\lambda>0$ is a $d\times d$ positive definite matrix. For any $t\geq 0$, define
\begin{align*}
\overline{V}_t=V+\sum_{s=1}^nX_sX_s^{\top},\, S_t=\sum_{s=1}^n\varepsilon_sX_s.
\end{align*}
Let $Y_t=\langle X_t,\bm\theta^* \rangle + \varepsilon_t$ and assume that $\left\| \bm\theta^* \right\|_2\leq S$. Then for any $\delta>0$, with probability at least $1-\delta$, for all $t\geq 0$, $\bm\theta^*$ lies in the set
\begin{align*}
C_t=\left\{ \bm\theta\in\mathbb{R}^d:\left\| \hat{\bm\theta}_t - \bm\theta
 \right\|_{\overline{V}_t}\leq R\sqrt{2\log\left( \frac{\det\left( \overline{V}_t \right)^{1/2}\det\left( \lambda I \right)^{-1/2}}{\delta}
 \right)} + \lambda^{1/2}S \right\}
\end{align*}
where $\hat{\bm\theta}_t=\left( \bm X_{1:t}^{\top}\bm X_{1:t} + \lambda I \right)^{-1}\bm X_{1:t}^{\top}\bm Y_{1:t}$ is the least squares estimate of $\bm\theta^*$, for $\bm X_{1:t}$ being the matrix whose rows are $X_1^{\top},\cdots, X_t^{\top}$ and $\bm Y_{1:t}=(Y_1,\cdots, Y_t)^{\top}$. Furthermore, if for all $t\geq 1$, $\left\| X_t \right\|_2\leq L$ then with probability at least $1-\delta$, for all $t\geq 0$, $\bm\theta^*$ lies in the set
\begin{align*}
C_t'=\left\{ \bm\theta\in\mathbb{R}^d:\left\| 
\hat{\bm\theta}_t - \bm\theta\right\|_{\overline{V}_t}\leq R\sqrt{d\log\left( \frac{1+tL^2/\lambda}{\delta} \right)}  + \lambda^{1/2}S \right\}.
\end{align*}
\end{lemma}

\begin{proof}
Lemma~\ref{le_confidence_ellipsoid} comes from Theorem 2 in~\citet{abbasi2011improved}.
\end{proof}

\begin{lemma}[Elliptic Potential Lemma]\label{le_elliptic_potential_lemma}
Let $\{\bm z_s\}_{s=1}^n$ be a sequence of vectors in $\mathbb{R}^d$ such that $\left\| \bm z_s \right\|\leq L$ for any $s\in[t]$. Let $V_t=\sum_{s=1}^{t-1}\bm z_s\bm z_s^{\top}+\lambda I$. Then,
\begin{align*}
\sum_{s=1}^n\left\| \bm z_s \right\|_{V_s^{-1}}^2\leq 2d\log\left( 1 + \frac{tL^2}{\lambda d} \right).
\end{align*}
\end{lemma}

\begin{proof}
\Cref{le_elliptic_potential_lemma} comes from Lemma C.2 in~\citet{das2024active}.
\end{proof}

\begin{lemma}[Lower Bound on the Minimum Eigenvalue]\label{le_lambdamin}
Let $\bm a_s,n\geq 1$ be generated sequentially from a random distribution such that $||\bm a||_2\leq 1$ and $\mathbb{E}[\bm a\bm a^{\top}]$ is full rank with minimal eigenvalue $\lambda_a>0$. Let $M_n=\sum_{s=1}^n\bm a_s\bm a_s^{\top}$. Then event 
\begin{align*}
\lambda_{\min}(M_n)\geq \left(n\lambda_a - \frac{1}{3}\sqrt{18nA(\delta) + A(\delta)^2} - \frac{1}{3}A(\delta) \right)
\end{align*}
holds with probability at least $1-\delta$ for $n\geq 0$ where $A(n, \delta)=\log\left(\frac{(n+1)(n+3)d}{\delta} \right)$. Furthermore,
\begin{align*}
\lambda_{\min}(M_n)\geq \frac{1}{2}\lambda_an,\ \forall n\geq\frac{16}{\lambda_a^2}\log\left(\frac{8d}{\lambda_a^2\delta}\right)
\end{align*}
holds with probability at least $1-\delta$.
\end{lemma}

\begin{proof}
Lemma~\ref{le_lambdamin} comes from Lemma 7 in~\citet{li2018online} and Lemma B.2 in~\citet{wang2025online}.
\end{proof}

\begin{lemma}[One‐step Update on the Euclidean Unit Ball]\label{le_one_step_update}
Let \(M\in\mathbb R^{d\times d}\) be symmetric positive semidefinite with eigenvalues \(\lambda_1(M)\le\lambda_2(M)\le\cdots\le\lambda_d(M),\)
and corresponding orthonormal eigenvectors \(q_1,\dots,q_d\).  Let
\begin{align}\label{align_ads_in_appendix}
z^* \coloneqq \arg\max_{\|z\|_2\le1} \;z^\top M^{-1}z,
\end{align}
and define the rank‐one update \(M^+ \;=\; M + z^*(z^*)^\top. \) Then the increase in the smallest eigenvalue satisfies
\[
\lambda_{\min}(M^+)-\lambda_{\min}(M)
\;=\;
\min\bigl\{1,\;\lambda_2(M)-\lambda_1(M)\bigr\}.
\]
Moreover, the original eigenvector $q_1$ remains an eigenvector of $M^+$, now with eigenvalue
\[
M^+q_1 = (\lambda_1(M) + 1)\,q_1.
\]

\end{lemma}

\begin{proof}
Write the spectral decomposition
\[
M = Q\,\mathrm{diag}(\lambda_1,\lambda_2,\dots,\lambda_d)\,Q^\top,
\]
with \(Q=[q_1,\dots,q_d]\) where \(Q^{-1}=Q^T\) due to its semi-definite property. For any \(z\) with \(\|z\|\le1\), let \(y=Q^\top z\), so \(\|y\|\le1\) and
\[
z^\top M^{-1}z
= y^\top\mathrm{diag}(1/\lambda_1,\dots,1/\lambda_d)\,y
= \sum_{i=1}^d\frac{y_i^2}{\lambda_i}.
\]
Since \(1/\lambda_1\ge1/\lambda_2\ge\cdots\), this quadratic form is maximized by concentrating all mass on the first coordinate:
\[
y^* = \pm e_1,
\quad
\Longrightarrow
\quad
z^* = Q\,y^* = \pm q_1,
\]
and without loss of generality \(z^*=q_1\).  Moreover, because we chose an orthonormal eigenbasis, \(\|q_1\|=1\), so \(\|z^*\|=1\).

Now consider \(M^+=M+q_1q_1^\top\).  Observe:
\[
M^+ q_1 = \lambda_1 q_1 + q_1 = (\lambda_1+1)q_1,
\quad
M^+ q_i = \lambda_i q_i
\quad (i\ge2),
\]
since \(q_i^\top q_1=0\).  Therefore the eigenvalues of \(M^+\) are
\(\lambda_1+1,\;\lambda_2,\dots,\lambda_d\), and so
\[
\lambda_{\min}(M^+)
= \min\{\lambda_1+1,\;\lambda_2\}.
\]
Subtracting \(\lambda_{\min}(M)=\lambda_1\) gives
\[
\lambda_{\min}(M^+)-\lambda_1
= \min\{\lambda_1+1,\lambda_2\}-\lambda_1
= \min\{1,\;\lambda_2-\lambda_1\},
\]
as claimed.
\end{proof}

\begin{lemma}[Multi‐step Update on the Euclidean Unit Ball]\label{le:multi_step_update}
Let \(M\in\R^{d\times d}\) be symmetric positive semidefinite with eigenvalues
\[
\lambda_1(M)\le\lambda_2(M)\le\cdots\le\lambda_d(M).
\]
Suppose that there exists an integer $s\in\{1,2,\cdots,d-1\}$ such that
\[
\lambda_{s+1}(M)\;\ge\;\lambda_1(M)+1.
\]
Perform \(s\) greedy rank‐one updates
\[
M^{(0)}=M,
\qquad
z_t \;=\;\arg\max_{\|z\|\le1}\;z^\top\bigl(M^{(t-1)}\bigr)^{-1}z,
\qquad
M^{(t)}=M^{(t-1)}+z_tz_t^\top,
\quad t=1,\dots,s.
\]
Then
\[
\lambda_1\bigl(M^{(s)}\bigr)\;\ge\;\lambda_1(M)+1.
\]
\end{lemma}

\begin{proof}
Let \(k\) be the largest index such that
\[
\lambda_k(M) < \lambda_1(M) + 1,
\]
so that \(1 \le k \le s\), and by definition, \(\lambda_{k+1}(M) \ge \lambda_1(M) + 1\). By Lemma~\ref{le_one_step_update}, each rank-one update increases the eigenvalue of the currently smallest dimension by 1; in particular, the smallest eigenvalue itself increases by 1 if the second-smallest eigenvalue is at least 1 larger. In our case, since \(\lambda_{k+1}(M) \ge \lambda_1(M) + 1\), the condition of the lemma is satisfied. Thus, after applying the first \(k\) updates (each to a direction aligned with the corresponding eigenvector), we have
\[
\lambda_1\bigl(M^{(k)}\bigr) \ge \lambda_1(M) + 1.
\]

For any \(i > k\), the original eigenvalue \(\lambda_i(M)\) already satisfies \(\lambda_i(M) \ge \lambda_{k+1}(M) \ge \lambda_1(M) + 1\), and rank-one updates can only increase or leave unchanged the eigenvalues. Therefore, the remaining \(s - k\) updates (if any) cannot decrease \(\lambda_1\bigl(M^{(k)}\bigr)\). It follows that
\[
\lambda_1\bigl(M^{(s)}\bigr) \ge \lambda_1\bigl(M^{(k)}\bigr) \ge \lambda_1(M) + 1,
\]
as claimed.
\end{proof}

\end{document}